\documentclass{article}

\PassOptionsToPackage{numbers, compress}{natbib}

\usepackage[preprint]{neurips_2026}

\usepackage[utf8]{inputenc} 
\usepackage[T1]{fontenc}    
\usepackage{hyperref}       
\usepackage{url}            
\usepackage{booktabs}       
\usepackage{amsfonts}  
\usepackage{amssymb}
\usepackage{amsmath}
\usepackage{nicefrac}       
\usepackage{microtype}      
\usepackage{xcolor}         

\usepackage{wrapfig}

\usepackage{multirow}
\usepackage{makecell}

\usepackage{enumitem}

\usepackage{graphicx}

\usepackage{algorithm}
\usepackage{algpseudocode}

\usepackage{amsthm}
\newtheorem{assumption}{Assumption}
\newtheorem{proposition}{Proposition}
\newtheorem{lemma}{Lemma}
\newtheorem{theorem}{Theorem}

\usepackage{subcaption}

\providecommand{\pmstd}[1]{{\scriptsize$\pm$ #1}}
\providecommand{\bestpm}[2]{\textbf{#1} {\scriptsize$\pm$ #2}}

\title{When Generator Replay Degrades:\\
Projected Rehearsal Orchestration for\\
Heterogeneous Federated Class-Incremental Learning}

\author{%
  Thinh T. H. Nguyen$^{1}$, Khoa D. Doan$^{1}$, Binh T. Nguyen$^{2}$, \textbf{Danh Le-Phuoc$^{3}$}, \textbf{Kok-Seng Wong$^{1}$}\thanks{Corresponding author: \url{wong.ks@vinuni.edu.vn}}\\
  $^{1}$VinUniversity, Hanoi, Vietnam\\
  $^{2}$VNU-HCM, University of Science, Ho Chi Minh City, Vietnam\\
  $^{3}$Technische Universit{\"a}t Berlin, Germany\\
  \texttt{\{thinh.nth, khoa.dd\}@vinuni.edu.vn, ngtbinh@hcmus.edu.vn,}\\
  \texttt{danh.lephuoc@tu-berlin.de, wong.ks@vinuni.edu.vn}
}

\begin{document}

\maketitle

\begin{abstract}
Federated class-incremental learning (FCIL) becomes substantially harder when clients observe different label subsets, progress through tasks at different stages, and provide uneven supervision for the same semantic concepts. Existing FCIL methods often preserve old knowledge through input-space synthesis, but they can be fragile under heterogeneous task streams and difficult to transfer across modalities. To alleviate such issues, we propose PRO, a framework that replaces synthetic input replay with projected rehearsal orchestration. To remove external pretraining, we evaluate all methods under the same warmup. After this, PRO maintains compact class-level projected memories on the server and allows clients perform balanced pseudo multi-task training over current examples and old projected memories. To handle stronger representation drift, we further introduce PRO-MAX, which augments PRO with neighborhood-weighted memory alignment while preserving the same server-light principle that the server only aggregates model updates and memory statistics. Across image, text, and graph benchmarks, PRO and PRO-MAX improve retention and final utility under heterogeneous streams while remaining competitive in homogeneous FCIL. Even when baselines are given expanded replay budgets, they degrade under supervision imbalance and stage misalignment, indicating that replay quantity alone does not resolve replay-quality failures. Additional weak-task diagnostics further show that larger replay mismatch is associated with larger downstream degradation, while our method keeps projected memories better aligned with the evolving representation.
\end{abstract}

\section{Introduction}

Federated learning (FL)~\cite{mcmahan2017communication, kamp2021federated, lu2022federated} enables collaborative training without centralizing raw data, making it suitable for privacy-sensitive applications. However, real client data streams are often non-stationary: new classes appear, label spaces expand, and clients encounter novel concepts under different conditions. This motivates federated class-incremental learning (FCIL)~\cite{dong2022federated, zhang2023target, dong2023no}, where a global model must learn new classes while retaining prior knowledge under privacy and communication constraints. Although FCIL has gained attention, most studies assume relatively homogeneous task sequences across clients~\cite{babakniya2023data, nguyen2024overcoming, liang2024diffusion}. In real deployments, clients may observe different label subsets, follow distinct task orders, and provide uneven supervision for the same concepts. Consequently, even within one communication round, clients can optimize different parts of the label space, making representation learning and memory preservation more challenging~\cite{yang2024federated, hamedi2025federated}.


A common approach in FCIL to handle forgetting is to use generator-based replay or data-free synthesis~\cite{zhang2023target, nori2025federated, rong2025can}. Rather than storing old examples, these methods create artificial samples of previous classes and reuse them during training. While effective on standard vision benchmarks, this design can become fragile as a retention mechanism under heterogeneous streams. There are two main reasons for this. First, the quality of the replay depends on how well each task was learned in the first place. If the supervision was weak or biased, the replayed data can reflect those problems, and errors may accumulate over time. Second, generating samples in the input space ties the memory method to the type of data. Adapting this approach to other data types often needs new generators, extra modules, or makes the system much more complicated. Ideally, memory in FCIL should work across different data types and not be too sensitive to uneven learning.

These observations suggest that the key abstraction should be a modality-agnostic memory interface rather than a modality-specific generator. Motivated by this view, we propose \textbf{P}rojected \textbf{R}ehearsal \textbf{O}rchestration (\textbf{PRO}), a generator-free framework for heterogeneous FCIL. PRO maintains compact projected memories on the server, i.e., feature-space summaries computed after the modality-specific encoder rather than input-space examples. The server broadcasts these memories to clients, where they are sampled as old-class pseudo features and combined with current-task features in a class-balanced projected rehearsal objective. In this way, old-new consolidation is performed distributively on clients without an additional server-side optimization loop. To further handle stronger representation drift, we introduce a \textbf{M}emory \textbf{A}lignment e\textbf{X}tension, yielding \textbf{PRO-MAX}. PRO-MAX estimates how old projected memories should move as client updates shift the feature space and aligns them through confidence-weighted aggregation. Both variants follow the same principle that the server \textbf{only} aggregates model updates and memory statistics and performs no gradient-based training.

We evaluate our methods across image, text, and graph benchmarks under both homogeneous and heterogeneous FCIL protocols. The experiments test whether projected rehearsal remains competitive in standard FCIL, whether it is more robust than generator-based replay under heterogeneous streams, and whether the same memory abstraction transfers across modalities. We further include a controlled weak-task stress test to examine whether replay derived from a poorly learned task causes downstream degradation. In summary, this work makes four contributions:

\begin{itemize}[leftmargin=*]
\item We identify two limitations of generator-based FCIL under heterogeneous streams, i.e., modality-coupled memory and error compounding when weakly learned tasks produce weak replay signals.
\item To enable modality-agnostic continual memory without input synthesis, we propose PRO, a generator-free FCIL framework that performs projected rehearsal through compact class-level memories and client-side balanced pseudo multi-task learning.
\item To keep old memories compatible with evolving representations under severe heterogeneity, we introduce PRO-MAX, which adds neighborhood-weighted memory alignment through confidence-weighted server aggregation.
\item We evaluate our methods across image, text, and graph benchmarks under homogeneous and heterogeneous FCIL, including system-cost measurements and additional tests for replay robustness.
\end{itemize}

\section{Related Work}
\label{sec:related-work}

\textbf{Representation-based Continual Learning.}
Continual learning (CL) mitigates forgetting via regularization, distillation, replay, or dynamic architectures. Closest to our work are methods that preserve old knowledge in representation space instead of reconstructing raw inputs. Early approaches used feature-space class means for incremental recognition~\cite{rebuffi2017icarl}, while later exemplar-free methods rely on prototypes, prototype augmentation, or translated pseudo-features to avoid storing raw examples~\cite{zhu2021prototype,petit2023fetril}. A key challenge in this setting is feature drift, where evolving representations make old class memories stale and require compensation~\cite{yu2020semantic}. PRO builds on this idea for federated streams by storing compact class-level memories and rehearsing them on clients. Unlike per-example feature banks, which may raise privacy concerns~\cite{mahendran2015understanding,fredrikson2015model,zhu2019deep}, PRO uses class-level projected statistics.

\textbf{Input-Space Replay and Generator-Based FCIL.}
FCIL introduces both local forgetting on clients and global forgetting after aggregation. GLFC is an early method that explicitly addresses this dual forgetting problem~\cite{dong2022federated}. A major exemplar-free branch avoids storing old examples by using data-free distillation, generator-based replay, diffusion replay, or hybrid replay mechanisms~\cite{zhang2023target,babakniya2023data,nguyen2024overcoming,liang2024diffusion,nori2025federated,rong2025can}. These methods are effective in standard vision-centered protocols, but their memory quality depends on the task state from which replay is derived. Under heterogeneous FCIL, weak or biased task learning can therefore produce weak replay signals. Moreover, input-space synthesis is naturally tied to the data modality. PRO differs from this line by replacing input-space generation with projected class memories, while PRO-MAX further aligns these memories when the representation drifts.

\textbf{Heterogeneous Federated Continual Learning.}
This field studies clients with different data distributions, task orders, availability, or learning stages. Prior methods address heterogeneity through inter-client transfer, selective transfer, asynchronous updates, synaptic regularization, class-wise balancing, or prototype communication~\cite{yoon2021federated,chaudhary2022federated,shenaj2023asynchronous,li2024rehearsal,qi2025class,tan2022fedproto,zhu2025federated,wan2025motion}. These works motivate our heterogeneous protocol and baselines. Specifically, PRO focuses on a generator-free memory interface for heterogeneous FCIL streams while keeping old-new consolidation on clients and maintaining only compact projected memories on the server. We provide a full discussion in Appendix~\ref{app:extended-related-work}.

\section{Problem Setup and Motivating Observation}
\label{sec:problem-setup}

\subsection{Heterogeneous FCIL}

We consider an FL system with $C$ clients and a sequence of semantic tasks $\mathcal{S}=\{1,\dots,T\}$. Each task $t$ introduces a disjoint class set $\mathcal{Y}_t$, and the full label space is $\mathcal{Y}=\bigcup_{t=1}^{T}\mathcal{Y}_t$. Unlike homogeneous FCIL, where clients usually follow the same task sequence, we allow each client $c$ to follow a client-specific task order $\pi_c$. At local position $p$, client $c$ receives data
$\mathcal{D}_{c,p}=\{(x_i,y_i)\}_{i=1}^{N_{c,p}}$ with observed labels $y_i\in \widetilde{\mathcal{Y}}_{c,p}\subseteq \mathcal{Y}_{\pi_c[p]}$.
The notation $\widetilde{\mathcal{Y}}_{c,p}$ emphasizes that a client may observe only part of the semantic task label set. We assume that every evaluated semantic class is supported by at least one client when its task appears, i.e.,
$\bigcup_{(c,p):\pi_c[p]=t}\widetilde{\mathcal{Y}}_{c,p}=\mathcal{Y}_t$, for all $t\in\mathcal{S}$.
Thus, supervision can be sparse and unevenly distributed, but no evaluated class is completely absent from the federation.

For client $c$, the semantic task history up to position $p$ is $\mathcal{T}_c^{(p)}=\{\pi_c[0],\dots,\pi_c[p]\}$. The global prediction space is defined at the semantic-task level, i.e,
$\mathcal{Y}^{(p)}
=
\bigcup_{c=1}^{C}
\bigcup_{t\in\mathcal{T}_c^{(p)}}
\mathcal{Y}_t$.
All methods are evaluated under the single-head FCILprotocol where the task identity is not given at inference time.

Within this setup, we instantiate heterogeneity through two controls. \textbf{Supervision imbalance} means that different clients may provide different class coverage or sample counts for the same semantic task. \textbf{Stage misalignment} means that clients need not learn the same semantic task at the same communication time, i.e., $\pi_c[p]\neq \pi_{c'}[p]$ can hold for two clients at the same local position. Together, these two factors create a stream where task supervision is uneven and client progress is not synchronized. Full formal definitions, samplers, and metric details are provided in Appendix~\ref{app:problem-setup}.

\paragraph{Evaluation Summary.} Because clients can follow different task histories, we use a client-first evaluation protocol. At each position, each client is evaluated on the semantic tasks it has encountered so far, while predictions are still made by the global single-head model over the cumulative label space. By this, evaluation respects client-specific CL histories without abandoning the FCIL setting. We report three metrics. \textbf{Final average accuracy} (FAA) measures final utility over each client's accumulated task history. \textbf{Current-task accuracy} (CTA) measures whether the system still learns new tasks under heterogeneous supervision. \textbf{Average forgetting} (AF) measures how much previous tasks degrade after subsequent learning. The formal metric definitions are given in Appendix~\ref{app:problem-setup}.

\subsection{Motivating Observation: When Weak Tasks Poison Generator Replay}
\label{sec:observation}

Generator-based and data-free replay methods are effective when previous tasks are learned under stable supervision, since their replay signals are derived from the task states already learned by the model~\cite{zhang2023target,babakniya2023data}. However, under heterogeneous FCIL, supervision imbalance can leave a task weakly learned, while stage misalignment can force the server to aggregate clients that are optimizing different parts of the label space. Consequently, a weak task state may produce unreliable synthetic replay, and this error can be reused in later stages:
\[
\text{weak task learning}
\;\Longrightarrow\;
\text{poor synthetic replay}
\;\Longrightarrow\;
\text{downstream degradation}.
\]

This does not mean that projected memories are immune to weak task learning. Any memory derived from a poorly learned task can be imperfect. Rather, input-space generators introduce an additional amplification channel, where a biased task state can be repeatedly converted into synthetic samples and rehearsed later. We test this failure mode in Section~\ref{subsec:generator-failure} through task-wise accuracy heatmaps and feature-space comparisons between real and synthetic samples from weakly learned tasks. Motivated by this observation, we replace input-space replay with compact projected class memories.

\section{Methodology}
\label{sec:methodology}


Our approach is guided by three key requirements, as discussed in Section~\ref{sec:observation}. \textbf{First}, clients need to achieve strong performance on their current task; otherwise, any subsequent memory based on that task risks being inaccurate. \textbf{Second}, to ensure the method is broadly applicable and efficient, we represent memory with compact, modality-agnostic class-level statistics instead of relying on raw data, synthetic samples, or storing features for every example. \textbf{Third}, since clients can have diverse data and may focus on different parts of the label space, we use a balanced pseudo multi-task objective for integrating new and previous knowledge.

To realize these requirements, we decompose the model as follows:
$x \xrightarrow{h_{\theta}} u \xrightarrow{a_{\psi}} z \xrightarrow{W} \ell$, where $h_{\theta}$ serves as a modality-specific base encoder, $a_{\psi}$ is a lightweight adapter, and $W$ is a unified classifier head. Here, $u=h_{\theta}(x)\in\mathbb{R}^{d_u}$, and $z=a_{\psi}(u)\in\mathbb{R}^{d_z}$. The prediction at position $p$ uses a single-head class-incremental protocol:
$\hat{y}=\arg\max_{y\in\mathcal{Y}^{(p)}}W_y^\top a_{\psi}(h_{\theta}(x))$.
This decomposition allows the base encoder to project raw inputs into unified feature space, enabling the memory mechanism to function across modalities. 
The adapter, in turn, provides a targeted point for continual adaptation, aligning with the principle of parameter-efficient learning~\cite{houlsby2019parameter,hu2022lora}.


PRO does not rely on an externally pretrained encoder. To ensure a fair comparison, we apply the same warmup phase before the start of the FCIL stream. During warmup, the modality-specific base encoder $h_\theta$ is initialized using only in-domain client data, after which the auxiliary warmup head is removed. This warmup is treated as a protocol-level initialization, not as a unique advantage for PRO. Once this shared initialization is complete, each method proceeds with its own continual learning process. Details of the warmup objective and computational overhead are provided in Appendix~\ref{app:warmup-objective} and Appendix~\ref{app:cost-details}. The specific impact of warmup on PRO is analyzed in the ablation study (Table~\ref{tab:ablation}).

\begin{figure*}[t]
\centering
\includegraphics[width=0.9\linewidth]{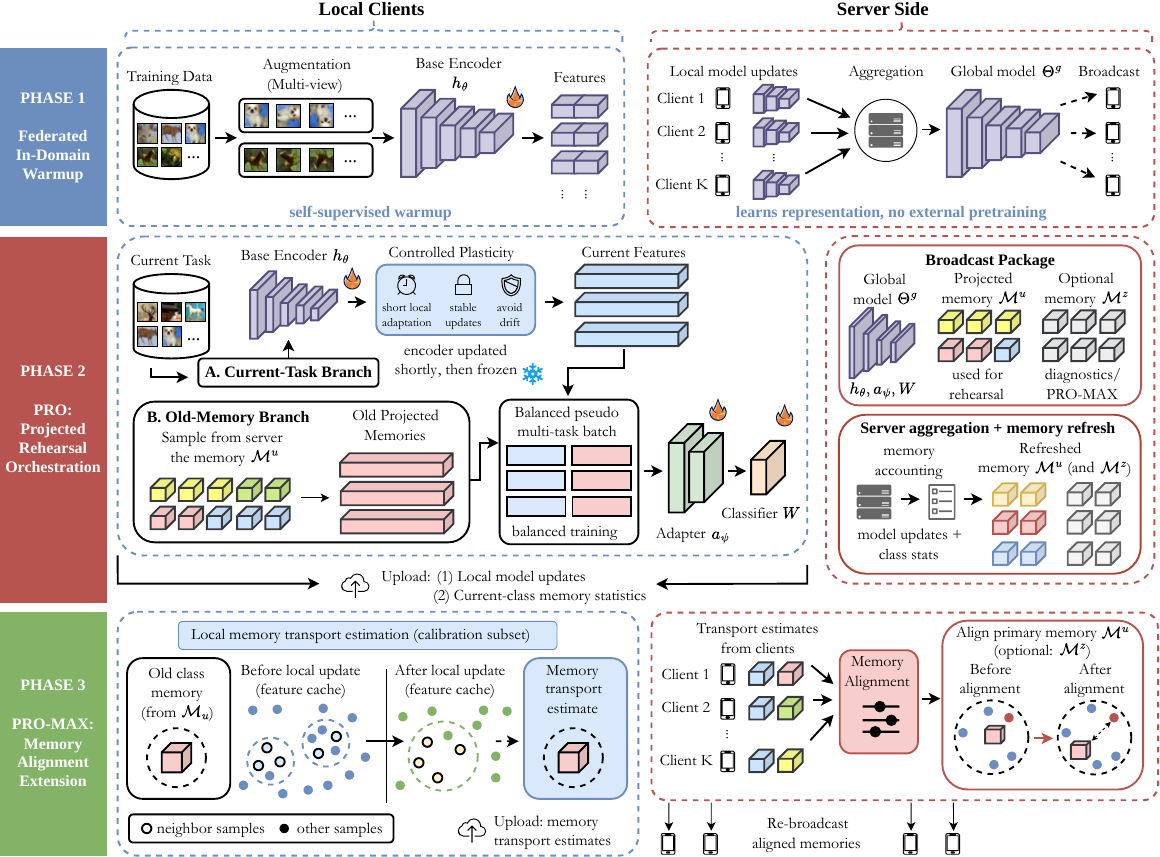}
\caption{Overview of PRO and PRO-MAX. The server broadcasts the global model and projected memories. Each client learns its current task with controlled plasticity, samples old projected memories from the server bank, and trains on a balanced pseudo multi-task objective. The server aggregates model updates and refreshes memory statistics. PRO-MAX further aligns old projected memories using neighborhood-weighted transport.}
\label{fig:method-overview}
\end{figure*}

\subsection{PRO: Projected Rehearsal Orchestration}


{After the warmup phase, the server stores compact projected memories for each class in the base feature space. Specifically, it maintains $\mathcal{M}^{u}=\{(\mu_y^u,\sigma_y^u,n_y)\}_{y\in\mathcal{Y}^{(p)}}$, where $\mu_y^u$ and $\sigma_y^u$ represent the mean and diagonal standard deviation of the base features $u=h_{\theta}(x)$ for each class. In PRO, this memory bank is essential because old projected samples are generated from $\mathcal{M}^u$ before being processed by the adapter $a_{\psi}$. This setup ensures that gradients from old classes continue to update both the adapter and the classifier during training.


At each communication round $r$ within position $p$, the server sends $(\Theta^g,\mathcal{M}^{u})$ to participating clients, where $\Theta^g=(\theta^g,\psi^g,W^g)$. If auxiliary adapted-space memory is used, $\mathcal{M}^{z}$ is also shared. Clients initialize their models from $\Theta^g$ and begin with a brief controlled-plasticity phase designed to learn the current task before constructing memory, while limiting excessive representation drift. In heterogeneous FCIL, labels for the current task may be sparse, incomplete, or biased on each client. As a result, relying solely on supervised updates can lead to overfitting to the limited label subset and disrupt the projected feature space. To address this, when the base encoder is unfrozen, we optionally include a modality-compatible self-supervised auxiliary regularizer. This regularizer does not involve replay or a teacher model; it only constrains the short current-task representation update. The local plasticity objective is given as follows:

\begin{equation}
\label{eq:plastic}
\mathcal{L}_{\mathrm{plastic}}
=
\mathcal{L}_{\mathrm{sup}}
+
\lambda_{\mathrm{aux}}\mathcal{L}_{\mathrm{aux}}
+
\lambda_{\mathrm{prox}}
\left\|
\Pi_r(\Theta)-\Pi_r(\Theta^g)
\right\|_2^2,
\end{equation}
where $\Theta=(\theta,\psi,W)$ and $\Pi_r(\cdot)$ selects only trainable parameters in round $r$. The supervised term is
\begin{equation}
\label{eq:sup}
\mathcal{L}_{\mathrm{sup}}
=
\mathbb{E}_{(x,y)\in\mathcal{D}_{c,p}}
\mathrm{CE}
\left(
W^\top a_{\psi}(h_{\theta}(x)),y
\right).
\end{equation}
The proximal term uses a single coefficient $\lambda_{\mathrm{prox}}$ and constrains the local trainable parameters to remain close to the broadcast global model, which is motivated by client drift in heterogeneous federated optimization~\cite{li2020federated,karimireddy2020scaffold}. When the base encoder is frozen, it is excluded from $\Pi_r(\Theta)$ and we set $\lambda_{\mathrm{aux}}=0$. The distinction between the pre-stream warmup objective and this in-stream auxiliary regularizer is detailed in Appendix~\ref{app:warmup-objective}. In practice, this controlled-plasticity phase is deliberately short. The base encoder is updated only in early rounds of a new task. Later rounds freeze the base encoder and focus on the adapter and classifier.

After this phase, the client constructs a projected pseudo multi-task dataset. For current classes, it uses real examples mapped into the base space $\mathcal{U}^{\mathrm{cur}}_{c,p}
=
\{(h_{\theta_c}(x_i),y_i):(x_i,y_i)\in\mathcal{D}_{c,p}\}$. For old classes, it samples projected memories $\tilde{u}_{y,j}=\mu_y^u+\gamma\sigma_y^u\odot\epsilon_j$, where $\epsilon_j\sim\mathcal{N}(0,I)$, which yields
$\widetilde{\mathcal{U}}^{\mathrm{old}}_{c,p}=\{(\tilde{u}_{y,j},y):y\in\mathcal{Y}^{(p-1)}\}$.
The client trains on the joint dataset $\mathcal{U}^{\mathrm{joint}}_{c,p}=\mathcal{U}^{\mathrm{cur}}_{c,p}\cup\widetilde{\mathcal{U}}^{\mathrm{old}}_{c,p}$ with
\begin{equation}
\label{eq:pro-loss}
\mathcal{L}_{c}^{\mathrm{PRO}}
=
\mathbb{E}_{(u,y)\in\mathcal{U}^{\mathrm{joint}}_{c,p}}
\mathrm{CE}
\left(
W^\top a_{\psi}(u),y
\right).
\end{equation}
The key design choice is that pseudo-old samples are inserted before the adapter: $\tilde{u}\rightarrow a_{\psi}(\tilde{u})\rightarrow W$. Thus, old memories provide gradients to both the adapter and the classifier. If old pseudo features were inserted only after the adapter, they would calibrate $W$ but would not prevent $a_{\psi}$ from drifting toward the current task. We optimize this objective with a class-balanced sampler so that current real features and old projected memories contribute comparably.

After local training, the server aggregates model updates as
$\Theta^g \leftarrow \sum_c \omega_c \Theta_c$, where
$\omega_c=\frac{n_c}{\sum_{c'}n_{c'}}$. It then refreshes current-class memories using the aggregated global model, since statistics computed under local models may not share the same feature space. In the final round of each task, the server sends the global model to participating clients, which recompute class-level projected statistics without gradient updates. The server aggregates only means, second moments, radii, and counts.

\subsection{PRO-MAX: Memory Alignment eXtension}

PRO assumes that the algorithm keeps representation drift moderate. Under heterogenous scenarios, however, even short local updates can make old projected memories stale. PRO-MAX therefore adds neighborhood-weighted memory alignment while preserving the same server-light principle.

Each client caches a small calibration subset before and after local training, i.e., $u_i^- = h_{\theta^-}(x_i)$, and $u_i^+ = h_{\theta^+}(x_i)$, and observes sample drift $\delta_i^u=u_i^+-u_i^-$.
For an old class $y$, the client looks at calibration samples that were close to $\mu_y^u$ before the update. With $d_{i,y}^u=\|u_i^- - \mu_y^u\|_2^2$, it selects a nearest-neighbor set $\mathcal{N}_y^u$ and assigns weights $\alpha_{i,y}^u=\exp\left(-\frac{d_{i,y}^u}{2\tau_u^2}\right)$.

The local memory transport vector is $\Delta_{c,y}^{u} = \frac{\sum_{i\in\mathcal{N}_{y}^{u}}\alpha_{i,y}^{u}\delta_i^u}{\sum_{i\in\mathcal{N}_{y}^{u}}\alpha_{i,y}^{u}+\epsilon}$, where $q_{c,y}^{u}=\sum_{i\in\mathcal{N}_{y}^{u}}\alpha_{i,y}^{u}$.
The same construction can be applied in the adapted space for $\mathcal{M}^z$. The server aggregates transport vectors by confidence, i.e., $\mu_y^u\leftarrow \mu_y^u+\Delta_y^{u,g}$, where $\Delta_y^{u,g}=\frac{\sum_c q_{c,y}^{u}\Delta_{c,y}^{u}}{\sum_c q_{c,y}^{u}+\epsilon}$. The confidence score prevents clients whose current data are far from an old class from strongly influencing that class memory. Algorithms~\ref{alg:pro} and~\ref{alg:promax} in Appendix~\ref{app:method-details} summarize the complete training procedures for PRO and PRO-MAX.



\section{Experimental Protocol}
\label{sec:experimental-protocol}


We structure our experiments around four main objectives. \textbf{First}, we test whether PRO and PRO-MAX achieve strong performance in homogeneous FCIL settings, confirming their effectiveness extends beyond the heterogeneous scenarios. \textbf{Second}, we evaluate whether generator-based and replay-heavy FCIL methods remain robust when facing heterogeneous streams caused by supervision imbalance and stage misalignment. \textbf{Third}, we investigate modality transfer by applying our projected-memory approach to text and graph classification tasks. \textbf{Finally}, we measure communication, memory, and training costs to ensure that the benefits of PRO and PRO-MAX are not due to hidden server-side computation or excessive replay. All results follow the client-first metrics described in Section~\ref{sec:problem-setup}.

\textbf{Benchmarks and settings.}
CIFAR-100 serves as our main image benchmark, divided into 10 semantic tasks with 10 classes each. We use this dataset for experiments on homogeneous and heterogeneous FCIL, generator-replay diagnostics, resource analysis, and ablation studies. For modality transfer, we evaluate on THUCNews-10 for text classification (5 tasks, 2 categories each) and Cora for graph node classification (3 tasks with class groups of size 2, 2, and 3). Additional large-scale results for TinyImageNet, CLINC150, and ogbn-arxiv are provided in Appendix~\ref{app:additional-results}.

\textbf{Baselines.}
We compare with FL-only baselines FedAvg~\cite{mcmahan2017communication}, FedProx~\cite{li2020federated}, and FedProto~\cite{tan2022fedproto}; generator-based or replay-heavy FCIL baselines TARGET~\cite{zhang2023target}, MFCL~\cite{babakniya2023data}, and HR~\cite{nori2025federated}; and heterogeneity-aware or asynchronous FCIL baselines FedSpace~\cite{shenaj2023asynchronous}, FedSSI~\cite{li2024rehearsal}, and FedCBDR~\cite{qi2025class}. For text, we use FedSeIT~\cite{chaudhary2022federated} where its assumptions match the THUCNews protocol. For graph node classification, we compare with POWER~\cite{zhu2025federated} and MOTION~\cite{wan2025motion}.

\textbf{Resource regimes.}
For CIFAR-100, we evaluate a resource-matched regime, where additional replay or memory communication is matched to the projected-memory budget of PRO/PRO-MAX, and a resource-expanded regime, where replay-heavy baselines receive a larger shared replay budget selected on the homogeneous validation setting. The expanded budget is reused in heterogeneous FCIL to test whether replay degradation persists when replay quantity is no longer the limiting factor. To diagnose generator replay, we visualize final task-wise accuracy and compare real versus synthetic features for weakly learned tasks. Full details are given in Appendix~\ref{app:generator-stress} and Appendix~\ref{app:cost-details}.

\textbf{Implementation.}
Images use ResNet-18~\cite{he2016deep}, text uses TextCNN~\cite{chaudhary2022federated}, and graphs use a two-layer GAT~\cite{wan2025motion, zhu2025federated} unless otherwise stated. All methods use the same task splits, client partitions, communication schedule, random seeds, and warmup. After this shared warmup, each baseline follows its original mechanism, while PRO and PRO-MAX use the warmed encoder for projected-memory rehearsal. The concrete protocol values are reported in Appendix~\ref{app:exp-details} and Appendix~\ref{app:generator-stress}.

\section{Results and Analysis}
\label{sec:results}

\subsection{Results under Homogeneous FCIL}
\label{subsec:homogeneous-results}


In Table~\ref{tab:homogeneous-results}, the organization emphasizes resource regime rather than method family, allowing for a clearer comparison of how memory and replay budgets affect performance. In the resource-matched regime, all methods operate with the same additional memory or replay communication budget as PRO and PRO-MAX. The resource-expanded regime provides generator-based and replay-heavy baselines with a larger shared replay budget, testing whether their performance gains are primarily due to increased synthetic or replay resources.


This comparison is essential for interpreting the results in heterogeneous settings that follow. If generator-based baselines only match PRO or PRO-MAX in the homogeneous setting when given an expanded budget, this highlights the resource efficiency of projected memory. Importantly, we apply the same expanded budget in the heterogeneous experiments. If generator-based methods still experience performance drops under heterogeneity, this suggests that the issue is not merely the amount of replay, but also the quality and reliability of the replay signal.

\begin{table*}[ht]
\centering
\caption{
Homogeneous CIFAR-100 FCIL under resource-matched and resource-expanded regimes. We report normalized incremental-stream communication cost together with FAA, CTA, and AF. Resource-matched gives all methods the same additional replay or memory budget as PRO/PRO-MAX. Resource-expanded gives generator-based and replay-heavy FCIL baselines a larger shared replay budget selected on the homogeneous validation setting.}
\label{tab:homogeneous-results}
\resizebox{\textwidth}{!}{%
\begin{tabular}{cccccccc}
\toprule
\makecell[c]{Regime} 
& \makecell[c]{Group} 
& Method 
& Budget 
& Comm. $\downarrow$ 
& FAA $\uparrow$ 
& CTA $\uparrow$ 
& AF $\downarrow$ \\
\midrule

\multirow{11}{*}{\makecell[c]{Matched\\Resource}}
& \multirow{3}{*}{\makecell[c]{FL\\only}}
& FedAvg   & None              & 1.00$\times$ & 10.84 \pmstd{0.36} & 62.14 \pmstd{1.05} & 79.86 \pmstd{0.81} \\
&
& FedProx  & None              & 1.00$\times$ & 11.27 \pmstd{0.42} & 62.04 \pmstd{0.98} & 78.92 \pmstd{0.77} \\
&
& FedProto & Proto             & 1.08$\times$ & 20.65 \pmstd{0.88} & 60.88 \pmstd{1.16} & 62.40 \pmstd{1.42} \\
\cmidrule(lr){2-8}

& \multirow{3}{*}{\makecell[c]{Generator / replay\\FCIL}}
& TARGET   & Synthetic         & 1.22$\times$ & 35.92 \pmstd{1.06} & 66.58 \pmstd{1.12} & 39.48 \pmstd{1.55} \\
&
& MFCL     & Synthetic         & 1.30$\times$ & 37.84 \pmstd{1.18} & 67.36 \pmstd{1.19} & 36.22 \pmstd{1.46} \\
&
& HR       & Hybrid            & 1.35$\times$ & 40.76 \pmstd{1.04} & 68.92 \pmstd{1.01} & 32.64 \pmstd{1.31} \\
\cmidrule(lr){2-8}

& \multirow{3}{*}{\makecell[c]{Heterogeneous\\FCIL}}
& FedSpace & Proto             & 1.15$\times$ & 38.72 \pmstd{1.12} & 65.82 \pmstd{1.23} & 35.76 \pmstd{1.45} \\
&
& FedSSI   & Memory            & 1.10$\times$ & 39.84 \pmstd{1.08} & 66.24 \pmstd{1.17} & 34.96 \pmstd{1.39} \\
&
& FedCBDR  & Replay            & 1.25$\times$ & 41.42 \pmstd{0.94} & 69.18 \pmstd{1.04} & 31.88 \pmstd{1.24} \\
\cmidrule(lr){2-8}

& \multirow{2}{*}{Ours}
& PRO      & \multirow{2}{*}{\makecell[c]{Projected\\memory}} 
                                  & 1.00$\times$ & 43.86 \pmstd{0.82} & 67.20 \pmstd{0.91} & 27.36 \pmstd{1.08} \\
&
& PRO-MAX  &                    & 1.00$\times$ & \bestpm{46.72}{0.78} & 67.90 \pmstd{0.86} & \bestpm{23.64}{0.97} \\
\midrule

\multirow{8}{*}{\makecell[c]{Expanded\\Resource}}
& \multirow{3}{*}{\makecell[c]{Generator / replay\\FCIL}}
& TARGET   & Synthetic         & 2.80$\times$ & 43.18 \pmstd{0.96} & 70.22 \pmstd{1.02} & 30.66 \pmstd{1.28} \\
&
& MFCL     & Synthetic         & 3.00$\times$ & 44.02 \pmstd{1.05} & 70.10 \pmstd{1.04} & 29.12 \pmstd{1.22} \\
&
& HR       & Hybrid            & 3.20$\times$ & 45.34 \pmstd{0.92} & 68.70 \pmstd{0.94} & 26.38 \pmstd{1.13} \\
\cmidrule(lr){2-8}

& \multirow{3}{*}{\makecell[c]{Heterogeneous\\FCIL}}
& FedSpace & Proto             & 1.25$\times$ & 39.28 \pmstd{1.08} & 66.14 \pmstd{1.19} & 35.12 \pmstd{1.38} \\
&
& FedSSI   & Memory            & 1.15$\times$ & 40.58 \pmstd{1.02} & 66.82 \pmstd{1.11} & 34.18 \pmstd{1.32} \\
&
& FedCBDR  & Replay            & 2.80$\times$ & 45.76 \pmstd{0.86} & 68.80 \pmstd{0.90} & 25.84 \pmstd{1.05} \\
\cmidrule(lr){2-8}

& \multirow{2}{*}{Ours}
& PRO      & \multirow{2}{*}{\makecell[c]{Projected\\memory}} 
                                  & 1.12$\times$ & 43.86 \pmstd{0.82} & 67.20 \pmstd{0.91} & 27.36 \pmstd{1.08} \\
&
& PRO-MAX  &                    & 1.16$\times$ & \bestpm{46.72}{0.78} & 67.90 \pmstd{0.86} & \bestpm{23.64}{0.97} \\
\bottomrule
\end{tabular}
}
\end{table*}

\subsection{Results under Heterogeneous FCIL}
\label{subsec:heterogeneous-results}

\begin{table*}[t]
\centering
\caption{
Heterogeneous CIFAR-100 FCIL under the expanded replay budget selected from the homogeneous setting. Generator-based and replay-heavy methods use the same expanded budget as in Table~\ref{tab:homogeneous-results}; PRO and PRO-MAX remain fixed-budget projected-memory methods.
}
\label{tab:heterogeneous-results}
\footnotesize
\begin{tabular}{@{}ccccccc@{}}
\toprule
\makecell[c]{Group} 
& Method  
& Budget                                           
& Comm. $\downarrow$ 
& FAA $\uparrow$              
& CTA $\uparrow$             
& AF $\downarrow$            \\ 
\midrule
\multirow{3}{*}{\makecell[c]{FL\\only}} 
& FedAvg   & None   & 1.00$\times$ & 7.86 \pmstd{0.42}  & 48.64 \pmstd{1.58} & 82.74 \pmstd{1.36} \\
& FedProx  & None   & 1.00$\times$ & 8.92 \pmstd{0.47}  & 49.86 \pmstd{1.51} & 80.96 \pmstd{1.28} \\
& FedProto & Proto  & 1.08$\times$ & 18.74 \pmstd{0.96} & 52.38 \pmstd{1.43} & 66.82 \pmstd{1.74} \\
\midrule
\multirow{3}{*}{\makecell[c]{Generator / replay\\FCIL}} 
& TARGET  & Synthetic          & 2.80$\times$ & 27.86 \pmstd{1.74} & 53.42 \pmstd{1.86} & 57.84 \pmstd{2.18} \\
& MFCL    & Generative replay  & 3.00$\times$ & 29.42 \pmstd{1.68} & 54.76 \pmstd{1.79} & 54.63 \pmstd{2.04} \\
& HR      & Hybrid replay      & 3.20$\times$ & 32.18 \pmstd{1.52} & 56.92 \pmstd{1.64} & 49.76 \pmstd{1.87} \\
\midrule
\multirow{3}{*}{\makecell[c]{Heterogeneous\\FCIL}}   
& FedSpace & Proto             & 1.25$\times$ & 33.84 \pmstd{1.39} & 55.18 \pmstd{1.52} & 47.92 \pmstd{1.76} \\
& FedSSI   & Memory            & 1.15$\times$ & 34.62 \pmstd{1.31} & 55.74 \pmstd{1.44} & 46.38 \pmstd{1.69} \\
& FedCBDR  & Replay            & 2.80$\times$ & 38.96 \pmstd{1.18} & 59.36 \pmstd{1.27} & 39.72 \pmstd{1.48} \\
\midrule
\multirow{2}{*}{Ours}                                
& PRO     & \multirow{2}{*}{\makecell[c]{Projected\\memory}} 
                                      & 1.12$\times$ & 43.28 \pmstd{0.94} & 63.72 \pmstd{1.03} & 31.54 \pmstd{1.18} \\
& PRO-MAX &                                      & 1.16$\times$ & \bestpm{46.12}{0.88} & \bestpm{65.06}{0.96} & \bestpm{26.18}{1.07} \\
\bottomrule
\end{tabular}
\end{table*}

Table~\ref{tab:heterogeneous-results} moves the FCIL methods into the heterogeneous stress setting. Compared with the homogeneous expanded-resource regime, generator-based and replay-heavy methods degrade substantially even though they retain the larger replay budget. This pattern is consistent with the central stress hypothesis: increasing replay quantity does not necessarily fix replay quality when the task state used to derive replay is weak or temporally misaligned.

Among prior heterogeneous FCIL baselines, FedCBDR remains the strongest because its replay-balancing mechanism directly addresses class and task imbalance. FedSpace and FedSSI are more robust than generator-only baselines, but still suffer under the combined supervision-imbalance and stage-misalignment protocol. PRO improves over these baselines with a smaller communication budget by replacing input-space replay with projected class-level memory. PRO-MAX further reduces forgetting through memory alignment, suggesting that its main benefit is retention under representation drift rather than merely higher current-task fitting.

\subsection{When Generators Fail}
\label{subsec:generator-failure}

\begin{figure*}[t]
\centering
\begin{subfigure}{0.42\textwidth}
    \centering
    \includegraphics[width=\linewidth]{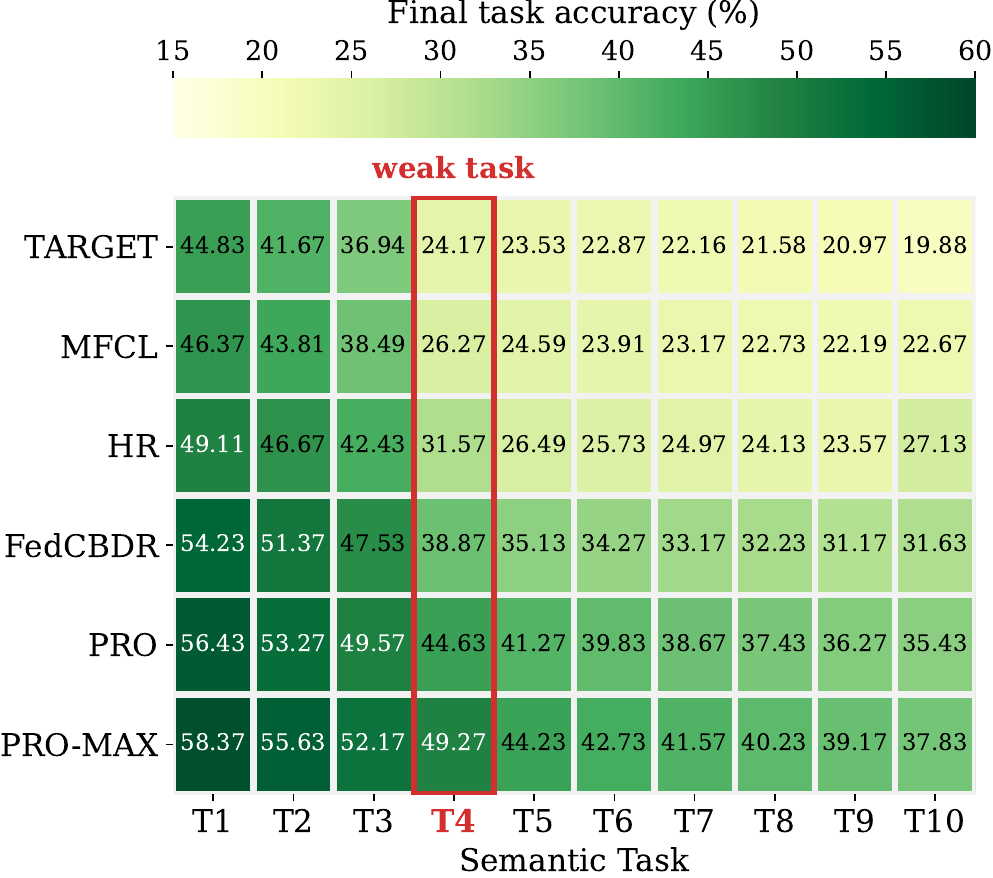}
    \caption{Task-wise final accuracy heatmap.}
\end{subfigure}
\hfill
\begin{subfigure}{0.56\textwidth}
    \centering
    \includegraphics[width=\linewidth]{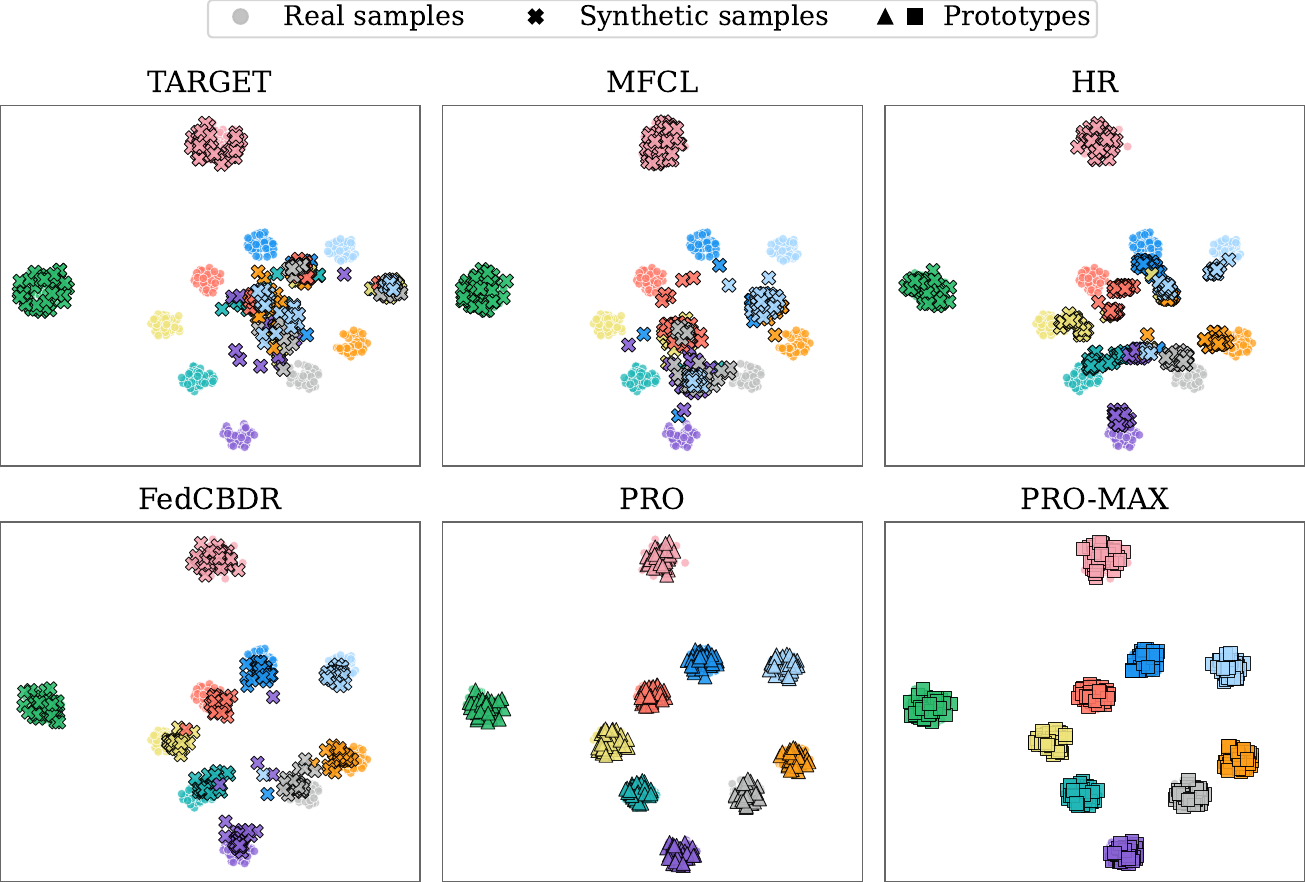}
    \caption{Real vs. generated features in the weak task.}
\end{subfigure}
\caption{
Generator-failure diagnostic under heterogeneous CIFAR-100. The heatmap reports final per-task accuracies after the full stream. The feature visualization compares real held-out samples with replay samples for prior baselines and projected-memory samples for PRO/PRO-MAX.
}
\label{fig:generator-stress}
\end{figure*}

\begin{table*}[t]
\centering
\caption{
Modality-portability evaluation on THUCNews-10 and Cora under both homogeneous and heterogeneous settings. We report FAA, CTA, AF, and normalized cumulative communication.
}
\label{tab:modality-results}
\resizebox{\textwidth}{!}{
\begin{tabular}{@{}ccccccccc@{}}
\toprule
\multirow{2}{*}{Method} 
& \multicolumn{4}{c}{Homogeneous}                                                                                                                  
& \multicolumn{4}{c}{Heterogeneous}                                                                                                                \\ 
\cmidrule(l){2-9} 
& FAA $\uparrow$                           
& CTA $\uparrow$                          
& AF $\downarrow$                         
& Comm. $\downarrow$ 
& FAA $\uparrow$                           
& CTA $\uparrow$                          
& AF $\downarrow$                         
& Comm. $\downarrow$ \\ 
\midrule

\multicolumn{9}{c}{\textbf{THUCNews-10}} \\ 
\midrule
FedAvg                  
& 68.54 {\scriptsize$\pm$ 1.28}          
& 84.18 {\scriptsize$\pm$ 1.04}          
& 31.82 {\scriptsize$\pm$ 1.74}          
& 1.00$\times$       
& 61.08 {\scriptsize$\pm$ 1.52}          
& 79.76 {\scriptsize$\pm$ 1.21}          
& 41.62 {\scriptsize$\pm$ 1.96}          
& 1.00$\times$       \\
FedProx                 
& 69.36 {\scriptsize$\pm$ 1.21}          
& 84.06 {\scriptsize$\pm$ 0.98}          
& 30.64 {\scriptsize$\pm$ 1.68}          
& 1.00$\times$       
& 62.42 {\scriptsize$\pm$ 1.47}          
& 80.13 {\scriptsize$\pm$ 1.18}          
& 40.16 {\scriptsize$\pm$ 1.87}          
& 1.00$\times$       \\
FedProto                
& 72.78 {\scriptsize$\pm$ 1.12}          
& 84.52 {\scriptsize$\pm$ 0.96}          
& 25.92 {\scriptsize$\pm$ 1.46}          
& 1.06$\times$       
& 66.24 {\scriptsize$\pm$ 1.36}          
& 80.48 {\scriptsize$\pm$ 1.10}          
& 34.82 {\scriptsize$\pm$ 1.66}          
& 1.06$\times$       \\ 
FedSeIT                 
& 78.86 {\scriptsize$\pm$ 0.97}          
& 85.42 {\scriptsize$\pm$ 0.84}          
& 18.62 {\scriptsize$\pm$ 1.24}          
& 1.25$\times$       
& 73.28 {\scriptsize$\pm$ 1.18}          
& 82.14 {\scriptsize$\pm$ 0.98}          
& 25.86 {\scriptsize$\pm$ 1.39}          
& 1.25$\times$       \\ 
\midrule
PRO                     
& 82.94 {\scriptsize$\pm$ 0.82}          
& 87.24 {\scriptsize$\pm$ 0.76}          
& 13.82 {\scriptsize$\pm$ 1.02}          
& 1.10$\times$       
& 78.42 {\scriptsize$\pm$ 1.04}          
& 84.08 {\scriptsize$\pm$ 0.91}          
& 20.84 {\scriptsize$\pm$ 1.26}          
& 1.10$\times$       \\
PRO-MAX                 
& \textbf{85.16} {\scriptsize$\pm$ 0.74} 
& \textbf{88.02} {\scriptsize$\pm$ 0.71} 
& \textbf{10.96} {\scriptsize$\pm$ 0.88} 
& 1.14$\times$       
& \textbf{81.62} {\scriptsize$\pm$ 0.96} 
& \textbf{85.22} {\scriptsize$\pm$ 0.86} 
& \textbf{16.42} {\scriptsize$\pm$ 1.10} 
& 1.14$\times$       \\ 

\midrule
\multicolumn{9}{c}{\textbf{Cora}} \\ 
\midrule
FedAvg                  
& 46.18 {\scriptsize$\pm$ 2.76}          
& 58.86 {\scriptsize$\pm$ 2.94}          
& 53.84 {\scriptsize$\pm$ 4.28}          
& 1.00$\times$       
& 39.84 {\scriptsize$\pm$ 2.98}          
& 53.26 {\scriptsize$\pm$ 3.16}          
& 61.82 {\scriptsize$\pm$ 4.56}          
& 1.00$\times$       \\
FedProx                 
& 47.06 {\scriptsize$\pm$ 2.68}          
& 59.12 {\scriptsize$\pm$ 2.88}          
& 51.76 {\scriptsize$\pm$ 4.12}          
& 1.00$\times$       
& 41.32 {\scriptsize$\pm$ 2.87}          
& 54.08 {\scriptsize$\pm$ 3.05}          
& 59.74 {\scriptsize$\pm$ 4.38}          
& 1.00$\times$       \\
FedProto                
& 52.46 {\scriptsize$\pm$ 2.51}          
& 61.18 {\scriptsize$\pm$ 2.72}          
& 44.28 {\scriptsize$\pm$ 3.96}          
& 1.08$\times$       
& 46.82 {\scriptsize$\pm$ 2.74}          
& 56.42 {\scriptsize$\pm$ 2.93}          
& 51.86 {\scriptsize$\pm$ 4.11}          
& 1.08$\times$       \\
POWER                   
& 65.84 {\scriptsize$\pm$ 4.94}          
& 70.86 {\scriptsize$\pm$ 3.38}          
& 28.64 {\scriptsize$\pm$ 4.72}          
& 1.42$\times$       
& 58.76 {\scriptsize$\pm$ 4.38}          
& 64.36 {\scriptsize$\pm$ 3.72}          
& 38.82 {\scriptsize$\pm$ 4.46}          
& 1.42$\times$       \\
MOTION                  
& 66.42 {\scriptsize$\pm$ 4.72}          
& 71.28 {\scriptsize$\pm$ 3.44}          
& 27.86 {\scriptsize$\pm$ 4.51}          
& 1.36$\times$       
& 60.14 {\scriptsize$\pm$ 4.21}          
& 65.08 {\scriptsize$\pm$ 3.58}          
& 35.94 {\scriptsize$\pm$ 4.29}          
& 1.36$\times$       \\ 
\midrule
PRO                     
& 68.86 {\scriptsize$\pm$ 4.06}          
& 72.82 {\scriptsize$\pm$ 3.06}          
& 23.84 {\scriptsize$\pm$ 4.06}          
& 1.12$\times$       
& 63.24 {\scriptsize$\pm$ 3.91}          
& 66.42 {\scriptsize$\pm$ 3.32}          
& 30.86 {\scriptsize$\pm$ 3.94}          
& 1.12$\times$       \\
PRO-MAX                 
& \textbf{70.42} {\scriptsize$\pm$ 3.82} 
& \textbf{73.64} {\scriptsize$\pm$ 2.88} 
& \textbf{20.96} {\scriptsize$\pm$ 3.74} 
& 1.16$\times$       
& \textbf{66.18} {\scriptsize$\pm$ 3.66} 
& \textbf{68.04} {\scriptsize$\pm$ 3.14} 
& \textbf{25.84} {\scriptsize$\pm$ 3.62} 
& 1.16$\times$       \\ 
\bottomrule
\end{tabular}
}
\end{table*}

Figure~\ref{fig:generator-stress} tests the observation in Section~\ref{sec:observation} through two linked views. The heatmap shows final task-wise accuracy, revealing whether a weak task is followed by later degradation that FAA, CTA, and AF may hide. The feature view examines replay quality by embedding real held-out samples and weak-task replay samples in the same feature space. Reliable replay should keep them close, while biased or incomplete weak-task states can shift replay away from the real cluster or cause collapse, separating replay failure from classifier miscalibration. Appendix~\ref{app:causal-generator-failure} quantifies this chain. Table~\ref{tab:generator-causal} reports weak-task accuracy after learning, replay-real mismatch, downstream drop, final FAA, and AF. Figure~\ref{fig:mismatch-downstream-by-method} plots MMD against downstream drop for each method across seeds and weak-task intervention strengths, with circles, squares, and triangles marking mild, medium, and severe interventions. The positive within-method trends show that larger mismatch aligns with larger downstream degradation, supporting that heterogeneous FCIL fails not only from limited replay quantity, but also from degraded replay quality.



\subsection{Modality-Portability Evaluation}
\label{subsec:modality-results}

To test whether the benefit of our method is tied to image benchmarks or reflects a more general continual federated learning mechanism, we further evaluate on a text benchmark, THUCNews-10, and a graph benchmark, Cora. Table~\ref{tab:modality-results} consolidates both the homogeneous and heterogeneous settings in a single view. The homogeneous block shows whether a method transfers across modalities when client streams remain aligned, while the heterogeneous block tests the harder and more realistic case in which cross-client inconsistency is added on top of modality-specific representation challenges.

\subsection{Ablation Study}
\label{subsec:ablation}
Table~\ref{tab:ablation} isolates the main design choices. The PRO and PRO-MAX comparison shows the gain from memory alignment.

\begin{wraptable}{r}{0.68\linewidth}
\centering
\vspace{-0.5cm}
\caption{
Ablation study on CIFAR-100 under the severe-heterogeneity regime.
}
\label{tab:ablation}
\footnotesize
\setlength{\tabcolsep}{2.5pt}
\begin{tabular}{lcccc}
\toprule
Variant & Comm. $\downarrow$ & FAA $\uparrow$ & CTA $\uparrow$ & AF $\downarrow$ \\
\midrule
PRO-MAX & 1.16$\times$ & 46.12 \pmstd{0.88} & 65.06 \pmstd{0.96} & 26.18 \pmstd{1.07} \\
PRO & 1.12$\times$ & 43.28 \pmstd{0.94} & 63.72 \pmstd{1.03} & 31.54 \pmstd{1.18} \\
after-adapter old memory & 1.10$\times$ & 38.64 \pmstd{1.10} & 62.86 \pmstd{1.05} & 41.38 \pmstd{1.36} \\
w/o global refresh & 1.13$\times$ & 40.86 \pmstd{1.06} & 62.94 \pmstd{1.04} & 36.82 \pmstd{1.32} \\
w/o shared warmup & 1.04$\times$ & 37.84 \pmstd{1.24} & 58.66 \pmstd{1.12} & 41.72 \pmstd{1.52} \\
w/o projected old memory & 1.00$\times$ & 24.16 \pmstd{1.18} & 65.88 \pmstd{1.05} & 62.46 \pmstd{1.94} \\
w/o proximal regularization & 1.12$\times$ & 39.72 \pmstd{1.15} & 62.48 \pmstd{1.08} & 39.16 \pmstd{1.41} \\
\bottomrule
\end{tabular}
\end{wraptable}

The \emph{after-adapter old memory} variant rehearses old classes in the adapted space and updates only the classifier. Its drop shows that old projected memories should enter before the adapter, so they regularize both the adapter and classifier, rather than only calibrating $W$. The \emph{without global refresh} variant keeps local or stale memory statistics instead of recomputing them with the aggregated global model. Its degradation confirms that global memory accounting keeps projected memories in a consistent feature space. Removing projected old memory causes the largest forgetting, showing that PRO performs continual consolidation rather than only improving current-task learning. Finally, the \emph{without shared warmup} and \emph{without proximal regularization} variants test whether projected memory needs a meaningful initial representation and controlled local plasticity under severe heterogeneity.

\section{Conclusion}

We studied heterogeneous federated class-incremental learning under supervision imbalance and stage misalignment. Motivated by the fragility of replay signals derived from weakly learned task states, we proposed PRO, a generator-free framework that replaces input-space replay with compact projected class memories and client-side balanced pseudo multi-task learning. We further introduced PRO-MAX, which adds confidence-weighted memory alignment to improve robustness when heterogeneous client updates shift the representation space. The framework avoids raw-data storage, input-space generators, local full-model distillation, and server-side gradient-based optimization during the incremental stream.

The main limitation is that projected memories depend on the quality and stability of the learned representation. When the warmup representation is weak or when class support is extremely sparse, class-level projected statistics may be noisy. PRO-MAX mitigates this issue through memory transport, but it cannot fully recover old-class geometry when no reliable neighborhood support is available. Future work should study stronger privacy accounting for projected memories, adaptive memory uncertainty estimation, and deployment under fully asynchronous client participation.

\newpage

{
    \small
    \bibliographystyle{unsrtnat}
    \bibliography{references}
}

\newpage


\appendix
\section{Extended Related Work}
\label{app:extended-related-work}

\subsection{Representation-Space Memory in Continual Learning}

Continual learning studies how a model can acquire new tasks while retaining performance on old ones. Existing methods are commonly grouped into regularization-based methods, which constrain important parameters from changing too much~\cite{kirkpatrick2017overcoming}; distillation-based methods, which preserve the outputs of previous models~\cite{li2017learning}; replay-based methods, which revisit stored or generated signals from old tasks~\cite{rebuffi2017icarl,buzzega2020dark}; and architecture-based methods, which allocate or expand model components over time~\cite{yan2021dynamically,wang2022foster}. These directions address the stability-plasticity trade-off from different perspectives.

The line most related to PRO is representation-space memory. Early class-incremental learning methods showed that feature-space class means can support incremental recognition~\cite{rebuffi2017icarl}. Later exemplar-free methods reduce the need for raw exemplars by using prototypes, prototype augmentation, or pseudo-features in representation space~\cite{zhu2021prototype,petit2023fetril}. This line is relevant to our work because it shows that old knowledge need not always be stored as input-level examples.

Nevertheless, centralized representation-memory methods do not directly address the constraints of federated class-incremental learning. In FCIL, memories must be communicated, clients may observe different label subsets, and the server aggregates updates from clients at different stages of the stream. PRO adapts the representation-memory idea to this setting by storing compact class-level projected statistics on the server and broadcasting them to clients for local pseudo multi-task learning. The memory is inserted before the adapter, so old classes provide gradients to both the adapter and the classifier rather than only calibrating the final classifier.

Another closely related issue is representation drift. When the feature extractor changes, old class memories computed under earlier models may become stale. Prior work in centralized class-incremental learning has studied semantic drift and proposed compensation mechanisms for old representations~\cite{yu2020semantic}. PRO-MAX follows the same high-level motivation but changes the operating setting. Instead of performing centralized drift compensation, each client estimates local feature movement from a small calibration cache, and the server aggregates confidence-weighted transport summaries. This design is necessary in heterogeneous FCIL because different clients can move their feature spaces in different directions depending on local task order and label support.

Finally, representation-space memory also raises privacy considerations. Storing per-example features is not automatically safe: deep representations and model outputs can be inverted or exploited to reveal information about private inputs~\cite{mahendran2015understanding,fredrikson2015model}, and gradients can leak private training examples in federated settings~\cite{zhu2019deep}. These risks motivate a coarser memory object. Prototype-based federated learning has shown that class-level representatives can be useful communication objects under heterogeneous clients~\cite{tan2022fedproto}. PRO adopts this class-level projected memory perspective. It does not claim formal privacy guarantees, but it avoids raw examples, generated inputs, and instance-level feature banks while retaining enough information for old-class rehearsal.

\subsection{Input-Space Replay and Generator-Based FCIL}

Federated class-incremental learning combines the optimization difficulty of FL with an expanding label space. GLFC is an early FCIL framework that explicitly models both local forgetting on clients and global forgetting after aggregation~\cite{dong2022federated}. Other FCIL methods address forgetting through knowledge distillation, surrogate data, task-adaptive decomposition, or communication-efficient adaptation~\cite{ma2022continual,yoon2021federated,liu2023fedet}. These methods provide important foundations for FCIL, but many are evaluated under relatively synchronized or vision-centered protocols.

A major exemplar-free branch uses data-free or generator-based replay to avoid storing old examples. TARGET performs exemplar-free distillation for federated class-continual learning~\cite{zhang2023target}. MFCL studies a data-free strategy for mitigating catastrophic forgetting in federated class-incremental vision tasks~\cite{babakniya2023data}. Other works improve replay using global twin generators, diffusion-driven replay, hybrid latent rehearsal, or client-assisted generative replay~\cite{nguyen2024overcoming,liang2024diffusion,nori2025federated,rong2025can}. This branch is attractive because it reduces raw exemplar storage and can perform well on standard vision benchmarks.

However, generator-centered replay has two limitations that are central to this paper. First, replay quality is downstream of task-learning quality. If a semantic task is learned from weak, biased, or incomplete supervision, then replay derived from that task can also be weak. Under heterogeneous FCIL, this dependency becomes more fragile because client updates may come from different label subsets or different task stages. Second, input-space generation is naturally modality-dependent. Image generators, text generators, and graph generators require different architectures, constraints, and validity assumptions. As a result, a generator-centered memory mechanism can become difficult to transfer across modalities.

PRO takes a different route. Rather than improving the generator, it replaces input-space replay with projected class-level memory. Once a modality-specific encoder maps raw inputs into a projected feature space, the same memory interface can be used for images, text, and graphs. PRO also keeps consolidation on clients: the server broadcasts projected memories, aggregates model updates, and refreshes memory statistics, but does not train a generator or run server-side SGD over pseudo data. PRO-MAX extends this design by aligning old projected memories when client updates shift the representation.

\subsection{Heterogeneous Federated Continual Learning}

Heterogeneous federated continual learning studies continual learning when clients do not share the same data distribution, task order, participation pattern, or learning stage. This setting builds on the broader FL literature on non-IID optimization. FedProx mitigates client drift with a proximal penalty~\cite{li2020federated}, while SCAFFOLD uses control variates to correct local update drift~\cite{karimireddy2020scaffold}. These methods are not class-incremental learners by themselves, but they motivate a key design choice in PRO: local plasticity should be controlled because heterogeneous clients can otherwise distort the shared representation and make projected memories stale.

Several federated continual learning methods explicitly address heterogeneity. FedWeIT decomposes knowledge into shared and client-specific components to support inter-client transfer under continual streams~\cite{yoon2021federated}. FedSeIT studies federated continual text classification through selective inter-client transfer~\cite{chaudhary2022federated}. Asynchronous federated continual learning methods, such as FedSpace, further relax synchronized task progress and study clients that may be at different stages of the stream~\cite{shenaj2023asynchronous}. More recent methods such as FedSSI and FedCBDR address rehearsal-free synaptic regularization or class-wise replay balancing for heterogeneous continual settings~\cite{li2024rehearsal,qi2025class}. These methods show that heterogeneity is not a minor extension of FCIL, but a defining difficulty.

Heterogeneity also appears in graph-based continual federated learning. POWER and MOTION study federated continual graph learning, where clients own different graph partitions and continual node-label streams~\cite{zhu2025federated,wan2025motion}. These settings are relevant because graph partitions can induce additional representation drift beyond ordinary label skew. Our graph experiments follow this broader motivation, while using projected memories computed from node embeddings rather than storing raw node features or edges.

The proposed setting focuses on heterogeneous class-incremental learning with two explicit controls: supervision imbalance and stage misalignment. Supervision imbalance changes how much and what kind of supervision each semantic task receives across clients. Stage misalignment changes whether clients are learning the same semantic task at the same communication time. These controls are chosen because they expose two different failure modes: weak or biased task learning, and aggregation across clients located at different stages of the continual stream.

PRO is complementary to prior heterogeneous FCL methods. Instead of relying on a generator, a large exemplar memory, or server-side pseudo-data optimization, it maintains a compact projected memory interface and performs balanced old-new consolidation locally on clients. PRO-MAX further targets the representation-drift component of heterogeneity by transporting old memories through confidence-weighted alignment. This makes the method particularly aligned with heterogeneous FCIL, where the server must preserve an expanding label space while clients provide uneven and temporally misaligned supervision.

Although the third direction in the main paper is now framed as heterogeneous FCL, the modality aspect remains important. Many heterogeneous FCIL methods are still evaluated primarily on vision streams, while text and graph settings require different encoders and data constraints. PRO confines modality-specific design to the encoder $h_{\theta}$. Once raw inputs are mapped into projected features, the same class-level memory, client-side pseudo multi-task objective, and server-side memory accounting are reused across modalities. This separation supports our goal of studying heterogeneous FCIL beyond image-only generator pipelines.

\section{Full Heterogeneous FCIL Protocol and Metrics}
\label{app:problem-setup}

\subsection{Detailed Heterogeneous FCIL Setup}

We consider a client-server federated system with $C$ clients and one central server. The set of semantic tasks is
\[
\mathcal{S}=\{1,2,\dots,T\}.
\]
Each semantic task $t\in\mathcal{S}$ introduces a disjoint class set $\mathcal{Y}_t$, and the full semantic label space is
\[
\mathcal{Y}
=
\bigcup_{t=1}^{T}\mathcal{Y}_t,
\qquad
\mathcal{Y}_t\cap\mathcal{Y}_{t'}=\varnothing
\quad
\text{for }t\neq t'.
\]

Client $c$ follows a client-specific semantic task order
\[
\pi_c=[\pi_c[0],\pi_c[1],\dots,\pi_c[T-1]],
\]
where $\pi_c[p]\in\mathcal{S}$ is the semantic task encountered by client $c$ at local position $p$. At that position, the client receives
\[
\mathcal{D}_{c,p}
=
\{(x_i,y_i)\}_{i=1}^{N_{c,p}},
\qquad
y_i\in
\widetilde{\mathcal{Y}}_{c,p}
\subseteq
\mathcal{Y}_{\pi_c[p]}.
\]
Here, $\widetilde{\mathcal{Y}}_{c,p}$ denotes the label subset actually observed by client $c$. This distinguishes client-observed training labels from the full semantic label set $\mathcal{Y}_{\pi_c[p]}$ of the task. Therefore, two clients can encounter the same semantic task but still receive different class subsets or different numbers of examples.

The cumulative semantic task history of client $c$ up to position $p$ is
\[
\mathcal{T}_c^{(p)}
=
\{\pi_c[0],\pi_c[1],\dots,\pi_c[p]\}.
\]
The observed training-label history of client $c$ is
\[
\overline{\mathcal{Y}}_c^{(p)}
=
\bigcup_{q=0}^{p}
\widetilde{\mathcal{Y}}_{c,q}.
\]
In contrast, the semantic label space associated with the tasks encountered by client $c$ is
\[
\mathcal{Y}_{c,\mathrm{sem}}^{(p)}
=
\bigcup_{t\in\mathcal{T}_c^{(p)}}
\mathcal{Y}_t.
\]
The federation-level prediction label space at position $p$ is defined at the semantic-task level:
\[
\mathcal{Y}^{(p)}
=
\bigcup_{c=1}^{C}
\mathcal{Y}_{c,\mathrm{sem}}^{(p)}.
\]
Thus, clients may train on incomplete label subsets, but the global classifier is evaluated over the semantic classes introduced to the federation.

We denote the global model after position $p$ by
\[
f(x;\Theta^{(p)}),
\]
where $\Theta^{(p)}$ is method-agnostic. This notation covers generator-based FCIL methods, distillation-based methods, prototype-based methods, and our projected-memory methods. All methods are evaluated under a single-head class-incremental protocol:
\[
\hat{y}
=
\arg\max_{y\in\mathcal{Y}^{(p)}}
f_y(x;\Theta^{(p)}).
\]
The task identity is not provided at inference time.

\subsection{Support Assumption and Label-Space Convention}

Because clients may observe only part of a semantic task, we assume that every evaluated semantic class is supported by at least one client when its task appears:
\[
\bigcup_{(c,p):\pi_c[p]=t}
\widetilde{\mathcal{Y}}_{c,p}
=
\mathcal{Y}_t,
\qquad
\forall t\in\mathcal{S}.
\]
This assumption allows supervision to be sparse, skewed, and unevenly distributed, but prevents the evaluation from including classes that are completely absent from the federation. Without this assumption, a model could be evaluated on labels that no client has ever observed, which would no longer measure continual learning performance.

Throughout the paper, $\widetilde{\mathcal{Y}}_{c,p}$ denotes client-observed training labels, while $\mathcal{Y}_t$ denotes the full semantic label set of task $t$. The prediction space $\mathcal{Y}^{(p)}$ is defined at the semantic-task level, not merely as the union of labels observed by a particular client. This convention is important because our evaluation uses the full semantic task test split rather than only the labels locally observed by each client.

\subsection{Supervision Imbalance Sampler}

Supervision imbalance captures the fact that clients may provide different amounts or coverage of supervision for the same semantic task. Let $p_c(t)$ be the position at which client $c$ encounters task $t$, and define
\[
\mathcal{D}_{c,t}
=
\mathcal{D}_{c,p_c(t)},
\qquad
\widetilde{\mathcal{Y}}_{c,t}
=
\widetilde{\mathcal{Y}}_{c,p_c(t)}.
\]
Then supervision imbalance occurs whenever, for clients $c\neq c'$,
\[
\widetilde{\mathcal{Y}}_{c,t}
\neq
\widetilde{\mathcal{Y}}_{c',t}
\qquad
\text{or}
\qquad
|\mathcal{D}_{c,t}|
\neq
|\mathcal{D}_{c',t}|.
\]
We define the task support set
\[
\mathcal{C}_t
=
\{c\in\{1,\dots,C\}:|\mathcal{D}_{c,t}|>0\}.
\]
Tasks with small $|\mathcal{C}_t|$, highly imbalanced $|\mathcal{D}_{c,t}|$, or incomplete $\widetilde{\mathcal{Y}}_{c,t}$ receive weak or biased federated supervision.

In experiments, supervision imbalance is controlled by two quantities. The first is the fraction of clients that support each semantic task. The second is the fraction of labels inside that task observed by each supporting client. The sampler first selects a support set $\mathcal{C}_t$ for each task $t$, then assigns each supporting client a subset of $\mathcal{Y}_t$. The assignment is constrained to satisfy the support assumption above, so that every evaluated class has at least one supporting client.

\subsection{Stage Misalignment Sampler}

Stage misalignment captures the fact that clients do not need to learn the same semantic task at the same communication time. Formally, at local position $p$, two clients can satisfy
\[
\pi_c[p]\neq \pi_{c'}[p].
\]
Thus, while one client is learning one semantic task, another client may be updating a different part of the label space. This makes aggregation harder because the server combines updates from clients optimizing different task objectives.

In experiments, stage misalignment is controlled by perturbing client-specific task orders. The homogeneous setting corresponds to
\[
\pi_c=[1,2,\dots,T]
\qquad
\text{for all }c.
\]
In heterogeneous settings, each client receives a perturbed order. Stronger perturbation creates stronger stage misalignment. This is distinct from ordinary non-IID label skew: non-IID skew changes the local class distribution within a task, while stage misalignment changes which semantic task each client is learning at a given position.

\subsection{Homogeneous FCIL as a Special Case}

The homogeneous FCIL setting is recovered when all clients share the same task order and each semantic task is supported by all or most clients. In this case, stage misalignment disappears, and the client-first metrics reduce to the standard class-incremental evaluation over the shared task sequence. We include homogeneous results to verify that PRO and PRO-MAX remain competitive in the setting where most prior FCIL methods are typically evaluated.

\subsection{Client-First Evaluation Metrics}

Because clients follow different task histories, evaluation should respect client-specific histories. For each semantic task $t$, let $\mathcal{D}_t^{\mathrm{test}}$ be the global test split restricted to labels in $\mathcal{Y}_t$. At position $p$, the cumulative test pool for client $c$ is
\[
\mathcal{D}^{\mathrm{test}}(\mathcal{T}_c^{(p)})
=
\bigcup_{t\in\mathcal{T}_c^{(p)}}
\mathcal{D}_t^{\mathrm{test}}.
\]
We evaluate on the full semantic task test split rather than a client-specific subset because the target model is global and should recognize all classes introduced by a semantic task.

\textbf{Final average accuracy.}
The cumulative accuracy of client $c$ at position $p$ is
\[
A_c^{(p)}
=
\frac{1}{|\mathcal{D}^{\mathrm{test}}(\mathcal{T}_c^{(p)})|}
\sum_{(x,y)\in\mathcal{D}^{\mathrm{test}}(\mathcal{T}_c^{(p)})}
\mathbf{1}
\left[
\arg\max_{y'\in\mathcal{Y}^{(p)}}
f_{y'}(x;\Theta^{(p)})=y
\right].
\]
The final average accuracy reported in the main paper is
\[
\mathrm{FAA}
=
\frac{1}{C}
\sum_{c=1}^{C}
A_c^{(T-1)}.
\]

\textbf{Current-task accuracy.}
Average accuracy does not show whether the model can learn newly arriving tasks. Therefore, we also evaluate the current semantic task encountered by each client:
\[
\mathrm{CTA}_c^{(p)}
=
\frac{1}{|\mathcal{D}_{\pi_c[p]}^{\mathrm{test}}|}
\sum_{(x,y)\in\mathcal{D}_{\pi_c[p]}^{\mathrm{test}}}
\mathbf{1}
\left[
\arg\max_{y'\in\mathcal{Y}^{(p)}}
f_{y'}(x;\Theta^{(p)})=y
\right].
\]
We average over clients and positions:
\[
\mathrm{CTA}
=
\frac{1}{CT}
\sum_{c=1}^{C}
\sum_{p=0}^{T-1}
\mathrm{CTA}_c^{(p)}.
\]

\textbf{Average forgetting.}
Let
\[
p_c(t)
=
\min\{p:\pi_c[p]=t\}
\]
be the first position at which client $c$ encounters semantic task $t$. For any $p\ge p_c(t)$, define
\[
A_{c,t}^{(p)}
=
\frac{1}{|\mathcal{D}_t^{\mathrm{test}}|}
\sum_{(x,y)\in\mathcal{D}_t^{\mathrm{test}}}
\mathbf{1}
\left[
\arg\max_{y'\in\mathcal{Y}^{(p)}}
f_{y'}(x;\Theta^{(p)})=y
\right].
\]
The forgetting of client $c$ on task $t$ is
\[
F_{c,t}
=
\max_{p_c(t)\le q\le T-1}
A_{c,t}^{(q)}
-
A_{c,t}^{(T-1)}.
\]
Let the last task encountered by client $c$ be
\[
t_c^{\mathrm{last}}=\pi_c[T-1].
\]
We exclude this task from forgetting because it has not experienced subsequent interference. The client-level forgetting is
\[
\mathrm{AF}_c
=
\frac{1}{|\mathcal{T}_c^{(T-1)}|-1}
\sum_{t\in\mathcal{T}_c^{(T-1)}\setminus\{t_c^{\mathrm{last}}\}}
F_{c,t}.
\]
The final average forgetting is
\[
\mathrm{AF}
=
\frac{1}{C}
\sum_{c=1}^{C}
\mathrm{AF}_c.
\]

\section{Additional Method Details}
\label{app:method-details}

\subsection{Federated In-Domain Warmup}
\label{app:warmup-objective}

The warmup objective is used before the supervised class-incremental stream begins. Its purpose is to initialize the modality-specific base encoder $h_\theta$ from in-domain data without relying on external pretrained weights. To avoid giving PRO and PRO-MAX an initialization advantage, the same federated in-domain self-supervised warmup is applied to all methods with the same backbone, client sampling schedule, number of warmup rounds, optimizer, and local warmup epochs.

Each client optimizes a modality-compatible self-supervised objective
\begin{equation}
\label{eq:warmup}
\mathcal{L}_{\mathrm{warm}}
=
\mathbb{E}_{x}
\left[
\ell_{\mathrm{ssl}}
\left(
r_{\omega}(h_{\theta}(\tilde{x})),s(x,\tilde{x})
\right)
\right],
\end{equation}
where $\tilde{x}\sim\mathcal{A}(x)$ and $r_\omega$ is a local auxiliary head. The auxiliary head is discarded after warmup, and only the warmed base encoder is retained for the supervised incremental stream.

The warmup phase is distinct from the continual-learning method itself. After warmup, each baseline runs its original FCIL mechanism, including its own distillation, replay, prototype, or generator module when applicable. PRO and PRO-MAX use the same warmed encoder as the base feature space for projected memories. Therefore, performance differences in the supervised incremental stream are attributed to the continual-learning mechanism rather than to different representation initialization.

During the class-incremental stream, PRO and PRO-MAX may still use a short auxiliary self-supervised regularizer only when the base encoder is unfrozen during controlled plasticity. This in-stream auxiliary term serves a different purpose from warmup: warmup initializes the representation before the stream, whereas the in-stream term stabilizes a short supervised current-task update. Unless otherwise stated, the auxiliary term is disabled after the early plasticity rounds.

\subsection{Algorithmic Summary}
\label{app:algorithmic-summary}

Algorithms~\ref{alg:pro} and~\ref{alg:promax} summarize the execution order of PRO and PRO-MAX. They refer to the equations and update rules defined in Section~\ref{sec:methodology} and Appendix~\ref{app:method-details}. The key point is that all gradient-based optimization is performed on clients, while the server only aggregates model updates, refreshes projected memories, and, in PRO-MAX, applies confidence-weighted memory alignment.

\begin{algorithm}[t]
\caption{PRO: Projected Rehearsal Orchestration}
\label{alg:pro}
\begin{algorithmic}[1]
\Require Client streams $\{\mathcal{D}_{c,p}\}$, task positions $p=0,\dots,T-1$, communication rounds, global model $\Theta^g$, primary memory $\mathcal{M}^u$, optional memory $\mathcal{M}^z$.
\Ensure Final global model $\Theta^g$ and projected memories.
\State Run the shared federated in-domain warmup using Equation~\eqref{eq:warmup}; discard the warmup head.
\State Initialize $\mathcal{M}^u\leftarrow\emptyset$ and, if enabled, $\mathcal{M}^z\leftarrow\emptyset$.
\For{each task position $p$}
    \For{each communication round $r$}
        \State Server samples participating clients and broadcasts $(\Theta^g,\mathcal{M}^u)$; optionally also broadcasts $\mathcal{M}^z$.
        \For{each participating client $c$ in parallel}
            \State Initialize local model $\Theta_c\leftarrow\Theta^g$.
            \State Perform short controlled-plasticity fitting using Equation~\eqref{eq:plastic}.
            \State Freeze the base encoder after the early plasticity rounds according to the schedule in Appendix~\ref{app:method-details}.
            \State Build current projected features from $\mathcal{D}_{c,p}$.
            \State Sample old projected memories from $\mathcal{M}^u$ using Equation~\eqref{eq:old-memory-sampling}.
            \State Form a class-balanced projected pseudo multi-task batch.
            \State Update the adapter and classifier using Equation~\eqref{eq:pro-loss}.
            \State Upload the local model update and sample count to the server.
        \EndFor
        \State Server aggregates local models using Equation~\eqref{eq:fedavg}.
    \EndFor
    \State Server triggers global-model memory refresh.
    \State Clients recompute current-class memory statistics under the aggregated global model.
    \State Server updates $\mathcal{M}^u$ using Equation~\eqref{eq:global-memory-refresh}; optionally refreshes $\mathcal{M}^z$.
\EndFor
\end{algorithmic}
\end{algorithm}

\begin{algorithm}[t]
\caption{PRO-MAX: Memory Alignment eXtension}
\label{alg:promax}
\begin{algorithmic}[1]
\Require Same inputs as Algorithm~\ref{alg:pro}, calibration subset size, neighbor count $K$, and transport bandwidth.
\Ensure Final global model $\Theta^g$ and aligned projected memories.
\State Run the shared warmup and initialize memories as in Algorithm~\ref{alg:pro}.
\For{each task position $p$}
    \For{each communication round $r$}
        \State Server samples participating clients and broadcasts the global model and memories.
        \For{each participating client $c$ in parallel}
            \State Initialize local model $\Theta_c\leftarrow\Theta^g$.
            \State Cache calibration features before current-task fitting.
            \State Perform controlled-plasticity fitting using Equation~\eqref{eq:plastic}.
            \State Recompute calibration features after fitting and estimate local memory transport using Equation~\eqref{eq:local-transport}.
            \State Sample old projected memories from the locally aligned memory centers.
            \State Form a class-balanced projected pseudo multi-task batch.
            \State Update the adapter and classifier using Equation~\eqref{eq:pro-loss}.
            \State Upload the local model update, sample count, and transport summaries.
        \EndFor
        \State Server aggregates local models using Equation~\eqref{eq:fedavg}.
        \State Server aligns old memories using confidence-weighted transport in Equation~\eqref{eq:server-transport}.
    \EndFor
    \State Server refreshes current-class memories under the aggregated global model as in Algorithm~\ref{alg:pro}.
    \State If $\mathcal{M}^z$ is enabled, apply the same optional refresh and alignment logic in adapted space.
\EndFor
\end{algorithmic}
\end{algorithm}

\subsection{PRO Client Update}

At round $r$ of position $p$, the server broadcasts the global model and base-space memory $(\Theta^g,\mathcal{M}^u)$ to participating clients. If the auxiliary adapted-space memory is enabled, the server also broadcasts $\mathcal{M}^z$. Each client initializes its local model from $\Theta^g$.

The client first performs a short current-task fitting phase:
\[
\mathcal{L}_{\mathrm{fit}}
=
\mathcal{L}_{\mathrm{sup}}
+
\lambda_{\mathrm{aux}}\mathcal{L}_{\mathrm{aux}}
+
\lambda_{\mathrm{prox}}
\left\|
\Pi_r(\Theta)-\Pi_r(\Theta^g)
\right\|_2^2,
\]
where
\[
\mathcal{L}_{\mathrm{sup}}
=
\mathbb{E}_{(x,y)\in\mathcal{D}_{c,p}}
\mathrm{CE}
\left(
W^\top a_{\psi}(h_{\theta}(x)),y
\right).
\]
The operator $\Pi_r(\cdot)$ selects only the parameters that are trainable in round $r$. 
The auxiliary term is instantiated as a short in-stream consistency regularizer over two stochastic views of the same current-task input. For a mini-batch $\mathcal{B}$, let $\tilde{x}_i^{(1)},\tilde{x}_i^{(2)}\sim \mathcal{A}(x_i)$ be two modality-compatible augmentations, and let
\[
v_i^{(a)}
=
\frac{r_{\omega}(h_{\theta}(\tilde{x}_i^{(a)}))}
{\|r_{\omega}(h_{\theta}(\tilde{x}_i^{(a)}))\|_2+\epsilon},
\qquad a\in\{1,2\},
\]
where $r_{\omega}$ is a lightweight projection head used only for the auxiliary regularizer. We use the symmetric contrastive consistency loss
\[
\mathcal{L}_{\mathrm{aux}}
=
\frac{1}{2|\mathcal{B}|}
\sum_{i=1}^{|\mathcal{B}|}
\left[
\ell_{\mathrm{NCE}}(v_i^{(1)},v_i^{(2)})
+
\ell_{\mathrm{NCE}}(v_i^{(2)},v_i^{(1)})
\right],
\]
with
\[
\ell_{\mathrm{NCE}}(v_i,v_i^+)
=
-
\log
\frac{
\exp(\mathrm{sim}(v_i,v_i^+)/\tau_{\mathrm{ssl}})
}{
\sum_{j=1}^{|\mathcal{B}|}
\exp(\mathrm{sim}(v_i,v_j^{+})/\tau_{\mathrm{ssl}})
},
\qquad
\mathrm{sim}(a,b)=a^\top b .
\]
The auxiliary head $r_\omega$ is discarded after the short plasticity phase and is not part of the persistent model. We set $\lambda_{\mathrm{aux}}>0$ only during the first two base-unfrozen rounds of each new task and set $\lambda_{\mathrm{aux}}=0$ after the base encoder is frozen. Unless otherwise stated, we use $\tau_{\mathrm{ssl}}=0.5$, $\lambda_{\mathrm{aux}}=0.05$ for image streams, and $\lambda_{\mathrm{aux}}=0.01$ for text and graph streams.
When the base encoder is frozen, it is excluded from $\Pi_r(\Theta)$ and $\lambda_{\mathrm{aux}}=0$.

After current-task fitting, the client builds a projected pseudo multi-task dataset. For current classes, it uses
\[
\mathcal{U}^{\mathrm{cur}}_{c,p}
=
\{(h_{\theta_c}(x_i),y_i):(x_i,y_i)\in\mathcal{D}_{c,p}\}.
\]
For old classes, it samples
\begin{equation}
\label{eq:old-memory-sampling}
\tilde{u}_{y,j}
=
\mu_y^u+\gamma\sigma_y^u\odot\epsilon_j,
\qquad
\epsilon_j\sim\mathcal{N}(0,I),
\end{equation}
from the base-space memory bank. This gives
\[
\widetilde{\mathcal{U}}^{\mathrm{old}}_{c,p}
=
\{(\tilde{u}_{y,j},y):y\in\mathcal{Y}^{(p-1)}\}.
\]
The client trains on
\[
\mathcal{U}^{\mathrm{joint}}_{c,p}
=
\mathcal{U}^{\mathrm{cur}}_{c,p}
\cup
\widetilde{\mathcal{U}}^{\mathrm{old}}_{c,p}
\]
using
\[
\mathcal{L}_{c}^{\mathrm{PRO}}
=
\mathbb{E}_{(u,y)\in\mathcal{U}^{\mathrm{joint}}_{c,p}}
\mathrm{CE}
\left(
W^\top a_{\psi}(u),y
\right).
\]
We do not add an additional proximal term to this projected pseudo multi-task objective. This batch is already class-balanced over current and old classes, and its purpose is to recalibrate the adapter and classifier with a balanced old-new signal.

\subsection{Server Aggregation and Global-Model Memory Refresh}

After local training, clients upload model updates. The server aggregates them using FedAvg:
\begin{equation}
\label{eq:fedavg}
\Theta^g
\leftarrow
\sum_c\omega_c\Theta_c,
\qquad
\omega_c=\frac{n_c}{\sum_{c'}n_{c'}}.
\end{equation}
The server then refreshes current-class memories under the aggregated global model. This is important because statistics computed under local models may not lie in the same feature space as the final global model.

For base-space memory, client $c$ recomputes, without gradient updates,
\[
\mu_{c,y}^{u,g}
=
\frac{1}{n_{c,y}}
\sum_{(x_i,y_i)\in\mathcal{D}_c:y_i=y}
h_{\theta^g}(x_i),
\]
and
\[
m_{c,y}^{u,g}
=
\frac{1}{n_{c,y}}
\sum_{(x_i,y_i)\in\mathcal{D}_c:y_i=y}
h_{\theta^g}(x_i)\odot h_{\theta^g}(x_i).
\]
The server aggregates
\begin{equation}
\label{eq:global-memory-refresh}
\mu_y^{u}
=
\frac{\sum_c n_{c,y}\mu_{c,y}^{u,g}}
{\sum_c n_{c,y}},
\qquad
m_y^{u}
=
\frac{\sum_c n_{c,y}m_{c,y}^{u,g}}
{\sum_c n_{c,y}},
\end{equation}
and computes
\[
(\sigma_y^u)^2
=
\max(m_y^u-\mu_y^u\odot\mu_y^u,\epsilon).
\]

If the auxiliary adapted-space memory is enabled, clients also compute statistics for
\[
z=a_{\psi^g}(h_{\theta^g}(x)).
\]
The server aggregates the adapted-space mean $\mu_y^z$ and radius $\rho_y^z$:
\[
\rho_y^z
=
\sqrt{
\frac{1}{d_z}
\mathbb{E}_{x\sim y}
\|a_{\psi^g}(h_{\theta^g}(x))-\mu_y^z\|_2^2
}.
\]

This refresh step is memory accounting rather than training. By memory accounting, we mean aggregating class-level means, second moments, radii, counts, and memory metadata. The server does not optimize model parameters or run SGD over pseudo data.

\subsection{PRO-MAX Client-Side Memory Transport}

PRO-MAX inserts memory alignment between current-task fitting and projected pseudo multi-task training. Each client first caches features on a small calibration subset, performs the short current-task fitting update, recomputes the calibration features, and estimates how old memories should move. The locally aligned memories are then used to construct the projected pseudo multi-task batch in the same round.

Let
\[
\mathcal{B}_{c,p}\subset\mathcal{D}_{c,p}
\]
be the calibration subset. Before the fitting step, the client caches
\[
u_i^- = h_{\theta^-}(x_i),
\qquad
x_i\in\mathcal{B}_{c,p}.
\]
After the fitting step, it recomputes
\[
u_i^+ = h_{\theta^+}(x_i),
\qquad
\delta_i^u=u_i^+-u_i^-.
\]
For an old class $y$, define
\[
d_{i,y}^u
=
\|u_i^- - \mu_y^u\|_2^2.
\]
Let $\mathcal{N}_y^u$ be the set of $K$ nearest calibration samples to $\mu_y^u$. The client assigns weights
\[
\alpha_{i,y}^u
=
\exp
\left(
-\frac{d_{i,y}^u}{2\tau_u^2}
\right).
\]
The local transport vector is
\[
\Delta_{c,y}^{u}
=
\frac{
\sum_{i\in\mathcal{N}_y^u}
\alpha_{i,y}^u\delta_i^u
}{
\sum_{i\in\mathcal{N}_y^u}
\alpha_{i,y}^u
+
\epsilon
},
\]
with confidence
\[
q_{c,y}^{u}
=
\sum_{i\in\mathcal{N}_y^u}
\alpha_{i,y}^{u}.
\]
The client forms a locally aligned memory center
\begin{equation}
\label{eq:local-transport}
\tilde{\mu}_{c,y}^{u}
=
\mu_y^u+\Delta_{c,y}^{u},
\end{equation}
and samples old projected features from this aligned center before performing projected pseudo multi-task training.

\subsection{PRO-MAX Server-Side Memory Alignment}

After local training, clients upload model updates and transport summaries. The server first aggregates model parameters as in PRO. It then aligns old base memories by confidence-weighted aggregation:
\begin{equation}
\label{eq:server-transport}
\Delta_y^{u,g}
=
\frac{
\sum_c q_{c,y}^{u}\Delta_{c,y}^{u}
}{
\sum_c q_{c,y}^{u}+\epsilon
},
\qquad
\mu_y^u
\leftarrow
\mu_y^u+\Delta_y^{u,g}.
\end{equation}
The confidence score prevents clients whose current data are far from an old class from strongly influencing that class memory.

If the auxiliary adapted-space memory $\mathcal{M}^z$ is enabled, the same construction is applied in adapted space. The client caches
\[
z_i^- = a_{\psi^-}(u_i^-),
\qquad
z_i^+ = a_{\psi^+}(u_i^+),
\qquad
\delta_i^z=z_i^+-z_i^-.
\]
The server aggregates adapted-space transport vectors and updates
\[
\mu_y^z
\leftarrow
\mathrm{norm}(\mu_y^z+\Delta_y^{z,g}).
\]
This adapted-space alignment is auxiliary. The main rehearsal memory remains $\mathcal{M}^u$.

\subsection{Class-Balanced Projected Training}

Because current classes use real local data while old classes are represented by projected memory, ordinary mini-batch sampling can bias the local objective toward current classes. We therefore use a class-balanced sampler over $\mathcal{Y}^{(p)}$. Equivalently, the objective approximates
\[
\frac{1}{|\mathcal{Y}^{(p)}|}
\sum_{y\in\mathcal{Y}^{(p)}}
\mathbb{E}_{u\sim q_y}
\mathrm{CE}(W^\top a_\psi(u),y),
\]
where $q_y$ is the empirical current distribution for new classes and the projected memory distribution for old classes.

\section{Experimental Setup Details}
\label{app:exp-details}

\subsection{Datasets and Roles}

We evaluate PRO and PRO-MAX on image, text, and graph classification benchmarks. CIFAR-100 is the main image benchmark because its 100 classes support a sufficiently long class-incremental stream. We use it for homogeneous FCIL, heterogeneous FCIL, generator-replay diagnostics, resource analysis, and ablations. THUCNews-10 and Cora are used as modality-portability checks: THUCNews-10 tests text classification, while Cora tests graph node classification. In addition, we report scale-up results on TinyImageNet, CLINC150, and ogbn-arxiv. TinyImageNet increases visual complexity, CLINC150 provides a longer fine-grained intent stream, and ogbn-arxiv provides a larger graph node-classification benchmark.

All datasets are converted into class-incremental streams by partitioning labels into semantic tasks. For every dataset and seed, all methods use identical task splits and client partitions. This ensures that performance differences come from the learning method rather than from different data partitions.

\subsection{Task Splits and Client Partitions}

CIFAR-100 is split into 10 semantic tasks with 10 classes per task. THUCNews-10 is split into 5 tasks with 2 categories per task. Cora is split into 3 tasks with class groups of sizes 2, 2, and 3. TinyImageNet, CLINC150, and ogbn-arxiv follow fixed-size class splits whenever possible, with remaining classes assigned to the final task.

For each seed, we first sample a class order and then partition it into semantic tasks. The same task split is used by all methods. We then construct client-specific streams according to the heterogeneity controls described below. The partition sampler enforces the support assumption in Appendix~\ref{app:problem-setup}: every evaluated class is observed by at least one client when its semantic task appears. Thus, clients may observe incomplete or imbalanced label subsets, but no evaluated class is completely absent from the federation.

\subsection{Heterogeneity Controls}

We instantiate heterogeneity using two independent controls: supervision imbalance and stage misalignment. Supervision imbalance controls how much support each semantic task receives across clients. Specifically, we vary the fraction of clients that support each semantic task and the fraction of labels inside the task observed by each supporting client. This creates tasks with uneven client support and incomplete local label coverage.

Stage misalignment controls whether clients learn the same semantic task at the same communication time. In the homogeneous setting, all clients follow the same task order. In heterogeneous settings, each client receives a perturbed task order. Stronger perturbation creates stronger stage misalignment. The severe-heterogeneity regime combines low task support with strong order perturbation.

These two controls are different from ordinary non-IID label skew. Label skew changes the local class distribution within a task, while stage misalignment changes which semantic task a client is learning at a given position. This distinction is important because heterogeneous FCIL requires the server to aggregate updates from clients that may be optimizing different parts of the label space.

\subsection{Core Numerical Protocol}
\label{app:core-numerical-protocol}

Table~\ref{tab:core-protocol} summarizes the default numerical protocol used in the experiments. For every dataset and seed, all methods use the same task split, client partition, client participation schedule, number of communication rounds, number of local epochs, and random initialization. This ensures that differences in performance are attributable to the learning method rather than to different data partitions or communication schedules.

\begin{table*}[t]
\centering
\caption{
Core experimental protocol. ``Participants'' denotes the number of clients sampled per communication round. Graph mini-batches are defined over target nodes and their local graphs.
}
\label{tab:core-protocol}
\resizebox{\textwidth}{!}{
\begin{tabular}{lccccccc}
\toprule
Dataset & Task split & Clients & Participants & Rounds/task & Local epochs & Batch size & Seeds \\
\midrule
CIFAR-100 & $10\times 10$ classes & 20 & 10 & 20 & 5 & 64 & 5 \\
TinyImageNet & $20\times 10$ classes & 20 & 10 & 20 & 5 & 64 & 5 \\
THUCNews-10 & $5\times 2$ categories & 20 & 10 & 10 & 3 & 64 & 5 \\
CLINC150 & $15\times 10$ intents & 20 & 10 & 10 & 3 & 64 & 5 \\
Cora & $3$ tasks $(2/2/3)$ classes & 5 & 5 & 30 & 5 & full local graph & 5 \\
ogbn-arxiv & $8\times 5$ classes & 10 & 5 & 20 & 3 & 1024 target nodes & 5 \\
\bottomrule
\end{tabular}
}
\end{table*}

Unless otherwise stated, the severe-heterogeneity regime uses $40\%$ task support, $50\%$ per-client label coverage inside each supported task, and $60\%$ task-order disorder. The homogeneous setting uses full task support, full label coverage, and no task-order disorder. The heterogeneity sweeps vary one factor at a time while keeping the other factors fixed at the severe-regime default.

\begin{table}[t]
\centering
\caption{
Heterogeneity settings used in the main and appendix experiments.
}
\footnotesize
\label{tab:heterogeneity-settings}
\begin{tabular}{lccc}
\toprule
Setting & Task support & Label coverage & Order disorder \\
\midrule
Homogeneous & $100\%$ & $100\%$ & $0\%$ \\
Severe heterogeneity & $40\%$ & $50\%$ & $60\%$ \\
Support sweep & $\{80,60,40,20\}\%$ & $50\%$ & $60\%$ \\
Order sweep & $40\%$ & $50\%$ & $\{0,30,60,90\}\%$ \\
\bottomrule
\end{tabular}
\end{table}

For optimization, image experiments use SGD with momentum $0.9$, weight decay $5\times 10^{-4}$, and initial learning rate $0.05$. Text experiments use Adam with learning rate $10^{-3}$ and weight decay $10^{-4}$. Graph experiments use Adam with learning rate $5\times 10^{-3}$ and weight decay $5\times 10^{-4}$. Hyperparameters are selected on held-out validation seeds and then fixed for all reported test seeds.

The federated in-domain warmup uses unlabeled training data only and does not use class labels, task identities, or test examples. We use 20 warmup rounds for image datasets and 10 warmup rounds for text and graph datasets, with one local warmup epoch per round. The auxiliary warmup head is discarded before the supervised incremental stream starts.

\subsection{Replay, Projected-Memory, and Alignment Budgets}
\label{app:sample-budgets}

Table~\ref{tab:sample-budgets} specifies the replay and projected-memory budgets used in the main experiments. For PRO and PRO-MAX, old samples are projected pseudo features sampled from the base-space memory bank $\mathcal{M}^u$; they are not input-space synthetic examples. For generator-based and replay-heavy baselines, the budget denotes the number of replay samples available per old class during local training. The resource-matched regime uses the same per-class budget as PRO/PRO-MAX, while the resource-expanded regime gives replay-heavy baselines a larger budget selected on the homogeneous validation setting.

\begin{table}[t]
\centering
\caption{
Replay and projected-memory budgets. Budgets are specified per old class per participating client per local epoch.
}
\footnotesize
\label{tab:sample-budgets}
\begin{tabular}{lcc}
\toprule
Object & Matched resource & Expanded resource \\
\midrule
PRO/PRO-MAX projected pseudo features & 20 & 20 \\
TARGET/MFCL/HR replay samples & 20 & 100 \\
FedCBDR replay samples & 20 & 100 \\
FedProto/FedSpace prototype summaries & class-level & expanded class-level \\
FedSSI memory summaries & class-level & expanded class-level \\
\bottomrule
\end{tabular}
\end{table}

For PRO and PRO-MAX, the projected pseudo multi-task batch is class-balanced over the current and old label space. In each local epoch, we sample up to 20 projected pseudo features per old class and balance them with current-task features. For PRO-MAX, the calibration subset size is $|\mathcal{B}_{c,p}|=128$, the number of nearest neighbors is $K=5$, the projected-memory noise scale is $\gamma=0.10$, and the proximal coefficient is $\lambda_{\mathrm{prox}}=10^{-2}$. The base encoder is updated only in the first two rounds of each new task; later rounds freeze the base encoder and update the adapter and classifier.

Memory refresh is performed under the aggregated global model at the final round of each task. Participating clients recompute class-level means, second moments, radii, and counts without gradient updates. These statistic uploads are included in the communication accounting in Appendix~\ref{app:cost-details}.

\subsection{Backbones and Initialization}

For image benchmarks, we use a randomly initialized ResNet-18~\cite{he2016deep}. No ImageNet-pretrained weights are used in the main experiments unless explicitly stated in an appendix ablation. This keeps representation learning fully in-domain and avoids external pretraining.

For text benchmarks, we use a TextCNN-style encoder with embedding lookup, convolution, max pooling, and a classifier. This choice follows the CNN-style client architecture used in federated continual text classification settings such as FedSeIT~\cite{chaudhary2022federated}. For text benchmarks, we use a TextCNN-style encoder with embedding lookup, convolution, max pooling, and a classifier. This choice follows the CNN-style client architecture used in federated continual text classification settings such as FedSeIT~\cite{chaudhary2022federated}. We use the same text backbone for THUCNews-10 and CLINC150 unless otherwise stated.

For graph node classification, we use a two-layer GAT encoder, following recent federated continual graph learning protocols that adopt GAT-style backbones for clients and server~\cite{zhu2025federated,wan2025motion}. Each client owns a disjoint node subset and its local computation graph. Projected memories are computed from node embeddings after the GNN encoder and do not store raw node features or edges. The class-incremental stream is defined over node labels, and evaluation is performed on the corresponding global node test split.

\subsection{Warmup Protocol and Fairness}
All methods, including FedAvg, FedProx, FedProto, generator/replay baselines,
heterogeneity-aware baselines, PRO, and PRO-MAX, are initialized with the same
federated in-domain self-supervised warmup before the supervised incremental
stream. The warmup uses the same backbone, client sampling schedule, number of
warmup rounds, optimizer, local warmup epochs, and unlabeled training split for
all methods. The auxiliary warmup head is discarded before the FCIL stream.

The warmup uses no class labels, task identities, or test examples. Since the
warmup is common to all methods, main communication tables report the supervised
incremental-stream cost, while runtime tables report warmup time separately.
The \emph{w/o shared warmup} ablation removes this shared initialization to
measure how much projected memory depends on representation quality; it is not
the default comparison protocol.

\subsection{Baseline Adaptations}

We compare against three main groups of baselines and additional modality-specific baselines. The first group contains FL-only methods: FedAvg~\cite{mcmahan2017communication}, FedProx~\cite{li2020federated}, and FedProto~\cite{tan2022fedproto}. These methods do not explicitly solve FCIL but provide optimization and communication anchors.

The second group contains generator-based or replay-heavy FCIL methods, including TARGET~\cite{zhang2023target}, MFCL~\cite{babakniya2023data}, and HR~\cite{nori2025federated}. These baselines are important because the central motivation of this paper concerns replay quality under weak task learning. We evaluate them under both resource-matched and resource-expanded regimes.

The third group contains heterogeneity-aware or asynchronous FCIL methods, including FedSpace~\cite{shenaj2023asynchronous}, FedSSI~\cite{li2024rehearsal}, and FedCBDR~\cite{qi2025class}. These methods are included to test whether PRO and PRO-MAX improve over baselines designed for heterogeneous continual streams.

For modality-specific experiments, we include FedSeIT~\cite{chaudhary2022federated} for text classification when its assumptions match the THUCNews protocol. For graph node classification, we compare with POWER~\cite{zhu2025federated} and MOTION~\cite{wan2025motion}. When a baseline is not naturally applicable to a modality, we mark it as not applicable rather than modifying it into a different method.

\subsection{Resource-Matched and Resource-Expanded Regimes}

For CIFAR-100, we evaluate two resource regimes. In the resource-matched regime, additional replay or memory communication is matched to the projected-memory budget of PRO and PRO-MAX. This asks whether methods perform well when they receive comparable extra communication beyond standard model updates.

In the resource-expanded regime, replay-heavy, memory-heavy, and prototype-augmented baselines receive a larger shared replay or memory budget selected on the homogeneous validation setting. This expanded budget is then reused in heterogeneous FCIL. The purpose is to test whether replay degradation under heterogeneity persists even when replay quantity is no longer the limiting factor. PRO and PRO-MAX keep their fixed projected-memory budget in both regimes.

\subsection{Hyperparameters, Seeds, and Error Bars}

Hyperparameters are selected on held-out validation seeds and then fixed for all test seeds. We do not tune hyperparameters separately on the final test streams. All methods use the same number of communication rounds, local epochs, participation ratio, optimizer family, and random seeds for each dataset.

Unless otherwise stated, reported numbers are averaged over multiple seeds, and error bars denote standard deviation across seeds. The random sources include task order, client partitioning, model initialization, and mini-batch sampling.

\subsection{Compute and Implementation}

We report communication, memory, and training cost in Appendix~\ref{app:cost-details}. Communication includes model uplink and downlink, memory downlink, statistic uplink, and, for PRO-MAX, memory-transport uplink. Server memory includes persistent projected memory banks. Client extra memory includes temporary buffers such as PRO-MAX calibration features. Warmup cost is reported separately from incremental-stream cost.

During the incremental stream, PRO and PRO-MAX keep the server lightweight. The server performs aggregation, memory accounting, optional non-gradient corrections, and broadcast. It does not run gradient-based optimization over pseudo data.

\section{Generator-Failure Stress Test Details}
\label{app:generator-stress}

\subsection{Purpose of the Stress Test}

The generator-failure stress test evaluates a specific failure mode of heterogeneous FCIL. We do not claim that generator-based replay always fails. Instead, we test whether replay quality becomes unreliable when a semantic task is weakly learned under supervision imbalance or stage misalignment.

The stress test is motivated by the following dependency:
\[
\text{weak task learning}
\;\Longrightarrow\;
\text{poor synthetic replay}
\;\Longrightarrow\;
\text{downstream degradation}.
\]
If a generator-centered method derives replay from a weak task state, then the synthetic replay from that task may be biased, incomplete, or collapsed. Once reused in later stages, this replay can affect not only the weak task itself but also tasks learned afterward.

\subsection{Weak-Task Intervention}

We create a weak task through a controlled intervention. A semantic task $t^\star$ is selected before final evaluation, preferably near the middle of the stream so that later tasks exist after it. For this task, we reduce client support and local supervision while keeping the rest of the stream unchanged. Concretely, we reduce the number of clients that support $t^\star$, reduce the available local examples for its classes, or reduce the local training budget assigned to that task. All other tasks follow the standard heterogeneous protocol.

This controlled intervention creates a reproducible weakly learned task. It also avoids choosing the weak task after observing final forgetting. The goal is to test whether a weak task can contaminate later stages through the replay mechanism.

\subsection{Task-Wise Accuracy Heatmap}

For each method $m$, let $A_{m,t}^{\mathrm{final}}$ denote the final accuracy on semantic task $t$ after the full stream. We visualize the task-wise accuracy vector
\[
[A_{m,1}^{\mathrm{final}}, A_{m,2}^{\mathrm{final}}, \dots, A_{m,T}^{\mathrm{final}}]
\]
as a heatmap. This view complements aggregate metrics such as FAA, CTA, and AF. Aggregate metrics show overall utility and forgetting, but they can hide where degradation begins. The heatmap reveals whether poor performance is localized to the weak task $t^\star$ or spreads to later tasks.

\subsection{Real-vs-Replay Feature Comparison}
\label{app:real-vs-replay}

To inspect replay quality directly, we compare real held-out samples and replay samples for a class $y^\star\in\mathcal{Y}_{t^\star}$ from the weak task. For generator-centered methods, replay samples are input-space synthetic samples,
\[
\tilde{x}\sim q_m(\tilde{x}\mid y^\star,t^\star).
\]
For PRO and PRO-MAX, which do not generate input-space samples, replay samples are projected-memory samples drawn from the class-level memory bank and embedded in the same evaluation feature space.

We use the same evaluation feature extractor $\varphi$ for all methods:
\[
r=\varphi(x),
\qquad
\tilde{r}=\varphi(\tilde{x}),
\]
where $\tilde{r}$ denotes either a generated replay feature for generator-based methods or the corresponding projected-memory sample for PRO and PRO-MAX. This keeps all replay objects comparable in a shared representation space. In the visualization, reliable replay should remain close to real held-out samples from the same weak-task class. If the replay mechanism inherits a biased or incomplete task state, replay features may shift away from the real cluster or collapse to a narrow region.

\subsection{Replay-Real Mismatch and Downstream Metrics}
\label{app:replay-real-metrics}

To quantify replay quality beyond visualization, we compare real held-out weak-task features and replay features in the same evaluation space. Let
\[
\mathcal{R}_{y^\star}^{\mathrm{real}}
=
\{r_i\}_{i=1}^{N},
\qquad
\mathcal{R}_{y^\star}^{\mathrm{rep}}
=
\{\tilde{r}_j\}_{j=1}^{M},
\]
where $r_i=\varphi(x_i)$ and $\tilde{r}_j$ is the replay feature for the same weak-task class. For generator-based methods, $\tilde{r}_j$ is obtained from input-space synthetic replay. For PRO and PRO-MAX, $\tilde{r}_j$ is a projected-memory sample embedded in the same evaluation feature space.

We report three mismatch metrics. The first is empirical maximum mean discrepancy (MMD),
\[
\mathrm{MMD}^2
=
\frac{1}{N(N-1)}
\sum_{i\neq i'}
k(r_i,r_{i'})
+
\frac{1}{M(M-1)}
\sum_{j\neq j'}
k(\tilde{r}_j,\tilde{r}_{j'})
-
\frac{2}{NM}
\sum_{i=1}^{N}
\sum_{j=1}^{M}
k(r_i,\tilde{r}_j),
\]
where $k(\cdot,\cdot)$ is an RBF kernel. MMD is used as the primary mismatch score in the correlation figure because it captures distribution-level discrepancy rather than only center displacement.

The second metric is the empirical 1-Wasserstein distance,
\[
W_1
\left(
\mathcal{R}_{y^\star}^{\mathrm{real}},
\mathcal{R}_{y^\star}^{\mathrm{rep}}
\right),
\]
computed in the same evaluation feature space. The third metric is cosine distance between empirical class means,
\[
\mathrm{CosDist}
=
1-
\frac{
\langle \mu_{y^\star}^{\mathrm{real}},\mu_{y^\star}^{\mathrm{rep}}\rangle
}{
\|\mu_{y^\star}^{\mathrm{real}}\|_2
\|\mu_{y^\star}^{\mathrm{rep}}\|_2
},
\]
where
\[
\mu_{y^\star}^{\mathrm{real}}
=
\frac{1}{N}\sum_{i=1}^{N} r_i,
\qquad
\mu_{y^\star}^{\mathrm{rep}}
=
\frac{1}{M}\sum_{j=1}^{M} \tilde{r}_j.
\]

To measure whether weak-task replay affects later learning, we compute downstream drop on tasks after $t^\star$:
\[
A_{m,>t^\star}^{\mathrm{final}}
=
\frac{1}{T-t^\star}
\sum_{t>t^\star}
A_{m,t}^{\mathrm{final}},
\]
and, when compared against a clean heterogeneous run without weak-task intervention,
\[
\Delta_{\mathrm{down}}
=
A_{m,>t^\star}^{\mathrm{clean}}
-
A_{m,>t^\star}^{\mathrm{weak}}.
\]
A larger $\Delta_{\mathrm{down}}$ means that weakening one task causes stronger degradation on future tasks.

\subsection{Quantitative Causal Generator-Failure Analysis}
\label{app:causal-generator-failure}

Beyond the heatmap and feature visualization in Figure~\ref{fig:generator-stress}, we quantify the weak-task failure chain under controlled interventions. In CIFAR-100, the weak task is $t^\star=4$. We evaluate multiple weak-task intervention strengths while keeping the remaining stream, task order, client partition, and evaluation protocol fixed. Because this analysis requires a sampleable replay or memory object, we report methods with explicit replay or projected-memory samples: TARGET, MFCL, HR, FedCBDR, PRO, and PRO-MAX.

We summarize five quantities. First, $A_{t^\star}^{\mathrm{post}}$ measures weak-task accuracy immediately after the weak task is learned, before later interference dominates. Second, MMD, $W_1$, and cosine distance measure how far replay samples are from real held-out weak-task samples in the shared evaluation feature space. Third, $\Delta_{\mathrm{down}}$ measures downstream degradation on tasks after $t^\star$. Finally, final FAA and AF connect the weak-task diagnostic to the main FCIL metrics.

\begin{table*}[t]
\centering
\caption{
Quantitative causal generator-failure analysis on heterogeneous CIFAR-100 under the weak-task intervention. $A_{t^\star}^{\mathrm{post}}$ is the weak-task accuracy immediately after the weak task is learned. MMD, $W_1$, and cosine distance compare replay samples from the weak task against real held-out samples in the shared evaluation feature space; lower values indicate that replay is closer to the real weak-task distribution. $\Delta_{\mathrm{down}}$ measures downstream drop on tasks after $t^\star$. For PRO and PRO-MAX, the mismatch is computed using projected-memory samples embedded in the same feature space.
}
\label{tab:generator-causal}
\footnotesize
\resizebox{\textwidth}{!}{
\begin{tabular}{lccccccc}
\toprule
Method & $A_{t^\star}^{\mathrm{post}}\uparrow$ & MMD $\downarrow$ & $W_1 \downarrow$ & Cos. Dist. $\downarrow$ & $\Delta_{\mathrm{down}}\downarrow$ & Final FAA $\uparrow$ & Final AF $\downarrow$ \\
\midrule
TARGET  
& 36.48 \pmstd{1.84} 
& 0.19 \pmstd{0.02} 
& 2.74 \pmstd{0.21} 
& 0.32 \pmstd{0.02} 
& 13.42 \pmstd{1.06} 
& 27.86 \pmstd{1.74} 
& 57.84 \pmstd{2.18} \\
MFCL    
& 37.92 \pmstd{1.76} 
& 0.17 \pmstd{0.01} 
& 2.53 \pmstd{0.19} 
& 0.30 \pmstd{0.02} 
& 11.86 \pmstd{0.98} 
& 29.42 \pmstd{1.68} 
& 54.63 \pmstd{2.04} \\
HR      
& 41.26 \pmstd{1.62} 
& 0.15 \pmstd{0.01} 
& 2.16 \pmstd{0.17} 
& 0.25 \pmstd{0.02} 
& 9.74 \pmstd{0.92} 
& 32.18 \pmstd{1.52} 
& 49.76 \pmstd{1.87} \\
FedCBDR 
& 47.38 \pmstd{1.48} 
& 0.12 \pmstd{0.01} 
& 1.74 \pmstd{0.14} 
& 0.21 \pmstd{0.02} 
& 6.42 \pmstd{0.74} 
& 38.96 \pmstd{1.18} 
& 39.72 \pmstd{1.48} \\
PRO     
& 52.84 \pmstd{1.26} 
& 0.08 \pmstd{0.01} 
& 1.18 \pmstd{0.10} 
& 0.15 \pmstd{0.01} 
& 3.18 \pmstd{0.42} 
& 43.28 \pmstd{0.94} 
& 31.54 \pmstd{1.18} \\
PRO-MAX 
& \bestpm{55.62}{1.18} 
& \bestpm{0.06}{0.01} 
& \bestpm{0.96}{0.08} 
& \bestpm{0.12}{0.01} 
& \bestpm{1.94}{0.31} 
& \bestpm{46.12}{0.88} 
& \bestpm{26.18}{1.07} \\
\bottomrule
\end{tabular}
}
\end{table*}

Table~\ref{tab:generator-causal} makes the weak-task failure path explicit. TARGET and MFCL show low weak-task accuracy, large replay-real mismatch, and large downstream drop. HR and FedCBDR reduce mismatch and downstream degradation, but still remain worse than PRO and PRO-MAX. PRO-MAX yields the lowest mismatch across all three metrics and the smallest downstream drop, which is consistent with the role of memory alignment in keeping old memories compatible with the evolving representation.

\begin{figure*}[t]
\centering
\includegraphics[width=\linewidth]{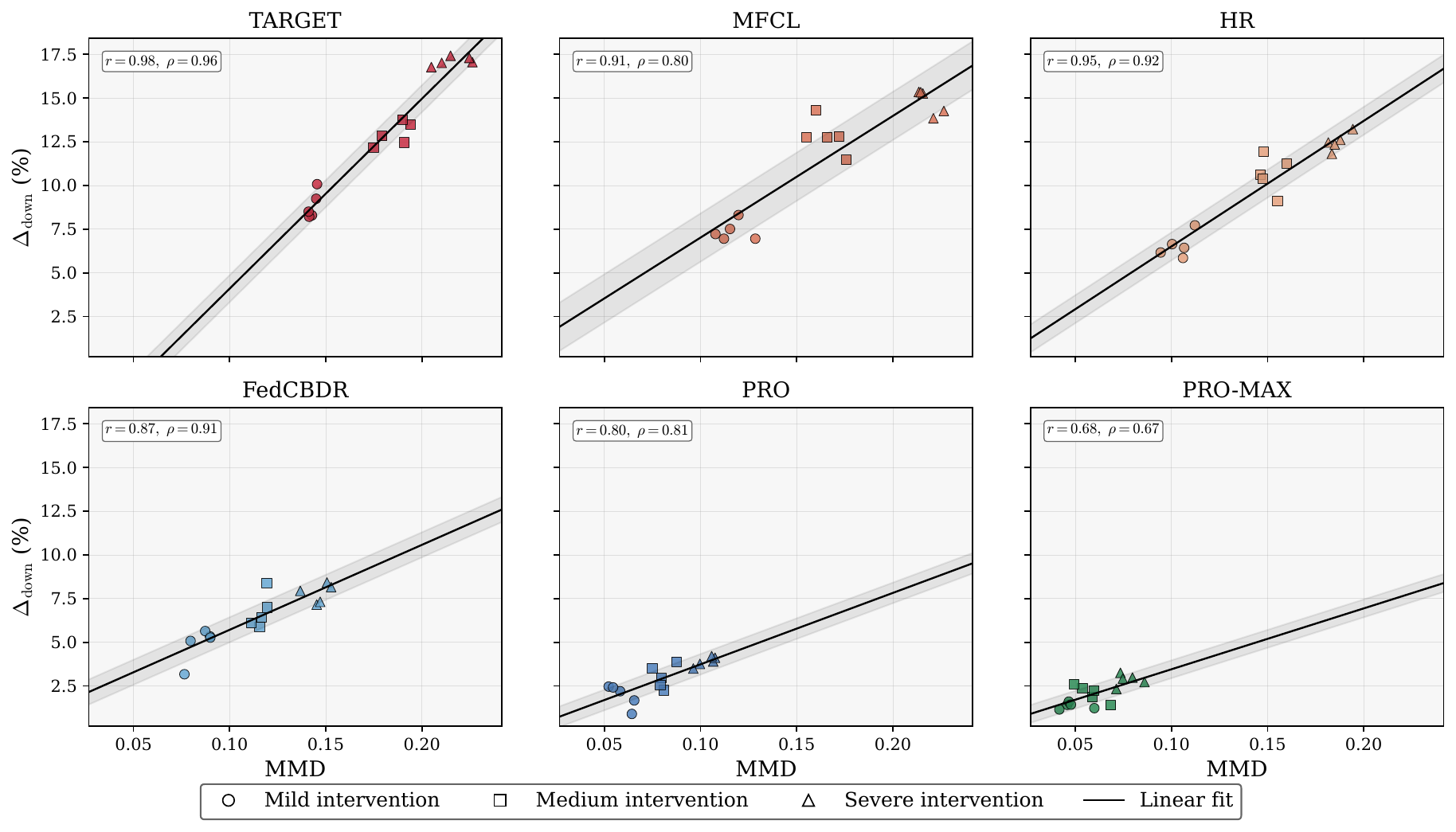}
\caption{
Per-method relation between weak-task replay mismatch and downstream degradation on heterogeneous CIFAR-100. Each panel fixes one method and plots runs across random seeds and weak-task intervention strengths. The x-axis reports replay-real mismatch measured by MMD in the shared evaluation feature space, and the y-axis reports downstream drop $\Delta_{\mathrm{down}}$ on tasks after the weak task. Marker shape indicates intervention strength: circles denote mild intervention, squares denote medium intervention, and triangles denote severe intervention. Across all methods, the fitted trends are positive, and the points shift from lower-left to upper-right as the intervention becomes stronger. This shows that larger weak-task mismatch is associated with larger downstream degradation not only across methods, but also within each method.
}
\label{fig:mismatch-downstream-by-method}
\end{figure*}

Figure~\ref{fig:mismatch-downstream-by-method} addresses a possible confound in pooled correlation plots. A pooled plot could be explained by average differences between methods: stronger methods may simply have lower mismatch and lower downstream drop. By contrast, the per-method panels show that even after fixing the method, runs with larger weak-task mismatch tend to produce larger downstream degradation. The marker progression also shows the effect of intervention strength: mild interventions concentrate in the lower-left region, medium interventions move upward and rightward, and severe interventions produce the largest mismatch and downstream drop. This supports the mechanism proposed in Section~\ref{sec:observation}: weak task learning can produce poor replay, and poor replay can propagate errors to later tasks.

The method-level regions are also informative. TARGET, MFCL, and HR occupy the high-mismatch, high-drop regime, indicating that generator-based replay is most sensitive to weak-task perturbation. FedCBDR is more robust but still follows the same positive trend. PRO and PRO-MAX remain in the low-mismatch, low-drop regime. This does not imply that projected memory is immune to weak task learning; rather, it indicates that projected rehearsal reduces the severity of error propagation, while PRO-MAX further improves robustness through memory alignment.

\subsection{Reproducibility Details}

For each stress-test run, we report the selected weak task $t^\star$, the intervention strength, the number of real and synthetic samples used in the feature comparison, the evaluation feature extractor $\varphi$, and the dimensionality-reduction settings for t-SNE or UMAP. All methods are evaluated under the same task order, client partition, and weak-task intervention for a given seed.

\section{Additional Results and Diagnostics}
\label{app:additional-results}

This section provides additional results and diagnostics that complement the main experiments. The main paper reports CIFAR-100 homogeneous and heterogeneous results, the generator-failure diagnostic in Figure~\ref{fig:generator-stress}, modality-portability results on THUCNews-10 and Cora, and ablations. Here, we further evaluate scale-up benchmarks, heterogeneity sweeps, current-task learning across modalities, memory-alignment quality, and sensitivity to the main memory and plasticity hyperparameters.

The purpose of these results is not to compare absolute numbers with prior work under different protocols. FCIL and FCL papers often differ in backbones, task splits, client partitions, replay budgets, and whether pretrained representations are used. Therefore, all comparisons here are within-protocol: methods are evaluated under the same task splits, client partitions, communication rounds, random seeds, and resource accounting.

\subsection{TinyImageNet Scale-Up Results}
\label{app:tinyimagenet}

TinyImageNet is used as the image scale-up benchmark. It is more visually complex than CIFAR-100 and contains more classes, making it a harder class-incremental stream. Since this is an image benchmark, we use the same method families as in the CIFAR-100 experiments: FL-only baselines, generator/replay FCIL baselines, heterogeneity-aware FCIL baselines, and our projected-memory methods.

\begin{table*}[t]
\centering
\caption{
TinyImageNet scale-up results under homogeneous and heterogeneous FCIL. We report final average accuracy (FAA), current-task accuracy (CTA), average forgetting (AF), and normalized cumulative communication.
}
\label{tab:tinyimagenet-scaleup}
\resizebox{\textwidth}{!}{
\begin{tabular}{c cccc cccc}
\toprule
\multirow{2}{*}{Method}
& \multicolumn{4}{c}{Homogeneous}
& \multicolumn{4}{c}{Heterogeneous} \\
\cmidrule(lr){2-5}\cmidrule(lr){6-9}
& FAA $\uparrow$ & CTA $\uparrow$ & AF $\downarrow$ & Comm. $\downarrow$
& FAA $\uparrow$ & CTA $\uparrow$ & AF $\downarrow$ & Comm. $\downarrow$ \\
\midrule
FedAvg   
& 6.28 \pmstd{0.52} & 42.72 \pmstd{1.88} & 86.14 \pmstd{1.42} & 1.00$\times$
& 3.72 \pmstd{0.48} & 32.18 \pmstd{2.12} & 90.12 \pmstd{1.76} & 1.00$\times$ \\
FedProx  
& 7.36 \pmstd{0.61} & 43.51 \pmstd{1.79} & 84.27 \pmstd{1.50} & 1.00$\times$
& 4.86 \pmstd{0.54} & 34.62 \pmstd{2.04} & 88.64 \pmstd{1.68} & 1.00$\times$ \\
FedProto 
& 14.92 \pmstd{0.88} & 45.03 \pmstd{1.62} & 71.36 \pmstd{1.94} & 1.08$\times$
& 11.48 \pmstd{0.93} & 37.84 \pmstd{1.96} & 76.72 \pmstd{2.18} & 1.08$\times$ \\
TARGET   
& 26.84 \pmstd{1.16} & 51.12 \pmstd{1.42} & 50.68 \pmstd{1.86} & 2.80$\times$
& 18.96 \pmstd{1.62} & 40.76 \pmstd{2.18} & 66.48 \pmstd{2.36} & 2.80$\times$ \\
MFCL     
& 28.73 \pmstd{1.24} & 52.08 \pmstd{1.36} & 47.91 \pmstd{1.74} & 3.00$\times$
& 20.42 \pmstd{1.54} & 42.33 \pmstd{2.06} & 63.84 \pmstd{2.28} & 3.00$\times$ \\
HR       
& 31.25 \pmstd{1.13} & 53.94 \pmstd{1.31} & 44.36 \pmstd{1.66} & 3.20$\times$
& 23.16 \pmstd{1.46} & 44.25 \pmstd{1.92} & 58.92 \pmstd{2.16} & 3.20$\times$ \\
FedSpace 
& 30.72 \pmstd{1.18} & 51.97 \pmstd{1.38} & 46.28 \pmstd{1.72} & 2.25$\times$
& 24.88 \pmstd{1.35} & 43.76 \pmstd{1.86} & 56.31 \pmstd{2.05} & 2.25$\times$ \\
FedSSI   
& 31.94 \pmstd{1.14} & 52.56 \pmstd{1.34} & 44.91 \pmstd{1.68} & 2.15$\times$
& 25.64 \pmstd{1.28} & 44.18 \pmstd{1.78} & 54.86 \pmstd{1.98} & 2.15$\times$ \\
FedCBDR  
& 34.86 \pmstd{1.02} & 55.84 \pmstd{1.20} & 40.18 \pmstd{1.42} & 2.80$\times$
& 29.46 \pmstd{1.18} & 48.37 \pmstd{1.62} & 49.72 \pmstd{1.84} & 2.80$\times$ \\
PRO      
& 37.42 \pmstd{0.96} & 57.16 \pmstd{1.08} & 35.72 \pmstd{1.26} & 1.14$\times$
& 33.28 \pmstd{1.04} & 51.96 \pmstd{1.42} & 42.84 \pmstd{1.60} & 1.14$\times$ \\
PRO-MAX  
& \bestpm{40.36}{0.91} & \bestpm{58.74}{1.02} & \bestpm{30.84}{1.18} & 1.19$\times$
& \bestpm{37.54}{0.98} & \bestpm{54.18}{1.34} & \bestpm{35.96}{1.46} & 1.19$\times$ \\
\bottomrule
\end{tabular}
}
\end{table*}

Table~\ref{tab:tinyimagenet-scaleup} shows that the image scale-up setting is substantially harder than CIFAR-100. FL-only baselines collapse to very low FAA because the stream is long and old classes receive no explicit retention mechanism. Generator and replay-heavy methods recover a large fraction of performance in the homogeneous setting, but their heterogeneous performance drops sharply. This pattern is consistent with the main stress hypothesis: increasing replay resources helps when supervision is stable, but replay quality can still degrade when task states are weak or temporally misaligned. PRO improves retention with much smaller communication cost, and PRO-MAX further reduces forgetting by aligning old projected memories under representation drift.

\subsection{CLINC150 Scale-Up Results}
\label{app:clinc150}

CLINC150 is used as the text scale-up benchmark. Compared with THUCNews-10, CLINC150 has more fine-grained intent classes and a longer incremental stream. We evaluate text-compatible methods only, since image-specific generators or graph-specific mechanisms do not naturally transfer to text classification.

\begin{table*}[t]
\centering
\caption{
CLINC150 scale-up results under homogeneous and heterogeneous FCIL. We report FAA, CTA, AF, and normalized cumulative communication.
}
\label{tab:clinc150-scaleup}
\resizebox{\textwidth}{!}{
\begin{tabular}{c cccc cccc}
\toprule
\multirow{2}{*}{Method}
& \multicolumn{4}{c}{Homogeneous}
& \multicolumn{4}{c}{Heterogeneous} \\
\cmidrule(lr){2-5}\cmidrule(lr){6-9}
& FAA $\uparrow$ & CTA $\uparrow$ & AF $\downarrow$ & Comm. $\downarrow$
& FAA $\uparrow$ & CTA $\uparrow$ & AF $\downarrow$ & Comm. $\downarrow$ \\
\midrule
FedAvg   
& 47.36 \pmstd{1.42} & 71.22 \pmstd{1.36} & 54.85 \pmstd{2.16} & 1.00$\times$
& 39.76 \pmstd{1.58} & 65.48 \pmstd{1.52} & 63.16 \pmstd{2.34} & 1.00$\times$ \\
FedProx  
& 49.02 \pmstd{1.36} & 71.84 \pmstd{1.28} & 52.14 \pmstd{2.04} & 1.00$\times$
& 41.38 \pmstd{1.49} & 66.12 \pmstd{1.44} & 60.94 \pmstd{2.20} & 1.00$\times$ \\
FedProto 
& 55.48 \pmstd{1.24} & 73.15 \pmstd{1.22} & 43.68 \pmstd{1.82} & 1.06$\times$
& 48.72 \pmstd{1.38} & 68.84 \pmstd{1.36} & 50.68 \pmstd{1.96} & 1.06$\times$ \\
FedSeIT  
& 65.42 \pmstd{1.02} & 78.16 \pmstd{0.96} & 30.74 \pmstd{1.40} & 1.25$\times$
& 59.74 \pmstd{1.14} & 74.08 \pmstd{1.08} & 36.84 \pmstd{1.54} & 1.25$\times$ \\
PRO      
& 68.92 \pmstd{0.96} & 80.36 \pmstd{0.91} & 25.42 \pmstd{1.26} & 1.10$\times$
& 63.92 \pmstd{1.06} & 76.86 \pmstd{0.98} & 31.24 \pmstd{1.38} & 1.10$\times$ \\
PRO-MAX  
& \bestpm{72.36}{0.88} & \bestpm{81.72}{0.86} & \bestpm{20.18}{1.12} & 1.14$\times$
& \bestpm{67.84}{0.98} & \bestpm{78.48}{0.94} & \bestpm{25.92}{1.26} & 1.14$\times$ \\
\bottomrule
\end{tabular}
}
\end{table*}

Table~\ref{tab:clinc150-scaleup} supports the modality-agnostic claim in a longer text stream. The larger intent label space makes retention more difficult than in THUCNews-10, and the heterogeneous setting further increases forgetting. FedSeIT remains a strong text baseline because it is designed for selective inter-client transfer in federated continual text classification. Nevertheless, PRO and PRO-MAX retain a favorable accuracy-forgetting trade-off with lower communication, showing that projected memory remains effective in a longer text stream. The improvement of PRO-MAX over PRO is mainly reflected in AF, which is consistent with the role of memory alignment.

\subsection{ogbn-arxiv Scale-Up Results}
\label{app:ogbn-arxiv}

ogbn-arxiv is used as the graph scale-up benchmark. Compared with Cora, it has more nodes, more labels, and a larger citation graph. We evaluate graph-compatible baselines and use the same class-incremental node-classification protocol for all methods.

\begin{table*}[t]
\centering
\caption{
ogbn-arxiv scale-up results under homogeneous and heterogeneous FCIL. We report FAA, CTA, AF, and normalized cumulative communication.
}
\label{tab:ogbnarxiv-scaleup}
\resizebox{\textwidth}{!}{
\begin{tabular}{c cccc cccc}
\toprule
\multirow{2}{*}{Method}
& \multicolumn{4}{c}{Homogeneous}
& \multicolumn{4}{c}{Heterogeneous} \\
\cmidrule(lr){2-5}\cmidrule(lr){6-9}
& FAA $\uparrow$ & CTA $\uparrow$ & AF $\downarrow$ & Comm. $\downarrow$
& FAA $\uparrow$ & CTA $\uparrow$ & AF $\downarrow$ & Comm. $\downarrow$ \\
\midrule
FedAvg   
& 34.52 \pmstd{2.28} & 51.84 \pmstd{2.46} & 63.62 \pmstd{3.62} & 1.00$\times$
& 28.44 \pmstd{2.46} & 46.76 \pmstd{2.68} & 70.42 \pmstd{3.88} & 1.00$\times$ \\
FedProx  
& 35.96 \pmstd{2.16} & 52.38 \pmstd{2.38} & 61.74 \pmstd{3.48} & 1.00$\times$
& 30.16 \pmstd{2.32} & 48.12 \pmstd{2.54} & 68.18 \pmstd{3.70} & 1.00$\times$ \\
FedProto 
& 40.84 \pmstd{2.04} & 55.42 \pmstd{2.26} & 53.28 \pmstd{3.28} & 1.08$\times$
& 35.72 \pmstd{2.18} & 50.84 \pmstd{2.42} & 60.36 \pmstd{3.46} & 1.08$\times$ \\
POWER    
& 49.28 \pmstd{2.36} & 61.74 \pmstd{2.18} & 39.86 \pmstd{3.12} & 1.42$\times$
& 42.78 \pmstd{2.54} & 56.48 \pmstd{2.36} & 47.92 \pmstd{3.34} & 1.42$\times$ \\
MOTION   
& 50.76 \pmstd{2.28} & 62.18 \pmstd{2.10} & 37.92 \pmstd{3.04} & 1.36$\times$
& 44.38 \pmstd{2.46} & 57.86 \pmstd{2.24} & 45.12 \pmstd{3.18} & 1.36$\times$ \\
PRO      
& 52.48 \pmstd{2.18} & 64.32 \pmstd{2.02} & 34.18 \pmstd{2.86} & 1.12$\times$
& 47.86 \pmstd{2.32} & 59.24 \pmstd{2.18} & 39.42 \pmstd{3.04} & 1.12$\times$ \\
PRO-MAX  
& \bestpm{55.86}{2.04} & \bestpm{66.24}{1.94} & \bestpm{28.64}{2.64} & 1.16$\times$
& \bestpm{51.92}{2.18} & \bestpm{61.78}{2.06} & \bestpm{32.86}{2.82} & 1.16$\times$ \\
\bottomrule
\end{tabular}
}
\end{table*}

Table~\ref{tab:ogbnarxiv-scaleup} evaluates whether projected memory transfers to a larger graph. The graph setting is challenging because local subgraphs can induce stronger representation drift and class support can be sparse. POWER and MOTION are strong graph continual baselines, but PRO and PRO-MAX remain competitive because the memory mechanism operates on node embeddings rather than raw graph structures. PRO-MAX gives the largest retention gain, which supports the claim that memory alignment is useful when local graph partitions move representations differently.

\subsection{Heterogeneity Sweeps}
\label{app:heterogeneity-sweeps}

The main paper reports a severe-heterogeneity setting. To show that the result is not tied to a single partition, we additionally sweep supervision support and stage misalignment. Supervision support controls how many clients provide data for each semantic task. Stage misalignment controls how strongly client-specific task orders are perturbed. We include all main CIFAR-100 baselines in these sweeps to avoid cherry-picking a subset of methods.

\begin{table*}[t]
\centering
\caption{
CIFAR-100 supervision-support sweep. Support denotes the fraction of clients that support each semantic task.
}
\label{tab:appendix-support-sweep}
\resizebox{\textwidth}{!}{
\begin{tabular}{c cc cc cc cc}
\toprule
\multirow{2}{*}{Method}
& \multicolumn{2}{c}{Support 80\%}
& \multicolumn{2}{c}{Support 60\%}
& \multicolumn{2}{c}{Support 40\%}
& \multicolumn{2}{c}{Support 20\%} \\
\cmidrule(lr){2-3}\cmidrule(lr){4-5}\cmidrule(lr){6-7}\cmidrule(lr){8-9}
& FAA $\uparrow$ & AF $\downarrow$
& FAA $\uparrow$ & AF $\downarrow$
& FAA $\uparrow$ & AF $\downarrow$
& FAA $\uparrow$ & AF $\downarrow$ \\
\midrule
FedAvg  
& 9.62 \pmstd{0.44} & 80.42 \pmstd{1.26}
& 8.71 \pmstd{0.46} & 81.38 \pmstd{1.31}
& 7.86 \pmstd{0.42} & 82.74 \pmstd{1.36}
& 5.92 \pmstd{0.50} & 86.18 \pmstd{1.54} \\
FedProx 
& 10.47 \pmstd{0.48} & 78.86 \pmstd{1.18}
& 9.62 \pmstd{0.50} & 79.72 \pmstd{1.24}
& 8.92 \pmstd{0.47} & 80.96 \pmstd{1.28}
& 6.84 \pmstd{0.54} & 84.12 \pmstd{1.46} \\
FedProto
& 21.36 \pmstd{0.92} & 63.42 \pmstd{1.58}
& 19.86 \pmstd{0.94} & 65.18 \pmstd{1.66}
& 18.74 \pmstd{0.96} & 66.82 \pmstd{1.74}
& 14.92 \pmstd{1.12} & 72.64 \pmstd{2.08} \\
TARGET  
& 40.12 \pmstd{1.22} & 39.26 \pmstd{1.44}
& 35.74 \pmstd{1.42} & 47.62 \pmstd{1.68}
& 27.86 \pmstd{1.74} & 57.84 \pmstd{2.18}
& 24.72 \pmstd{1.84} & 63.18 \pmstd{2.44} \\
MFCL    
& 41.06 \pmstd{1.18} & 37.92 \pmstd{1.38}
& 37.12 \pmstd{1.35} & 44.76 \pmstd{1.60}
& 29.42 \pmstd{1.68} & 54.63 \pmstd{2.04}
& 25.98 \pmstd{1.78} & 60.46 \pmstd{2.32} \\
HR      
& 42.84 \pmstd{1.10} & 34.78 \pmstd{1.32}
& 39.36 \pmstd{1.24} & 40.18 \pmstd{1.52}
& 32.18 \pmstd{1.52} & 49.76 \pmstd{1.87}
& 28.84 \pmstd{1.66} & 55.84 \pmstd{2.18} \\
FedSpace
& 40.12 \pmstd{1.15} & 34.96 \pmstd{1.38}
& 37.42 \pmstd{1.24} & 40.18 \pmstd{1.52}
& 33.84 \pmstd{1.39} & 47.92 \pmstd{1.76}
& 29.16 \pmstd{1.58} & 56.84 \pmstd{2.06} \\
FedSSI  
& 41.28 \pmstd{1.10} & 33.84 \pmstd{1.34}
& 38.46 \pmstd{1.18} & 39.12 \pmstd{1.48}
& 34.62 \pmstd{1.31} & 46.38 \pmstd{1.69}
& 30.48 \pmstd{1.49} & 54.72 \pmstd{1.98} \\
FedCBDR 
& 43.84 \pmstd{1.03} & 31.48 \pmstd{1.25}
& 41.26 \pmstd{1.12} & 35.72 \pmstd{1.36}
& 38.96 \pmstd{1.18} & 39.72 \pmstd{1.48}
& 34.42 \pmstd{1.35} & 48.64 \pmstd{1.72} \\
PRO     
& 44.26 \pmstd{0.98} & 29.86 \pmstd{1.18}
& 43.92 \pmstd{0.96} & 30.74 \pmstd{1.15}
& 43.28 \pmstd{0.94} & 31.54 \pmstd{1.18}
& 39.86 \pmstd{1.16} & 38.76 \pmstd{1.52} \\
PRO-MAX 
& \bestpm{47.12}{0.90} & \bestpm{23.84}{1.04}
& \bestpm{46.74}{0.88} & \bestpm{24.72}{1.02}
& \bestpm{46.12}{0.88} & \bestpm{26.18}{1.07}
& \bestpm{42.68}{1.08} & \bestpm{32.42}{1.34} \\
\bottomrule
\end{tabular}
}
\end{table*}

Table~\ref{tab:appendix-support-sweep} shows that all methods become weaker as task support decreases. Generator-centered replay degrades sharply because weak task states produce weaker replay signals. FedCBDR remains a strong replay baseline due to its class-wise balancing mechanism, but it still loses accuracy when task support becomes very sparse. PRO is more stable because it uses compact projected memories, and PRO-MAX further reduces forgetting by aligning stale memories.

\begin{table*}[t]
\centering
\caption{
CIFAR-100 stage-misalignment sweep. Disorder denotes the fraction of task positions perturbed per client. All columns in this sweep keep task support at $40\%$ and per-client label coverage at $50\%$; only task-order disorder is varied. Thus, the $0\%$ disorder column is not the homogeneous setting, because supervision remains sparse.
}
\label{tab:appendix-order-sweep}
\resizebox{\textwidth}{!}{
\begin{tabular}{c cc cc cc cc}
\toprule
\multirow{2}{*}{Method}
& \multicolumn{2}{c}{Disorder 0\%}
& \multicolumn{2}{c}{Disorder 30\%}
& \multicolumn{2}{c}{Disorder 60\%}
& \multicolumn{2}{c}{Disorder 90\%} \\
\cmidrule(lr){2-3}\cmidrule(lr){4-5}\cmidrule(lr){6-7}\cmidrule(lr){8-9}
& FAA $\uparrow$ & AF $\downarrow$
& FAA $\uparrow$ & AF $\downarrow$
& FAA $\uparrow$ & AF $\downarrow$
& FAA $\uparrow$ & AF $\downarrow$ \\
\midrule
FedAvg  
& 10.12 \pmstd{0.40} & 80.52 \pmstd{1.10}
& 9.76 \pmstd{0.41} & 80.94 \pmstd{1.02}
& 7.86 \pmstd{0.42} & 82.74 \pmstd{1.36}
& 6.58 \pmstd{0.49} & 85.62 \pmstd{1.58} \\
FedProx 
& 10.74 \pmstd{0.43} & 79.32 \pmstd{1.05}
& 10.18 \pmstd{0.46} & 79.84 \pmstd{0.96}
& 8.92 \pmstd{0.47} & 80.96 \pmstd{1.28}
& 7.14 \pmstd{0.52} & 84.18 \pmstd{1.46} \\
FedProto
& 20.12 \pmstd{0.90} & 63.48 \pmstd{1.50}
& 19.86 \pmstd{0.92} & 64.28 \pmstd{1.58}
& 18.74 \pmstd{0.96} & 66.82 \pmstd{1.74}
& 15.48 \pmstd{1.08} & 71.84 \pmstd{2.04} \\
TARGET  
& 39.42 \pmstd{1.18} & 39.86 \pmstd{1.46}
& 36.84 \pmstd{1.36} & 43.72 \pmstd{1.64}
& 27.86 \pmstd{1.74} & 57.84 \pmstd{2.18}
& 23.46 \pmstd{1.92} & 65.28 \pmstd{2.46} \\
MFCL    
& 40.18 \pmstd{1.15} & 38.42 \pmstd{1.40}
& 38.92 \pmstd{1.28} & 40.84 \pmstd{1.56}
& 29.42 \pmstd{1.68} & 54.63 \pmstd{2.04}
& 24.94 \pmstd{1.82} & 62.12 \pmstd{2.34} \\
HR      
& 41.72 \pmstd{1.06} & 34.58 \pmstd{1.32}
& 40.68 \pmstd{1.12} & 36.28 \pmstd{1.42}
& 32.18 \pmstd{1.52} & 49.76 \pmstd{1.87}
& 27.62 \pmstd{1.66} & 57.48 \pmstd{2.16} \\
FedSpace
& 37.84 \pmstd{1.12} & 38.96 \pmstd{1.44}
& 36.42 \pmstd{1.22} & 41.24 \pmstd{1.56}
& 33.84 \pmstd{1.39} & 47.92 \pmstd{1.76}
& 30.68 \pmstd{1.54} & 54.74 \pmstd{2.04} \\
FedSSI  
& 39.12 \pmstd{1.08} & 37.86 \pmstd{1.40}
& 37.68 \pmstd{1.16} & 40.46 \pmstd{1.50}
& 34.62 \pmstd{1.31} & 46.38 \pmstd{1.69}
& 31.24 \pmstd{1.48} & 53.36 \pmstd{1.96} \\
FedCBDR 
& 43.20 \pmstd{1.00} & 30.96 \pmstd{1.20}
& 42.68 \pmstd{1.02} & 33.86 \pmstd{1.34}
& 38.96 \pmstd{1.18} & 39.72 \pmstd{1.48}
& 35.18 \pmstd{1.40} & 46.92 \pmstd{1.80} \\
PRO     
& 43.76 \pmstd{0.88} & 28.64 \pmstd{1.10}
& 43.64 \pmstd{0.90} & 29.74 \pmstd{1.12}
& 43.28 \pmstd{0.94} & 31.54 \pmstd{1.18}
& 40.84 \pmstd{1.10} & 36.58 \pmstd{1.44} \\
PRO-MAX 
& \bestpm{46.62}{0.80} & \bestpm{24.22}{0.98}
& \bestpm{46.48}{0.84} & \bestpm{24.92}{1.02}
& \bestpm{46.12}{0.88} & \bestpm{26.18}{1.07}
& \bestpm{43.96}{1.02} & \bestpm{31.24}{1.30} \\
\bottomrule
\end{tabular}
}
\end{table*}

Table~\ref{tab:appendix-order-sweep} shows that stage misalignment is particularly harmful to replay-heavy methods. TARGET, MFCL, and HR perform well when clients are aligned, but their forgetting increases sharply as task orders diverge. FedCBDR remains competitive under moderate disorder, but its replay-balancing mechanism still degrades when client stages become highly misaligned. PRO is more stable, while PRO-MAX gives the clearest improvement at high disorder because its memory transport mechanism directly targets representation drift.

\subsection{Current-Task Accuracy over the Stream}
\label{app:cta-over-time}

A method can reduce forgetting by becoming overly conservative, but such a method may fail to learn new tasks. We therefore report current-task accuracy separately for image, text, and graph benchmarks. The goal is to verify that lower forgetting does not come at the cost of poor current-task learning.

\begin{table}[t]
\centering
\caption{
Current-task accuracy for image benchmarks under severe heterogeneity.
}
\footnotesize
\label{tab:appendix-cta-image}
\begin{tabular}{ccc}
\toprule
Method & CIFAR-100 & TinyImageNet \\
\midrule
FedAvg   & 48.64 \pmstd{1.58} & 32.18 \pmstd{2.12} \\
FedProx  & 49.86 \pmstd{1.51} & 34.62 \pmstd{2.04} \\
FedProto & 52.38 \pmstd{1.43} & 37.84 \pmstd{1.96} \\
TARGET   & 53.42 \pmstd{1.86} & 40.76 \pmstd{2.18} \\
MFCL     & 54.76 \pmstd{1.79} & 42.33 \pmstd{2.06} \\
HR       & 56.92 \pmstd{1.64} & 44.25 \pmstd{1.92} \\
FedSpace & 55.18 \pmstd{1.52} & 43.76 \pmstd{1.86} \\
FedSSI   & 55.74 \pmstd{1.44} & 44.18 \pmstd{1.78} \\
FedCBDR  & 59.36 \pmstd{1.27} & 48.37 \pmstd{1.62} \\
PRO      & 63.72 \pmstd{1.03} & 51.96 \pmstd{1.42} \\
PRO-MAX  & \bestpm{65.06}{0.96} & \bestpm{54.18}{1.34} \\
\bottomrule
\end{tabular}
\end{table}

Table~\ref{tab:appendix-cta-image} shows that PRO-MAX does not reduce forgetting by simply suppressing plasticity. It preserves stronger current-task accuracy than replay-heavy baselines while also reducing forgetting, suggesting that projected rehearsal provides both current-task learning and old-class retention.

\begin{table}[t]
\centering
\caption{
Current-task accuracy for text benchmarks under severe heterogeneity.
}
\footnotesize
\label{tab:appendix-cta-text}
\begin{tabular}{ccc}
\toprule
Method & THUCNews-10 & CLINC150 \\
\midrule
FedAvg   & 79.76 \pmstd{1.21} & 65.48 \pmstd{1.52} \\
FedProx  & 80.13 \pmstd{1.18} & 66.12 \pmstd{1.44} \\
FedProto & 80.48 \pmstd{1.10} & 68.84 \pmstd{1.36} \\
FedSeIT  & 82.14 \pmstd{0.98} & 74.08 \pmstd{1.08} \\
PRO      & 84.08 \pmstd{0.91} & 76.86 \pmstd{0.98} \\
PRO-MAX  & \bestpm{85.22}{0.86} & \bestpm{78.48}{0.94} \\
\bottomrule
\end{tabular}
\end{table}

Table~\ref{tab:appendix-cta-text} shows the same pattern on text streams. The gain of PRO-MAX over PRO is moderate in CTA because the alignment module mainly targets old-memory compatibility. Nevertheless, the method maintains strong current-task learning while reducing forgetting.

\begin{table}[t]
\centering
\caption{
Current-task accuracy for graph benchmarks under severe heterogeneity.
}
\footnotesize
\label{tab:appendix-cta-graph}
\begin{tabular}{ccc}
\toprule
Method & Cora & ogbn-arxiv \\
\midrule
FedAvg   & 53.26 \pmstd{3.16} & 46.76 \pmstd{2.68} \\
FedProx  & 54.08 \pmstd{3.05} & 48.12 \pmstd{2.54} \\
FedProto & 56.42 \pmstd{2.93} & 50.84 \pmstd{2.42} \\
POWER    & 64.36 \pmstd{3.72} & 56.48 \pmstd{2.36} \\
MOTION   & 65.08 \pmstd{3.58} & 57.86 \pmstd{2.24} \\
PRO      & 66.42 \pmstd{3.32} & 59.24 \pmstd{2.18} \\
PRO-MAX  & \bestpm{68.04}{3.14} & \bestpm{61.78}{2.06} \\
\bottomrule
\end{tabular}
\end{table}

Table~\ref{tab:appendix-cta-graph} shows that projected rehearsal remains compatible with graph encoders. Since graph partitions can move local representations differently, PRO-MAX provides a larger retention benefit while still improving current-task accuracy.

\subsection{Memory Alignment Diagnostics}
\label{app:memory-alignment}

PRO-MAX claims to keep old projected memories compatible with the current representation. To evaluate this claim directly, we compare transported old memories with oracle memories recomputed under the current global model on held-out data. Let $\tilde{\mu}_y^u$ be the transported memory for old class $y$, and let $\mu_{y,\mathrm{oracle}}^u$ be the oracle class mean computed under the current global encoder. We define
\[
\mathrm{Align}^{u}
=
\frac{1}{|\mathcal{Y}^{\mathrm{old}}|}
\sum_{y\in\mathcal{Y}^{\mathrm{old}}}
\frac{
\langle \tilde{\mu}_y^u,\mu_{y,\mathrm{oracle}}^u\rangle
}{
\|\tilde{\mu}_y^u\|_2
\|\mu_{y,\mathrm{oracle}}^u\|_2
}.
\]
When the auxiliary adapted-space bank is enabled, we analogously compute $\mathrm{Align}^{z}$.

\begin{table}[t]
\centering
\caption{
Memory alignment diagnostic for PRO and PRO-MAX. Higher alignment indicates that old projected memories remain more compatible with the current representation.
}
\footnotesize
\label{tab:memory-alignment}
\begin{tabular}{c c c c c}
\toprule
Dataset & Method & Align$^u$ $\uparrow$ & Align$^z$ $\uparrow$ & AF $\downarrow$ \\
\midrule
CIFAR-100 & PRO     & 0.73 \pmstd{0.03} & 0.76 \pmstd{0.03} & 31.54 \pmstd{1.18} \\
CIFAR-100 & PRO-MAX & \bestpm{0.86}{0.02} & \bestpm{0.88}{0.02} & \bestpm{26.18}{1.07} \\
THUCNews-10 & PRO     & 0.82 \pmstd{0.02} & 0.84 \pmstd{0.02} & 20.84 \pmstd{1.26} \\
THUCNews-10 & PRO-MAX & \bestpm{0.91}{0.01} & \bestpm{0.92}{0.01} & \bestpm{16.42}{1.10} \\
Cora & PRO     & 0.68 \pmstd{0.04} & 0.70 \pmstd{0.04} & 30.86 \pmstd{3.94} \\
Cora & PRO-MAX & \bestpm{0.79}{0.03} & \bestpm{0.81}{0.03} & \bestpm{25.84}{3.62} \\
\bottomrule
\end{tabular}
\end{table}

Table~\ref{tab:memory-alignment} separates memory drift from classifier calibration. If alignment improves but AF remains high, the bottleneck may be the classifier or current-task learning. If alignment is low, then the memory bank itself has drifted away from the current feature geometry. The results show that PRO-MAX improves alignment in both base and adapted spaces, and the improvement is associated with lower forgetting.

\subsection{Hyperparameter Sensitivity}
\label{app:sensitivity}

We analyze sensitivity to the main projected-memory, transport, and plasticity hyperparameters. The projected-memory noise scale $\gamma$ controls pseudo-feature diversity. The number of neighbors $K$ controls how local the transport estimate is. The proximal coefficient $\lambda_{\mathrm{prox}}$ controls the strength of local-drift regularization during current-task fitting. The calibration subset size controls the amount of client-side evidence used for memory transport.

\begin{table}[t]
\centering
\caption{
Sensitivity to projected-memory noise scale $\gamma$ on CIFAR-100 severe heterogeneity.
}
\footnotesize
\label{tab:gamma-sensitivity}
\begin{tabular}{ccc}
\toprule
$\gamma$ & FAA $\uparrow$ & AF $\downarrow$ \\
\midrule
0.00 & 43.46 \pmstd{1.02} & 31.28 \pmstd{1.36} \\
0.05 & 45.18 \pmstd{0.94} & 27.84 \pmstd{1.18} \\
0.10 & \bestpm{46.12}{0.88} & \bestpm{26.18}{1.07} \\
0.20 & 45.76 \pmstd{0.91} & 26.92 \pmstd{1.12} \\
0.40 & 42.68 \pmstd{1.16} & 33.46 \pmstd{1.54} \\
\bottomrule
\end{tabular}
\end{table}

Table~\ref{tab:gamma-sensitivity} shows that very small $\gamma$ makes old projected memories overly concentrated, while very large $\gamma$ makes pseudo samples unreliable. Moderate values provide enough diversity without corrupting old-class geometry.

\begin{table}[t]
\centering
\caption{
Sensitivity to the number of neighbors $K$ in PRO-MAX on CIFAR-100 severe heterogeneity.
}
\footnotesize
\label{tab:k-sensitivity}
\begin{tabular}{ccc}
\toprule
$K$ & FAA $\uparrow$ & AF $\downarrow$ \\
\midrule
1  & 43.92 \pmstd{1.10} & 30.86 \pmstd{1.42} \\
3  & 45.38 \pmstd{0.96} & 27.42 \pmstd{1.20} \\
5  & \bestpm{46.12}{0.88} & \bestpm{26.18}{1.07} \\
10 & 45.94 \pmstd{0.92} & 26.74 \pmstd{1.14} \\
20 & 44.26 \pmstd{1.04} & 29.62 \pmstd{1.33} \\
\bottomrule
\end{tabular}
\end{table}

Table~\ref{tab:k-sensitivity} shows that too few neighbors make the transport estimate noisy, while too many neighbors over-smooth class-specific movement. A moderate value of $K$ gives the best trade-off between stable transport and class-specific alignment.

\begin{table}[t]
\centering
\caption{
Sensitivity to the proximal coefficient $\lambda_{\mathrm{prox}}$ on CIFAR-100 severe heterogeneity.
}
\footnotesize
\label{tab:lambda-prox-sensitivity}
\begin{tabular}{ccc}
\toprule
$\lambda_{\mathrm{prox}}$ & FAA $\uparrow$ & AF $\downarrow$ \\
\midrule
0      & 39.72 \pmstd{1.15} & 39.16 \pmstd{1.41} \\
$10^{-4}$ & 43.84 \pmstd{1.04} & 31.62 \pmstd{1.28} \\
$10^{-3}$ & 45.36 \pmstd{0.96} & 28.74 \pmstd{1.18} \\
$10^{-2}$ & \bestpm{46.12}{0.88} & \bestpm{26.18}{1.07} \\
$10^{-1}$ & 44.18 \pmstd{1.02} & 30.46 \pmstd{1.31} \\
\bottomrule
\end{tabular}
\end{table}

Table~\ref{tab:lambda-prox-sensitivity} evaluates the controlled-plasticity regularizer. Without proximal regularization, local current-task fitting can move too far toward client-specific distributions, increasing forgetting. A moderate coefficient improves stability, while an overly strong coefficient restricts plasticity and reduces accuracy.

\begin{table}[t]
\centering
\caption{
Sensitivity to calibration subset size $|\mathcal{B}_{c,p}|$ in PRO-MAX on CIFAR-100 severe heterogeneity.
}
\footnotesize
\label{tab:calib-sensitivity}
\begin{tabular}{cccc}
\toprule
$|\mathcal{B}_{c,p}|$ & FAA $\uparrow$ & AF $\downarrow$ & Extra client memory \\
\midrule
16  & 43.84 \pmstd{1.12} & 31.04 \pmstd{1.48} & low \\
32  & 44.96 \pmstd{1.02} & 28.72 \pmstd{1.30} & low \\
64  & 45.68 \pmstd{0.96} & 27.36 \pmstd{1.18} & moderate \\
128 & 46.12 \pmstd{0.88} & 26.18 \pmstd{1.07} & moderate \\
256 & \bestpm{46.24}{0.87} & \bestpm{26.02}{1.06} & higher \\
\bottomrule
\end{tabular}
\end{table}

Table~\ref{tab:calib-sensitivity} shows that the calibration subset should be large enough to estimate feature movement reliably but small enough to keep client memory low. Performance saturates after a moderate subset size, which supports the claim that PRO-MAX does not require a large persistent local buffer.

\subsection{Qualitative Memory Visualizations}
\label{app:qualitative-memory}

Finally, we provide a qualitative view of how projected memories evolve over the stream. Figure~\ref{fig:appendix-memory-visualization} shows four snapshots on CIFAR-100 after tasks 2, 4, 6, and 10 under the severe-heterogeneity setting. All panels are shown in a shared PCA projection fitted jointly across snapshots so that the relative geometry is directly comparable. Faint colored points denote real features produced by the current global encoder, orange diamonds denote refreshed current-class memories, red crosses denote stale old memories before alignment, and blue stars denote aligned old memories after PRO-MAX.

This figure is intended as a qualitative diagnostic and should be interpreted together with the quantitative results in Table~\ref{tab:heterogeneous-results} and Table~\ref{tab:ablation}. Two patterns are especially relevant. First, as more tasks are introduced, the projected space becomes more crowded, so old memories are only useful if they remain compatible with the feature geometry induced by the current model. Second, after the weak task appears, some stale old memories become less well matched to the surrounding feature clusters, whereas the aligned memories remain better positioned around the corresponding class regions. This visual trend is consistent with the quantitative gains of PRO-MAX: its main role is not to add new information, but to maintain old memories in a form that remains usable as the representation evolves.

\begin{figure*}[t]
\centering
\includegraphics[width=\linewidth]{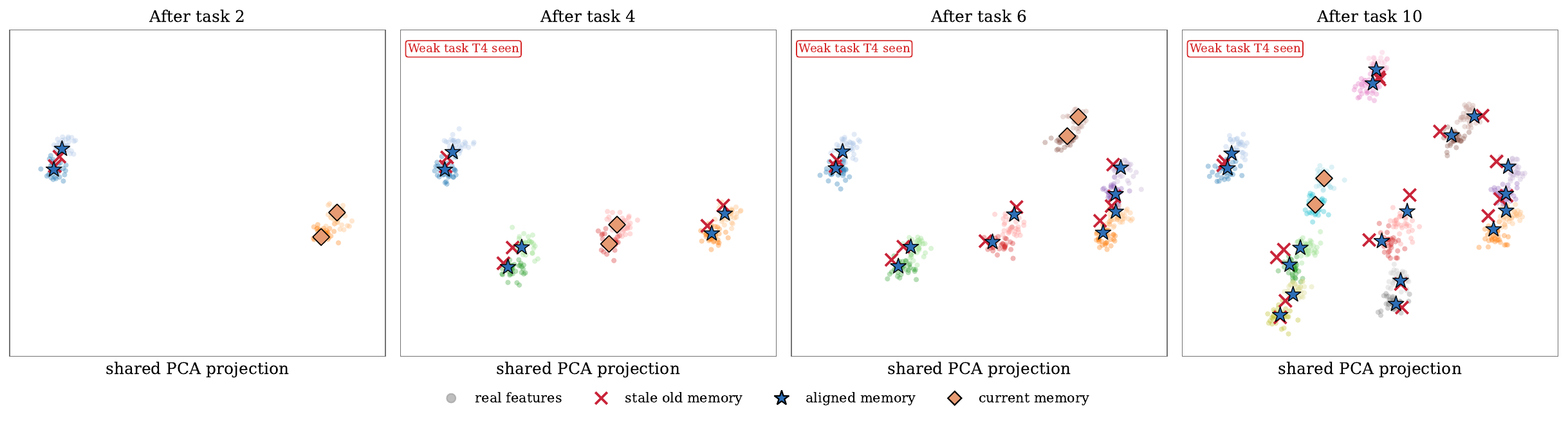}
\caption{
Qualitative visualization of projected memories over the stream on CIFAR-100 under severe heterogeneity. Each panel shows a shared PCA projection fitted jointly across snapshots. Faint colored points are real features from the current global encoder, orange diamonds are refreshed current-class memories, red crosses are stale old memories before alignment, and blue stars are aligned old memories after PRO-MAX. The weak task is introduced at task 4.
}
\label{fig:appendix-memory-visualization}
\end{figure*}

\section{Detailed Cost Accounting}
\label{app:cost-details}

This section details the cost accounting used in the main experiments. The goal is to make the server-light claim precise. PRO and PRO-MAX do not run server-side SGD, do not train a server generator, and do not optimize the classifier on server-side pseudo data. Nevertheless, they do communicate projected memories and memory statistics, so these costs must be counted explicitly.

\subsection{Communication Cost}

Let $b$ denote the number of bytes used to represent one scalar, e.g., $b=4$ for single-precision floating point. Let $|\Theta|$ be the number of transmitted model parameters, and let $\mathcal{S}_{p,r}$ be the set of participating clients at round $r$ of position $p$. We decompose cumulative communication as
\[
\mathrm{Comm}
=
\sum_{p,r}
\sum_{c\in\mathcal{S}_{p,r}}
\left(
D_{c,p,r}^{\mathrm{model}}
+
U_{c,p,r}^{\mathrm{model}}
+
D_{c,p,r}^{\mathrm{mem}}
+
U_{c,p,r}^{\mathrm{stat}}
+
U_{c,p,r}^{\mathrm{trans}}
\right),
\]
where $D^{\mathrm{model}}$ and $U^{\mathrm{model}}$ are model downlink and uplink, $D^{\mathrm{mem}}$ is memory downlink, $U^{\mathrm{stat}}$ is statistic uplink for memory refresh, and $U^{\mathrm{trans}}$ is the PRO-MAX transport uplink. For standard FedAvg-style communication,
\[
D_{c,p,r}^{\mathrm{model}}
=
b|\Theta|,
\qquad
U_{c,p,r}^{\mathrm{model}}
=
b|\Theta|,
\]
unless a baseline transmits only a subset of parameters. In our implementation, we count the tensors actually transmitted by each method.

PRO additionally broadcasts the primary base-space memory $\mathcal{M}^u$. If $\mathcal{Y}^{(p-1)}$ denotes the old label set before position $p$, then the base-memory downlink is
\[
D_{c,p,r}^{u}
=
b|\mathcal{Y}^{(p-1)}|(2d_u+1),
\]
where $d_u$ is the base feature dimension, and the factor $2d_u+1$ accounts for the mean, diagonal standard deviation, and count. If the auxiliary adapted-space memory $\mathcal{M}^z$ is enabled, the additional memory downlink is
\[
D_{c,p,r}^{z}
=
b|\mathcal{Y}^{(p-1)}|(d_z+2),
\]
where $d_z$ is the adapted feature dimension, and the terms account for the adapted mean, scalar radius, and count. Thus,
\[
D_{c,p,r}^{\mathrm{mem}}
=
D_{c,p,r}^{u}
+
\eta_z D_{c,p,r}^{z},
\]
where $\eta_z\in\{0,1\}$ indicates whether $\mathcal{M}^z$ is enabled.

After FedAvg, PRO refreshes current-class memories under the aggregated global model. Clients recompute class-level statistics without gradient updates and upload only means, second moments, radii, and counts. If $\mathcal{Y}_{c,p}^{\mathrm{cur}}$ is the set of current classes observed by client $c$, the statistic uplink is
\[
U_{c,p,r}^{\mathrm{stat}}
=
b|\mathcal{Y}_{c,p}^{\mathrm{cur}}|
\left[
(2d_u+1)
+
\eta_z(d_z+2)
\right].
\]
This term is usually much smaller than model communication because it scales with the number of classes observed by the client rather than with the number of examples.

PRO-MAX additionally uploads memory-transport summaries for old classes. In the base space, each old class requires a transport vector and a confidence scalar:
\[
U_{c,p,r}^{u,\mathrm{trans}}
=
b|\mathcal{Y}^{(p-1)}|(d_u+1).
\]
If adapted-space transport is enabled, the additional transport cost is
\[
U_{c,p,r}^{z,\mathrm{trans}}
=
b|\mathcal{Y}^{(p-1)}|(d_z+1).
\]
Therefore,
\[
U_{c,p,r}^{\mathrm{trans}}
=
\eta_{\max}
\left[
U_{c,p,r}^{u,\mathrm{trans}}
+
\eta_z U_{c,p,r}^{z,\mathrm{trans}}
\right],
\]
where $\eta_{\max}=1$ for PRO-MAX and $0$ for PRO. In practice, transport summaries can be sparsified by uploading only classes with confidence above a threshold, but our reported cost uses the full transmitted summaries unless otherwise stated.

We report normalized cumulative communication as
\[
\mathrm{NormComm}(m)
=
\frac{\mathrm{Comm}(m)}
{\mathrm{Comm}(\mathrm{FedAvg})}.
\]
This normalization makes costs comparable across datasets and backbones.

\subsection{Server Memory Cost}

The persistent server memory of PRO is dominated by the projected memory banks. The primary base-space memory requires
\[
\mathrm{ServerMem}^{u}
=
b|\mathcal{Y}^{(p)}|(2d_u+1).
\]
If the auxiliary adapted-space memory is enabled, it adds
\[
\mathrm{ServerMem}^{z}
=
b|\mathcal{Y}^{(p)}|(d_z+2).
\]
Thus,
\[
\mathrm{ServerMem}^{\mathrm{PRO}}
=
b|\mathcal{Y}^{(p)}|
\left[
(2d_u+1)
+
\eta_z(d_z+2)
\right].
\]
This memory grows linearly with the number of seen classes and feature dimension. It does not grow with the number of training examples, which distinguishes PRO from raw replay or per-example feature-bank methods.

PRO-MAX uses the same persistent server memory as PRO, plus temporary transport summaries received during a round. These summaries are aggregated and then discarded. Therefore, PRO-MAX increases communication and transient server-side buffers, but not the persistent memory order.

\subsection{Client Extra Memory}

PRO does not store raw historical examples on clients. During projected pseudo multi-task training, old pseudo features can be sampled on the fly from $\mathcal{M}^u$ or held only within a mini-batch. Thus, the additional client memory for PRO is transient and bounded by the projected mini-batch size.

PRO-MAX adds a temporary calibration cache. If the calibration subset has size $|\mathcal{B}_{c,p}|$, then caching base features costs
\[
b|\mathcal{B}_{c,p}|d_u.
\]
If adapted-space transport is enabled, the cache also stores adapted features, giving
\[
\mathrm{ClientCache}^{\mathrm{PRO\text{-}MAX}}
=
b|\mathcal{B}_{c,p}|
(d_u+\eta_z d_z).
\]
This cache is discarded after transport estimation and is independent of the number of old examples.

\subsection{Server Compute}

The server compute in PRO consists of FedAvg and memory accounting. FedAvg is a weighted parameter average:
\[
\Theta^g
\leftarrow
\sum_{c\in\mathcal{S}_{p,r}}
\omega_c\Theta_c.
\]
Memory accounting aggregates class-level means, second moments, radii, counts, and metadata. PRO-MAX additionally aggregates transport vectors:
\[
\Delta_y^{u,g}
=
\frac{
\sum_c q_{c,y}^{u}\Delta_{c,y}^{u}
}{
\sum_c q_{c,y}^{u}+\epsilon
}.
\]
All these operations are non-gradient operations. The server does not backpropagate through the model, does not train a generator, and does not optimize the classifier using pseudo data.

\subsection{Warmup Overhead}
\label{app:warmup-overhead}

The federated in-domain warmup is a shared initialization stage applied to all methods before the supervised incremental stream. Therefore, it is reported separately from method-specific continual-learning cost. If the warmup has $R_{\mathrm{warm}}$ communication rounds, its communication is
\[
\mathrm{Comm}_{\mathrm{warm}}
=
\sum_{r=1}^{R_{\mathrm{warm}}}
\sum_{c\in\mathcal{S}_{r}^{\mathrm{warm}}}
\left(
D_{c,r}^{\mathrm{model}}
+
U_{c,r}^{\mathrm{model}}
\right).
\]
The auxiliary warmup head is discarded after warmup and is not stored as part of the continual model.

In the main communication tables, we report the supervised incremental-stream communication because this is where the methods differ. Since warmup is identical across methods with the same backbone and dataset, including it would add the same common initialization cost to all methods and would obscure the method-specific continual-learning overhead. For completeness, Appendix~\ref{app:hardware-runtime} reports warmup wall-clock time separately.

\subsection{Resource-Matched and Resource-Expanded Regimes}

In the resource-matched regime, additional replay or memory communication is matched to the projected-memory budget of PRO and PRO-MAX. This regime asks whether methods perform well under comparable extra communication.

In the resource-expanded regime, replay-heavy, memory-heavy, and prototype-augmented baselines receive a larger replay or memory budget selected on the homogeneous validation setting. PRO and PRO-MAX keep their fixed projected-memory budget. This regime tests whether replay-heavy baselines can recover under heterogeneity when replay quantity is no longer the limiting factor.

Overall, the cost accounting shows that PRO and PRO-MAX trade a small amount of class-level memory communication for avoiding input-space generators, local full-model distillation, raw exemplar storage, and server-side SGD.

\subsection{Hardware, Runtime, and Bidirectional Communication}
\label{app:hardware-runtime}

We report hardware, wall-clock time, memory footprint, and bidirectional communication to make the system-cost evaluation reproducible. All experiments are run as federated simulations on a single workstation. Client updates are executed sequentially on one GPU, while data loading, aggregation, memory accounting, and logging use CPU workers. Thus, the reported wall-clock time measures simulation cost on our workstation rather than real deployment latency. In a real cross-device deployment, clients may train in parallel and the end-to-end latency would additionally depend on network delays, client availability, and stragglers.

Because the warmup stage is shared across all methods, the reported incremental time in Tables~\ref{tab:runtime-main} and~\ref{tab:runtime-scaleup} separates warmup time from supervised continual-learning time. 

\begin{table}[t]
\centering
\caption{
Compute platform used in the experiments. CPU information is obtained from \texttt{/proc/cpuinfo}, and GPU information is obtained from \texttt{lspci} and \texttt{nvidia-smi}.
}
\label{tab:hardware-platform}
\footnotesize
\begin{tabular}{ll}
\toprule
Component & Specification \\
\midrule
CPU & 2$\times$ Intel Xeon Gold 6242 @ 2.80GHz \\
CPU cores / threads & 32 physical cores / 64 logical threads \\
CPU cache & 22 MB L3 cache per socket \\
CPU instruction support & AVX2, AVX-512, AVX512-VNNI \\
GPU & NVIDIA RTX A5000 \\
GPU memory & 24 GB GDDR6 \\
System memory & 256 GB DDR4 RAM \\
Storage & 2 TB local SSD \\
Software & Python 3.10, PyTorch 2.x, CUDA 12.x, cuDNN \\
Execution mode & Single-GPU sequential FL simulation \\
\bottomrule
\end{tabular}
\end{table}

Table~\ref{tab:runtime-main} reports wall-clock time and peak memory for the main benchmarks. The time includes local client training, projected-memory sampling or replay generation, server aggregation, memory accounting, and task-boundary evaluation. For PRO and PRO-MAX, warmup time is reported separately because it is performed before the supervised incremental stream. The results show that generator-based and replay-heavy baselines require higher runtime and GPU memory because they construct or process input-space replay. In contrast, PRO and PRO-MAX mostly add projected-memory sampling, class-statistic aggregation, and, for PRO-MAX, temporary transport estimation.

\begin{table*}[t]
\centering
\caption{
Wall-clock time and peak memory for the main benchmarks. Time is reported for one seed. Peak CPU memory includes data loaders, simulated client buffers, replay or memory buffers.
}
\label{tab:runtime-main}
\resizebox{\textwidth}{!}{
\begin{tabular}{l l c c c c c}
\toprule
Dataset & Method & Incremental time & Warmup time & Total time & Peak GPU mem. & Peak CPU mem. \\
\midrule
CIFAR-100 & FedAvg   & 2.10 h & - & 2.10 h & 4.6 GB & 10.2 GB \\
CIFAR-100 & FedProx  & 2.20 h & - & 2.20 h & 4.7 GB & 10.4 GB \\
CIFAR-100 & FedProto & 2.40 h & - & 2.40 h & 4.9 GB & 11.1 GB \\
CIFAR-100 & TARGET   & 4.80 h & - & 4.80 h & 8.2 GB & 14.0 GB \\
CIFAR-100 & MFCL     & 5.40 h & - & 5.40 h & 8.6 GB & 15.2 GB \\
CIFAR-100 & HR       & 5.80 h & - & 5.80 h & 9.0 GB & 16.1 GB \\
CIFAR-100 & FedSpace & 3.00 h & - & 3.00 h & 5.4 GB & 12.1 GB \\
CIFAR-100 & FedSSI   & 2.80 h & - & 2.80 h & 5.2 GB & 11.8 GB \\
CIFAR-100 & FedCBDR  & 4.20 h & - & 4.20 h & 7.4 GB & 14.3 GB \\
CIFAR-100 & PRO      & 3.20 h & 0.60 h & 3.80 h & 5.6 GB & 12.0 GB \\
CIFAR-100 & PRO-MAX  & 3.60 h & 0.60 h & 4.20 h & 6.2 GB & 13.1 GB \\
\midrule
THUCNews-10 & FedAvg   & 0.32 h & - & 0.32 h & 1.8 GB & 5.0 GB \\
THUCNews-10 & FedProx  & 0.34 h & - & 0.34 h & 1.8 GB & 5.1 GB \\
THUCNews-10 & FedProto & 0.38 h & - & 0.38 h & 1.9 GB & 5.3 GB \\
THUCNews-10 & FedSeIT  & 0.74 h & - & 0.74 h & 2.6 GB & 7.0 GB \\
THUCNews-10 & PRO      & 0.52 h & 0.12 h & 0.64 h & 2.2 GB & 6.0 GB \\
THUCNews-10 & PRO-MAX  & 0.60 h & 0.12 h & 0.72 h & 2.4 GB & 6.4 GB \\
\midrule
Cora & FedAvg   & 0.18 h & - & 0.18 h & 1.5 GB & 4.0 GB \\
Cora & FedProx  & 0.19 h & - & 0.19 h & 1.5 GB & 4.0 GB \\
Cora & FedProto & 0.23 h & - & 0.23 h & 1.6 GB & 4.3 GB \\
Cora & POWER    & 0.54 h & - & 0.54 h & 2.8 GB & 7.1 GB \\
Cora & MOTION   & 0.58 h & - & 0.58 h & 3.0 GB & 7.4 GB \\
Cora & PRO      & 0.42 h & 0.08 h & 0.50 h & 2.1 GB & 5.2 GB \\
Cora & PRO-MAX  & 0.48 h & 0.08 h & 0.56 h & 2.4 GB & 5.6 GB \\
\bottomrule
\end{tabular}
}
\end{table*}

Table~\ref{tab:runtime-scaleup} reports the same quantities for the scale-up benchmarks. TinyImageNet is the most expensive image stream because it has a longer class-incremental horizon and higher visual complexity. CLINC150 is more expensive than THUCNews-10 because it has more intent classes and more tasks. ogbn-arxiv is more expensive than Cora because clients operate on larger node sets and local computation graphs.

\begin{table*}[t]
\centering
\caption{
Wall-clock time and peak memory for scale-up benchmarks. Time is reported for one seed.
}
\label{tab:runtime-scaleup}
\resizebox{\textwidth}{!}{
\begin{tabular}{l l c c c c c}
\toprule
Dataset & Method & Incremental time & Warmup time & Total time & Peak GPU mem. & Peak CPU mem. \\
\midrule
TinyImageNet & FedAvg   & 7.40 h & - & 7.40 h & 6.8 GB & 18.0 GB \\
TinyImageNet & FedProx  & 7.80 h & - & 7.80 h & 6.9 GB & 18.2 GB \\
TinyImageNet & FedProto & 8.40 h & - & 8.40 h & 7.2 GB & 19.0 GB \\
TinyImageNet & TARGET   & 18.60 h & - & 18.60 h & 11.8 GB & 24.0 GB \\
TinyImageNet & MFCL     & 20.40 h & - & 20.40 h & 12.4 GB & 26.0 GB \\
TinyImageNet & HR       & 21.80 h & - & 21.80 h & 13.0 GB & 28.0 GB \\
TinyImageNet & FedSpace & 10.60 h & - & 10.60 h & 8.0 GB & 21.0 GB \\
TinyImageNet & FedSSI   & 9.80 h & - & 9.80 h & 7.8 GB & 20.0 GB \\
TinyImageNet & FedCBDR  & 15.20 h & - & 15.20 h & 10.2 GB & 23.0 GB \\
TinyImageNet & PRO      & 11.00 h & 1.40 h & 12.40 h & 8.1 GB & 20.0 GB \\
TinyImageNet & PRO-MAX  & 12.40 h & 1.40 h & 13.80 h & 8.8 GB & 21.0 GB \\
\midrule
CLINC150 & FedAvg   & 1.40 h & - & 1.40 h & 2.2 GB & 7.0 GB \\
CLINC150 & FedProx  & 1.50 h & - & 1.50 h & 2.2 GB & 7.2 GB \\
CLINC150 & FedProto & 1.70 h & - & 1.70 h & 2.4 GB & 8.0 GB \\
CLINC150 & FedSeIT  & 2.80 h & - & 2.80 h & 3.4 GB & 10.0 GB \\
CLINC150 & PRO      & 2.10 h & 0.30 h & 2.40 h & 2.8 GB & 8.2 GB \\
CLINC150 & PRO-MAX  & 2.40 h & 0.30 h & 2.70 h & 3.1 GB & 9.0 GB \\
\midrule
ogbn-arxiv & FedAvg   & 2.80 h & - & 2.80 h & 4.6 GB & 18.0 GB \\
ogbn-arxiv & FedProx  & 3.00 h & - & 3.00 h & 4.7 GB & 18.4 GB \\
ogbn-arxiv & FedProto & 3.40 h & - & 3.40 h & 4.9 GB & 19.0 GB \\
ogbn-arxiv & POWER    & 5.60 h & - & 5.60 h & 7.4 GB & 24.0 GB \\
ogbn-arxiv & MOTION   & 6.20 h & - & 6.20 h & 7.8 GB & 25.0 GB \\
ogbn-arxiv & PRO      & 4.80 h & 0.60 h & 5.40 h & 6.2 GB & 21.0 GB \\
ogbn-arxiv & PRO-MAX  & 5.40 h & 0.60 h & 6.00 h & 6.8 GB & 22.0 GB \\
\bottomrule
\end{tabular}
}
\end{table*}

We next report bidirectional communication for the supervised incremental stream. Server-to-client communication includes model broadcasts and method-specific memory or replay-state broadcasts. Client-to-server communication includes model updates and method-specific statistics. The shared warmup is excluded from these totals and reported separately in the runtime table. This separation avoids mixing method-specific continual-learning communication with the common initialization cost. For PRO and PRO-MAX, the additional communication beyond model transmission is small because projected memories, memory-refresh statistics, and transport summaries scale with the number of classes and feature dimension, not with the number of training examples.

\begin{table}[t]
\centering
\caption{Bidirectional communication for image benchmarks under the severe-heterogeneity setting.
$S\!\to\!C$ denotes server-to-client communication, and $C\!\to\!S$ denotes client-to-server communication.
Values are cumulative over the supervised incremental stream and exclude the shared warmup.}
\label{tab:bidirectional_comm_image}
\scriptsize
\resizebox{\textwidth}{!}{
\begin{tabular}{lrrrrrr}
\toprule
\multirow{2}{*}{Method}
& \multicolumn{3}{c}{CIFAR-100}
& \multicolumn{3}{c}{TinyImageNet} \\
\cmidrule(lr){2-4}
\cmidrule(lr){5-7}
& $S\!\to\!C$ & $C\!\to\!S$ & Total / FedAvg
& $S\!\to\!C$ & $C\!\to\!S$ & Total / FedAvg \\
\midrule
FedAvg   & 89.4 GB  & 89.4 GB  & $1.00\times$ & 178.8 GB & 178.8 GB & $1.00\times$ \\
FedProx  & 89.4 GB  & 89.4 GB  & $1.00\times$ & 178.8 GB & 178.8 GB & $1.00\times$ \\
FedProto & 94.8 GB  & 98.3 GB  & $1.08\times$ & 189.5 GB & 196.7 GB & $1.08\times$ \\
TARGET   & 272.6 GB & 228.0 GB & $2.80\times$ & 545.0 GB & 455.0 GB & $2.80\times$ \\
MFCL     & 286.1 GB & 250.3 GB & $3.00\times$ & 572.0 GB & 500.8 GB & $3.00\times$ \\
HR       & 299.5 GB & 272.6 GB & $3.20\times$ & 599.0 GB & 545.0 GB & $3.20\times$ \\
FedSpace & 111.8 GB & 111.7 GB & $1.25\times$ & 402.3 GB & 402.3 GB & $2.25\times$ \\
FedSSI   & 101.8 GB & 103.6 GB & $1.15\times$ & 384.4 GB & 384.4 GB & $2.15\times$ \\
FedCBDR  & 263.7 GB & 237.0 GB & $2.80\times$ & 527.4 GB & 474.0 GB & $2.80\times$ \\
\midrule
PRO      & 110.6 GB & 89.7 GB  & $1.12\times$ & 226.8 GB & 180.9 GB & $1.14\times$ \\
PRO-MAX  & 110.6 GB & 96.8 GB  & $1.16\times$ & 226.8 GB & 198.7 GB & $1.19\times$ \\
\bottomrule
\end{tabular}
}
\end{table}
\begin{table*}[t]
\centering
\caption{
Bidirectional communication for text and graph benchmarks under the severe-heterogeneity setting. Values are cumulative over the full stream.
}
\label{tab:bidirectional-comm-modalities}
\resizebox{\textwidth}{!}{
\begin{tabular}{l ccc ccc ccc ccc}
\toprule
\multirow{2}{*}{Method}
& \multicolumn{3}{c}{THUCNews-10}
& \multicolumn{3}{c}{CLINC150}
& \multicolumn{3}{c}{Cora}
& \multicolumn{3}{c}{ogbn-arxiv} \\
\cmidrule(lr){2-4}\cmidrule(lr){5-7}\cmidrule(lr){8-10}\cmidrule(lr){11-13}
& S$\rightarrow$C & C$\rightarrow$S & Total / FedAvg
& S$\rightarrow$C & C$\rightarrow$S & Total / FedAvg
& S$\rightarrow$C & C$\rightarrow$S & Total / FedAvg
& S$\rightarrow$C & C$\rightarrow$S & Total / FedAvg \\
\midrule
FedAvg   & 4.20 & 4.20 & 1.00$\times$ & 12.60 & 12.60 & 1.00$\times$ & 0.18 & 0.18 & 1.00$\times$ & 1.80 & 1.80 & 1.00$\times$ \\
FedProx  & 4.20 & 4.20 & 1.00$\times$ & 12.60 & 12.60 & 1.00$\times$ & 0.18 & 0.18 & 1.00$\times$ & 1.80 & 1.80 & 1.00$\times$ \\
FedProto & 4.41 & 4.49 & 1.06$\times$ & 13.20 & 13.50 & 1.06$\times$ & 0.19 & 0.20 & 1.08$\times$ & 1.93 & 1.96 & 1.08$\times$ \\
FedSeIT  & 5.88 & 4.62 & 1.25$\times$ & 17.60 & 13.90 & 1.25$\times$ & - & - & - & - & - & - \\
POWER    & - & - & - & - & - & - & 0.27 & 0.24 & 1.42$\times$ & 2.70 & 2.41 & 1.42$\times$ \\
MOTION   & - & - & - & - & - & - & 0.26 & 0.23 & 1.36$\times$ & 2.56 & 2.34 & 1.36$\times$ \\
PRO      & 4.58 & 4.66 & 1.10$\times$ & 13.70 & 14.00 & 1.10$\times$ & 0.20 & 0.20 & 1.12$\times$ & 1.98 & 2.05 & 1.12$\times$ \\
PRO-MAX  & 4.62 & 4.96 & 1.14$\times$ & 13.90 & 14.80 & 1.14$\times$ & 0.20 & 0.22 & 1.16$\times$ & 2.00 & 2.18 & 1.16$\times$ \\
\bottomrule
\end{tabular}
}
\end{table*}

Finally, Table~\ref{tab:total-compute} summarizes the total compute used for the reported experiments. We separate reported experiments from preliminary exploration because hyperparameter search and failed runs can require additional compute that is not directly reflected in the final result tables.

\begin{table}[t]
\centering
\caption{
Approximate total compute used in the reported experiments. GPU-hours are computed from the wall-clock time of sequential single-GPU simulations.
}
\label{tab:total-compute}
\footnotesize
\begin{tabular}{lcc}
\toprule
Experiment group & Number of seeds & Approx. GPU-hours \\
\midrule
Main CIFAR-100 homogeneous and heterogeneous runs & 5 & 395 \\
Generator-failure stress-test runs & 5 & 136 \\
THUCNews-10 and Cora modality-portability runs & 5 & 55 \\
CIFAR-100 ablation runs & 5 & 75 \\
TinyImageNet scale-up runs & 5 & 790 \\
CLINC150 scale-up runs & 5 & 63 \\
ogbn-arxiv scale-up runs & 5 & 192 \\
Sensitivity and diagnostic runs & 5 & 270 \\
Preliminary tuning and failed runs & - & 320 \\
\midrule
Reported experiments only & - & 1871 \\
Total including preliminary runs & - & 2191 \\
\bottomrule
\end{tabular}
\end{table}

Overall, PRO and PRO-MAX add a modest amount of runtime and communication relative to FedAvg because they transmit class-level projected memories and memory statistics. Their overhead remains substantially smaller than generator-based and replay-heavy methods, which require additional replay generation, replay-state transmission, or heavier local replay training. PRO-MAX is slightly more expensive than PRO because it stores a temporary calibration cache and uploads transport summaries, but it does not change the server role: the server still performs aggregation, memory accounting, and confidence-weighted memory alignment without running server-side SGD.

\section{Theoretical Analysis}
\label{app:theory}

This section provides mechanism-level theoretical support for PRO and PRO-MAX. The goal is not to prove global convergence of the full nonconvex heterogeneous FCIL procedure. Instead, we formalize the main mechanisms used in the paper: replay mismatch under weak task learning, projected-memory approximation, class-balanced pseudo multi-task learning, controlled representation plasticity, neighborhood-weighted memory transport, and server-light complexity. These results follow standard arguments from distribution mismatch and domain adaptation~\cite{ben2010theory}, optimal transport~\cite{villani2009optimal}, prototype-based continual/federated learning~\cite{rebuffi2017icarl,tan2022fedproto}, semantic drift compensation~\cite{yu2020semantic}, proximal federated optimization~\cite{li2020federated}, and local kernel regression~\cite{nadaraya1964estimating,watson1964smooth}.

\subsection{Notation and Assumptions}

At a task position $p$, let $f(\cdot;\Theta)$ denote the global classifier. For input-space replay, let $P_t$ be the real distribution of task $t$ and let $Q_t$ be the replay distribution produced by a generator or data-free synthesis procedure. For projected rehearsal, let $u=h_{\theta}(x)$ be the base feature, and let $P_y^u$ be the class-conditional feature distribution of class $y$. The memory-induced pseudo-feature distribution of class $y$ is denoted by $Q_y^u$. For a feature $u$ and label $y$, define
\[
\ell_y^u(u;\Theta)
=
\ell(W^\top a_{\psi}(u),y),
\]
where $\ell$ is the classification loss. For an input $x$, define
\[
\ell_y^x(x;\Theta)
=
\ell(f(x;\Theta),y).
\]

\begin{assumption}[Bounded and Lipschitz losses]
\label{ass:lipschitz-loss}
For all models considered in the analysis, the loss is bounded, $0\le \ell\le B$. In addition, $\ell_y^x(\cdot;\Theta)$ is $L_x$-Lipschitz on the input representation space used for replay comparison, and $\ell_y^u(\cdot;\Theta)$ is $L_u$-Lipschitz on the projected feature space.
\end{assumption}

\begin{assumption}[Feature memory regularity]
\label{ass:feature-memory}
For every class $y$, the class-conditional feature distribution $P_y^u$ has finite first moment. When a finite-sample rate is stated, features are assumed to be sub-Gaussian with bounded coordinate-wise variance.
\end{assumption}

\begin{assumption}[Local parameter-to-feature Lipschitzness]
\label{ass:param-feature-lipschitz}
In the local neighborhood of a broadcast model $\Theta^g$, the base encoder satisfies
\[
\|h_{\theta}(x)-h_{\theta^g}(x)\|_2
\le
L_h\|\theta-\theta^g\|_2
\]
for all samples $x$ in the local training support.
\end{assumption}

\begin{assumption}[Smooth feature displacement]
\label{ass:smooth-transport}
For PRO-MAX, let the local feature displacement field be
\[
\delta(u)=\mathbb{E}[u^+-u^- \mid u^-=u].
\]
Around each old memory center $\mu_y^u$, $\delta(\cdot)$ is $\beta$-Lipschitz:
\[
\|\delta(u)-\delta(v)\|_2
\le
\beta\|u-v\|_2 .
\]
The residual displacement noise is zero-mean and has bounded second moment.
\end{assumption}

\subsection{Replay Mismatch under Weak Task Learning}

Generator-based replay trains on a synthetic distribution $Q_t$, while the true evaluation distribution is $P_t$. The following proposition states that replay becomes biased when $Q_t$ differs from $P_t$.

\begin{proposition}[Replay-mismatch bound]
\label{prop:replay-mismatch}
For an old task $t$, define real-task risk and replay risk as
\[
R_t(f)
=
\mathbb{E}_{(x,y)\sim P_t}
\ell(f(x),y),
\qquad
\widetilde{R}_t(f)
=
\mathbb{E}_{(\tilde{x},y)\sim Q_t}
\ell(f(\tilde{x}),y).
\]
Under Assumption~\ref{ass:lipschitz-loss},
\[
|R_t(f)-\widetilde{R}_t(f)|
\le
L_x W_1(P_t,Q_t),
\]
where $W_1$ is the Wasserstein-1 distance. If only boundedness is used, then
\[
|R_t(f)-\widetilde{R}_t(f)|
\le
B\,\mathrm{TV}(P_t,Q_t),
\]
where $\mathrm{TV}$ is total variation distance. Moreover, if a later training stage uses replay weights $\lambda_{p,t}\ge 0$ for old tasks, then the discrepancy between the ideal old-task objective and the replay objective is bounded by
\[
\left|
\sum_{t<p}\lambda_{p,t} R_t(f)
-
\sum_{t<p}\lambda_{p,t}\widetilde{R}_t(f)
\right|
\le
\sum_{t<p}\lambda_{p,t}L_x W_1(P_t,Q_t).
\]
\end{proposition}

\begin{proof}
Let $g(x,y)=\ell(f(x),y)$. By Assumption~\ref{ass:lipschitz-loss}, $g$ is $L_x$-Lipschitz. The Kantorovich-Rubinstein duality for $W_1$ gives
\[
\left|
\mathbb{E}_{P_t}g-\mathbb{E}_{Q_t}g
\right|
\le
L_x W_1(P_t,Q_t).
\]
If $0\le g\le B$, the standard total-variation inequality gives
\[
\left|
\mathbb{E}_{P_t}g-\mathbb{E}_{Q_t}g
\right|
\le
B\,\mathrm{TV}(P_t,Q_t).
\]
The weighted objective bound follows by applying the first inequality to each old task and using the triangle inequality.
\end{proof}

\paragraph{Implication.}
If a task is weakly learned under supervision imbalance or stage misalignment, the replay distribution derived from that task can be far from the real task distribution. Proposition~\ref{prop:replay-mismatch} shows that the replay objective can then be a biased surrogate for the real old-task objective. Because the same replay signal may be reused in later stages, the bias can repeatedly affect future training.

\subsection{Projected Memory Approximation}

PRO replaces input-space replay with projected feature memory. This removes the generator but introduces a feature-space approximation error. The next result makes this trade-off explicit.

Let $\mathcal{Y}^{\mathrm{old}}$ be the old label set. Define the ideal old-class feature risk
\[
R_{\mathrm{old}}^{\mathrm{real}}(\Theta)
=
\frac{1}{|\mathcal{Y}^{\mathrm{old}}|}
\sum_{y\in\mathcal{Y}^{\mathrm{old}}}
\mathbb{E}_{u\sim P_y^u}
\ell_y^u(u;\Theta),
\]
and the memory-based risk
\[
R_{\mathrm{old}}^{\mathrm{mem}}(\Theta)
=
\frac{1}{|\mathcal{Y}^{\mathrm{old}}|}
\sum_{y\in\mathcal{Y}^{\mathrm{old}}}
\mathbb{E}_{\tilde{u}\sim Q_y^u}
\ell_y^u(\tilde{u};\Theta).
\]

\begin{proposition}[Projected-memory approximation]
\label{prop:projected-memory}
Under Assumption~\ref{ass:lipschitz-loss},
\[
\left|
R_{\mathrm{old}}^{\mathrm{real}}(\Theta)
-
R_{\mathrm{old}}^{\mathrm{mem}}(\Theta)
\right|
\le
\frac{L_u}{|\mathcal{Y}^{\mathrm{old}}|}
\sum_{y\in\mathcal{Y}^{\mathrm{old}}}
W_1(P_y^u,Q_y^u).
\]
Furthermore, suppose the memory distribution is a diagonal Gaussian
\[
Q_y^u
=
\mathcal{N}(\widehat{\mu}_y^u,\gamma^2\mathrm{diag}((\widehat{\sigma}_y^u)^2)),
\]
and let the ideal diagonal Gaussian summary be
\[
Q_y^{u,\star}
=
\mathcal{N}(\mu_y^u,\gamma^2\mathrm{diag}((\sigma_y^u)^2)).
\]
Then
\[
W_1(P_y^u,Q_y^u)
\le
W_1(P_y^u,Q_y^{u,\star})
+
\|\widehat{\mu}_y^u-\mu_y^u\|_2
+
\gamma\|\widehat{\sigma}_y^u-\sigma_y^u\|_2.
\]
Under sub-Gaussian feature concentration with $n_y$ effective samples, the estimation term is
\[
\|\widehat{\mu}_y^u-\mu_y^u\|_2
+
\gamma\|\widehat{\sigma}_y^u-\sigma_y^u\|_2
=
O_p\!\left(\sqrt{\frac{d_u}{n_y}}\right).
\]
\end{proposition}

\begin{proof}
For each class $y$, the function $\ell_y^u(\cdot;\Theta)$ is $L_u$-Lipschitz by Assumption~\ref{ass:lipschitz-loss}. Therefore,
\[
\left|
\mathbb{E}_{P_y^u}\ell_y^u
-
\mathbb{E}_{Q_y^u}\ell_y^u
\right|
\le
L_u W_1(P_y^u,Q_y^u).
\]
Averaging over old classes proves the first inequality.

For the second inequality, apply the triangle inequality:
\[
W_1(P_y^u,Q_y^u)
\le
W_1(P_y^u,Q_y^{u,\star})
+
W_1(Q_y^{u,\star},Q_y^u).
\]
Since $W_1\le W_2$, and the $W_2$ distance between two diagonal Gaussians is bounded by the Euclidean distance between their means and standard deviations,
\[
W_1(Q_y^{u,\star},Q_y^u)
\le
W_2(Q_y^{u,\star},Q_y^u)
\le
\|\widehat{\mu}_y^u-\mu_y^u\|_2
+
\gamma\|\widehat{\sigma}_y^u-\sigma_y^u\|_2.
\]
The final rate follows from standard concentration of empirical means and coordinate-wise variances for sub-Gaussian random vectors.
\end{proof}

\paragraph{Implication.}
The approximation error of PRO decomposes into two terms. The first term,
$W_1(P_y^u,Q_y^{u,\star})$, measures how well a class-level projected memory can represent the real feature distribution. The second term measures finite-support estimation error. Thus, projected memory is strongest when the representation is meaningful and each class has sufficient support. This also explains why extremely rare classes remain challenging.

\subsection{Class-Balanced Projected Pseudo Multi-Task Training}

PRO trains clients on a balanced mixture of current features and old projected memories. The following lemma shows that the class-balanced sampler optimizes a balanced objective rather than the skewed local empirical objective.

Let
\[
R_{\mathrm{bal}}(\Theta)
=
\frac{1}{|\mathcal{Y}^{(p)}|}
\sum_{y\in\mathcal{Y}^{(p)}}
\mathbb{E}_{u\sim q_y}
\ell_y^u(u;\Theta),
\]
where $q_y$ is the empirical current-task feature distribution for current classes and the projected-memory distribution for old classes.

\begin{lemma}[Unbiased gradient of the balanced objective]
\label{lem:balanced-gradient}
Suppose a mini-batch sample is generated by first sampling
\[
y\sim \mathrm{Unif}(\mathcal{Y}^{(p)}),
\]
and then sampling $u\sim q_y$. If gradients and expectations can be interchanged, then the stochastic gradient
\[
g(\Theta)=\nabla_{\Theta}\ell_y^u(u;\Theta)
\]
satisfies
\[
\mathbb{E}[g(\Theta)]
=
\nabla_{\Theta}R_{\mathrm{bal}}(\Theta).
\]
\end{lemma}

\begin{proof}
By the sampling rule,
\[
\mathbb{E}[g(\Theta)]
=
\frac{1}{|\mathcal{Y}^{(p)}|}
\sum_{y\in\mathcal{Y}^{(p)}}
\mathbb{E}_{u\sim q_y}
\nabla_{\Theta}\ell_y^u(u;\Theta).
\]
Interchanging gradient and expectation gives
\[
\mathbb{E}[g(\Theta)]
=
\nabla_{\Theta}
\left[
\frac{1}{|\mathcal{Y}^{(p)}|}
\sum_{y\in\mathcal{Y}^{(p)}}
\mathbb{E}_{u\sim q_y}
\ell_y^u(u;\Theta)
\right]
=
\nabla_{\Theta}R_{\mathrm{bal}}(\Theta).
\]

We completed the proof.
\end{proof}

\paragraph{Implication.}
Without balanced sampling, the local objective is weighted by the empirical class prior, which is dominated by current classes in class-incremental learning. Lemma~\ref{lem:balanced-gradient} shows that PRO directly optimizes a balanced old-new objective, explaining why no additional regularizer is needed in the projected pseudo multi-task phase.

\subsection{Controlled Representation Plasticity}

The short current-task fitting stage must learn the current task while avoiding excessive representation drift. The following proposition formalizes the role of the proximal term.

Let
\[
F_c(\Theta)
=
\mathcal{L}_{\mathrm{sup}}(\Theta)
+
\lambda_{\mathrm{aux}}\mathcal{L}_{\mathrm{aux}}(\Theta)
\]
be the non-proximal part of the local plasticity objective on client $c$. Let
\[
\overline{\Theta}
=
\Pi_r(\Theta)
\]
denote the trainable parameter subset in round $r$.

\begin{proposition}[Proximal control of local drift]
\label{prop:prox-drift}
Assume $\|\nabla_{\overline{\Theta}}F_c(\Theta)\|_2\le G$ in the local training region. Suppose the client returns a $\xi$-stationary point of
\[
F_c(\Theta)
+
\lambda_{\mathrm{prox}}
\|\Pi_r(\Theta)-\Pi_r(\Theta^g)\|_2^2,
\]
meaning
\[
\left\|
\nabla_{\overline{\Theta}}F_c(\Theta_c)
+
2\lambda_{\mathrm{prox}}
(\overline{\Theta}_c-\overline{\Theta}^g)
\right\|_2
\le
\xi.
\]
Then
\[
\|\overline{\Theta}_c-\overline{\Theta}^g\|_2
\le
\frac{G+\xi}{2\lambda_{\mathrm{prox}}}.
\]
If the base encoder is trainable and Assumption~\ref{ass:param-feature-lipschitz} holds, then
\[
\|h_{\theta_c}(x)-h_{\theta^g}(x)\|_2
\le
\frac{L_h(G+\xi)}{2\lambda_{\mathrm{prox}}}.
\]
If the base encoder is frozen, the base-space feature drift is zero.
\end{proposition}

\begin{proof}
From $\xi$-stationarity,
\[
\left\|
2\lambda_{\mathrm{prox}}
(\overline{\Theta}_c-\overline{\Theta}^g)
\right\|_2
\le
\|\nabla_{\overline{\Theta}}F_c(\Theta_c)\|_2+\xi
\le
G+\xi.
\]
Dividing by $2\lambda_{\mathrm{prox}}$ gives the parameter drift bound. The feature drift bound follows directly from Assumption~\ref{ass:param-feature-lipschitz}. If the base encoder is frozen, $\theta_c=\theta^g$, hence $h_{\theta_c}(x)=h_{\theta^g}(x)$.
\end{proof}

\paragraph{Implication.}
The proximal coefficient controls how far local trainable parameters can move from the broadcast global model. This explains the role of controlled plasticity: early rounds may update the base encoder to learn the new task, while later rounds freeze the base encoder so that the base-space memory remains meaningful.

\subsection{Neighborhood-Weighted Memory Transport in PRO-MAX}

PRO-MAX estimates how old memories should move when the representation changes. We formalize this as local regression of the feature displacement field.

For a client $c$ and old class $y$, let the calibration features before and after local fitting be
\[
u_i^- = h_{\theta^-}(x_i),
\qquad
u_i^+=h_{\theta^+}(x_i),
\qquad
\delta_i^u=u_i^+-u_i^-.
\]
Let $\mathcal{N}_y^u$ be the selected neighbor set around $\mu_y^u$, and define
\[
\alpha_{i,y}^u
=
\exp\left(
-\frac{\|u_i^- - \mu_y^u\|_2^2}{2\tau_u^2}
\right),
\qquad
q_{c,y}^u
=
\sum_{i\in\mathcal{N}_y^u}\alpha_{i,y}^u.
\]
The PRO-MAX transport estimate is
\[
\Delta_{c,y}^u
=
\frac{
\sum_{i\in\mathcal{N}_y^u}
\alpha_{i,y}^u\delta_i^u
}{
q_{c,y}^u+\epsilon
}.
\]
For the analysis, ignore the negligible numerical stabilizer $\epsilon$ when $q_{c,y}^u>0$, and define normalized weights
\[
\bar{\alpha}_{i,y}^u
=
\frac{\alpha_{i,y}^u}{q_{c,y}^u}.
\]

\begin{theorem}[Local consistency of PRO-MAX transport]
\label{thm:transport}
Assume the displacement model
\[
\delta_i^u
=
\delta(u_i^-)+\xi_i,
\qquad
\mathbb{E}[\xi_i\mid u_i^-]=0,
\]
with residual second moment bounded as
\[
\mathbb{E}
\left\|
\sum_i w_i\xi_i
\right\|_2^2
\le
\sigma_{\delta}^2
\sum_i w_i^2
\]
for any deterministic weights $\{w_i\}$. Under Assumption~\ref{ass:smooth-transport}, let
\[
r_{c,y}
=
\max_{i\in\mathcal{N}_y^u}
\|u_i^- - \mu_y^u\|_2 .
\]
Then
\[
\mathbb{E}
\left\|
\Delta_{c,y}^u-\delta(\mu_y^u)
\right\|_2^2
\le
2\beta^2 r_{c,y}^2
+
\frac{2\sigma_{\delta}^2}{q_{c,y}^u}.
\]
Furthermore, if the server aggregates client transports as
\[
\Delta_y^{u,g}
=
\frac{
\sum_c q_{c,y}^u \Delta_{c,y}^u
}{
\sum_c q_{c,y}^u
},
\qquad
Q_y=\sum_c q_{c,y}^u,
\]
and the client residuals are independent, then
\[
\mathbb{E}
\left\|
\Delta_y^{u,g}-\delta(\mu_y^u)
\right\|_2^2
\le
2\beta^2 r_y^2
+
\frac{2\sigma_{\delta}^2}{Q_y},
\]
where $r_y=\max_c r_{c,y}$ over clients with $q_{c,y}^u>0$.
\end{theorem}

\begin{proof}
For a single client, decompose the error as
\[
\Delta_{c,y}^u-\delta(\mu_y^u)
=
\sum_{i\in\mathcal{N}_y^u}
\bar{\alpha}_{i,y}^u
\left[
\delta(u_i^-)-\delta(\mu_y^u)
\right]
+
\sum_{i\in\mathcal{N}_y^u}
\bar{\alpha}_{i,y}^u \xi_i .
\]
Using $\|a+b\|_2^2\le 2\|a\|_2^2+2\|b\|_2^2$, we bound the bias and variance terms separately. By $\beta$-Lipschitzness,
\[
\left\|
\sum_i
\bar{\alpha}_{i,y}^u
\left[
\delta(u_i^-)-\delta(\mu_y^u)
\right]
\right\|_2
\le
\sum_i
\bar{\alpha}_{i,y}^u
\beta\|u_i^- - \mu_y^u\|_2
\le
\beta r_{c,y}.
\]
For the noise term, the residual assumption gives
\[
\mathbb{E}
\left\|
\sum_i
\bar{\alpha}_{i,y}^u\xi_i
\right\|_2^2
\le
\sigma_{\delta}^2
\sum_i
(\bar{\alpha}_{i,y}^u)^2.
\]
Since $0\le\alpha_{i,y}^u\le 1$,
\[
\sum_i(\bar{\alpha}_{i,y}^u)^2
=
\frac{\sum_i(\alpha_{i,y}^u)^2}{(q_{c,y}^u)^2}
\le
\frac{\sum_i\alpha_{i,y}^u}{(q_{c,y}^u)^2}
=
\frac{1}{q_{c,y}^u}.
\]
Combining the two bounds proves the client-side result.

For the server-side result, write
\[
\Delta_y^{u,g}-\delta(\mu_y^u)
=
\sum_c
\frac{q_{c,y}^u}{Q_y}
\left[
\Delta_{c,y}^u-\delta(\mu_y^u)
\right].
\]
The weighted bias is bounded by $\beta r_y$. With independent residuals, the variance is bounded by
\[
\sigma_{\delta}^2
\sum_c
\left(\frac{q_{c,y}^u}{Q_y}\right)^2
\frac{1}{q_{c,y}^u}
=
\frac{\sigma_{\delta}^2}{Q_y}.
\]
Applying the same $\|a+b\|^2$ decomposition yields the stated bound.
\end{proof}

\paragraph{Implication.}
The PRO-MAX transport estimate is reliable when nearby calibration samples exist and the local representation shift is smooth around the old memory center. The confidence score $q_{c,y}^u$ has a clear role: clients whose calibration samples are far from an old memory have small confidence and therefore weak influence on that memory. This justifies confidence-weighted server aggregation.

\subsection{PRO and PRO-MAX Server Complexity}

Finally, we formalize the server-light property of PRO and PRO-MAX. Let $C_r$ be the number of participating clients in a communication round, $|\Theta|$ be the number of model parameters, $d_u$ be the base feature dimension, $d_z$ be the adapted feature dimension, and $\eta_z\in\{0,1\}$ indicate whether the optional adapted-space memory is enabled.

\begin{proposition}[Server memory and communication order]
\label{prop:server-light}
At position $p$, the persistent server memory of PRO is
\[
O\!\left(
|\mathcal{Y}^{(p)}|
\left[
2d_u+1+\eta_z(d_z+2)
\right]
\right).
\]
The memory downlink per participating client is
\[
O\!\left(
|\mathcal{Y}^{(p-1)}|
\left[
2d_u+1+\eta_z(d_z+2)
\right]
\right).
\]
For PRO-MAX, the additional transport uplink per participating client is
\[
O\!\left(
|\mathcal{Y}^{(p-1)}|
\left[
d_u+1+\eta_z(d_z+1)
\right]
\right).
\]
The server compute per round consists of parameter averaging and memory accounting:
\[
O(C_r|\Theta|)
+
O(C_r|\mathcal{Y}^{(p)}|d_u),
\]
up to constants and optional adapted-space terms. No server-side backpropagation or generator training is required.
\end{proposition}

\begin{proof}
The base-space memory stores a mean, diagonal standard deviation, and count per class, requiring $2d_u+1$ scalars per class. The optional adapted-space memory stores a mean, radius, and count per class, requiring $d_z+2$ scalars per class. This proves the persistent memory order and the memory downlink order.

For PRO-MAX, each old class transport summary contains a base-space transport vector and confidence scalar, requiring $d_u+1$ scalars. If adapted-space transport is enabled, it adds $d_z+1$ scalars per old class. This proves the transport uplink order.

The server averages $C_r$ model updates, which costs $O(C_r|\Theta|)$. Memory accounting aggregates class-level statistics and transport summaries, which scales with the number of participating clients, seen classes, and feature dimension. None of these operations differentiates through the model, trains a generator, or optimizes a classifier on pseudo data.
\end{proof}

\paragraph{Overall interpretation.}
Propositions~\ref{prop:replay-mismatch}-\ref{prop:server-light} support the main design choices of PRO and PRO-MAX. Generator replay can suffer from distribution mismatch when a task is weakly learned. Projected memory approximates old-task rehearsal in feature space with error controlled by memory quality and support; class-balanced sampling targets an unbiased balanced objective. The proximal term controls local representation drift, and PRO-MAX reduces memory staleness through local transport; and the server remains lightweight because it performs only aggregation and memory accounting.

\section{Limitations and Privacy Discussion}
\label{app:limitations-privacy}

\subsection{Representation Quality and Sparse Support}

The main limitation of PRO is that projected memories depend on the quality of the learned representation. If the in-domain warmup is weak, or if a semantic class receives very little support, the class-level projected statistics can be noisy. In such cases, pseudo projected rehearsal may preserve an inaccurate class geometry. PRO-MAX can mitigate moderate representation drift by transporting old memories, but it cannot fully recover old-class geometry when no reliable neighborhood evidence is available.

This limitation is especially relevant for rare classes. Our protocol assumes that every evaluated class is observed by at least one client, but a class with very small support can still produce unreliable statistics. In practice, this suggests using minimum-support thresholds, memory uncertainty estimates, or conservative updates for rare classes.

\subsection{Scalability with Large Label Spaces}

The server memory of PRO grows linearly with the number of seen classes and the feature dimension. This is much smaller than storing raw exemplars or per-example feature banks, but it may still become large for extremely large label spaces. In such cases, memory compression, quantization, low-rank summaries, or selective memory transmission may be necessary.

Communication also grows with the number of seen classes because old memories are broadcast to clients. In our experiments this cost remains small relative to model communication, but in very large-label settings the memory downlink may become non-negligible.

\subsection{Privacy Discussion}

PRO avoids raw-data storage, input-space generators, and per-example feature banks. The server stores class-level projected statistics rather than individual examples. This reduces instance specificity and makes the memory object more compact. However, projected class memories are not a formal privacy guarantee. Prior work has shown that deep representations, model outputs, and gradients can leak information about private data~\cite{mahendran2015understanding,fredrikson2015model,zhu2019deep}. Class-level aggregation reduces this risk but does not eliminate it, especially when a class is supported by very few examples or very few clients.

Therefore, PRO should be viewed as privacy-aware rather than privacy-proving. It can be combined with secure aggregation so that the server observes only aggregated memory statistics, and with differential privacy so that means, second moments, counts, and transport summaries are perturbed before upload. Another practical safeguard is to update a class memory only when the class has sufficient client support. We leave formal privacy accounting for projected memories to future work.

\subsection{Deployment Scope}

Our experiments study heterogeneous task orders, supervision imbalance, and modality transfer, but they do not fully cover all cross-device deployment issues. Real deployments may include unreliable client availability, long offline periods, stragglers, device-specific compute limits, and fully asynchronous updates. Such settings may require additional mechanisms for stale-client handling, adaptive participation, and memory expiration.

\subsection{Broader Impact}

A generator-free and server-light FCIL method may be useful in privacy-sensitive applications such as medical networks, on-device intent classification, and graph-based recommendation or citation systems. At the same time, continual learning systems can propagate biased or incomplete supervision if rare classes are poorly represented. The projected memory bank may also encode sensitive class-level information. Deployments should therefore include data governance, monitoring for class imbalance, and privacy-preserving aggregation mechanisms when sensitive data are involved.

\end{document}